\documentclass[twoside,11pt]{article}
\usepackage{jmlr2e}
\usepackage[utf8]{inputenc}
\inputencoding{utf8}

\usepackage{color,soul,booktabs}
\setulcolor{blue}
\newcommand{\BibTeX}{\textsc{B\kern-0.1emi\kern-0.017emb}\kern-0.15em\TeX}

\usepackage{siunitx}
\usepackage{hyperref}
\usepackage{support-caption}
\usepackage{subcaption}
\usepackage{amsmath,amssymb,amsfonts}
\usepackage{algorithmic}
\usepackage{booktabs}
\usepackage{tikz-cd}
\usetikzlibrary{shapes}
\usetikzlibrary{arrows}
\usetikzlibrary{er,positioning}
\usepackage[linesnumbered,algoruled,boxed,lined,ruled]{algorithm2e}

\usepackage{graphicx}
\usepackage{textcomp}
\usepackage{xcolor}
\usetikzlibrary{fit,calc}
\newcommand*{\tikzmk}[1]{\tikz[remember picture,overlay,] \node (#1) {};\ignorespaces}
\newcommand{\boxit}[1]{\tikz[remember picture,overlay]{\node[yshift=3pt,fill=#1,opacity=.25,fit={(A)($(B)+(.95\linewidth,.7\baselineskip)$)}] {};}\ignorespaces}
\usepackage{textcomp}
\makeatletter
\newcommand\@erelb@r[1]{%
  \mathrel{\tikz[baseline=-.5ex]\draw[#1] (0,0)--(.5,0);}
}
\newcommand{\erelbar}[1]{\@erelbar#1}
\def\@erelbar#1#2{%
  \ifcase\numexpr#1*4+#2\relax
    \@erelb@r{-}\or     
    \@erelb@r{->}\or    
    \@erelb@r{-|}\or    
    \@erelb@r{-o}\or   
    \@erelb@r{<-}\or    
    \@erelb@r{<->}\or   
    \@erelb@r{<-|}\or   
    \@erelb@r{<-o}\or   
    \@erelb@r{|-}\or    
    \@erelb@r{|->}\or   
    \@erelb@r{|-|}\or   
    \@erelb@r{|-o}\or 
    \@erelb@r{o-}\or   
    \@erelb@r{o->}\or  
    \@erelb@r{o-|}\or  
    \@erelb@r{o-o}    
  \else
    \@wrong
  \fi
}
\newcommand{\aic}{{\sc\texttt{FWD-AIC}}}
\newcommand{\bic}{{\sc\texttt{FWD-BIC}}}
\newcommand{\iamb}{{\sc\texttt{IAMB-FDR}}}
\newcommand{\lcd}{{\sc\texttt{LCD-AMP}}}
\newcommand{\opc}{{\sc \texttt{PC-like}}}
\newcommand{\spc}{{\sc \texttt{Stable-PC4AMP}}}
\newcommand{\cpc}{{\sc \texttt{Conservative-PC4AMP}}}
\newcommand{\scpc}{{\sc \texttt{Stable-Conservative-PC4AMP}}}

\jmlrheading{}{2020}{}{00/00}{00/00}{maj}{Mohammad Ali Javidian}

\firstpageno{1}
\ShortHeadings{AMP CGs: Minimal Separators and Structure Learning Algorithms}{Javidian, Valtorta, and Jamshidi}

\begin{document}

\title{AMP Chain Graphs: Minimal Separators\\ and Structure Learning Algorithms}
\author{\name{Mohammad Ali Javidian} \email{javidian@email.sc.edu}\\ 
	\name{Marco Valtorta} \email{mgv@cse.sc.edu}\\ 
	\name{Pooyan Jamshidi} \email{pjamshid@cse.sc.edu}\\
	\addr Department of Computer Science \& Engineering, University of South Carolina, Columbia, SC, 29201, USA.}

\editor{X}
\maketitle

\begin{abstract}We address the problem of finding a minimal separator in an Andersson–Madigan–Perlman chain graph (AMP CG), namely,  finding a set $Z$ of nodes that separates a given non-adjacent pair of nodes such that no proper subset of $Z$ separates that pair.
We analyze several versions of this problem
and offer \textit{polynomial time} algorithms for each.
These include finding a
minimal separator from a restricted set of nodes, finding a minimal separator for two given disjoint sets, and testing whether a given separator is minimal. 
To address the problem of learning the structure of AMP CGs from data, we show that the \opc~algorithm \citep{penea12amp} is \textit{order-dependent},
in the sense that the output can depend on the order in which the variables are given. We propose several modifications of the \opc~algorithm that remove part or all of this order-dependence. We also extend the decomposition-based approach for learning Bayesian networks (BNs) proposed by \citep{xie} to learn AMP CGs, which include BNs as a special case, under the faithfulness assumption. We prove the correctness of our extension using the minimal separator results. 
Using standard benchmarks and synthetically generated models and data in our experiments demonstrate the competitive
performance of our decomposition-based method, called \lcd, in comparison with the (modified versions of) \opc~algorithm. The \lcd~algorithm usually outperforms the \opc~algorithm, and our modifications of the \opc~algorithm learn structures that are more similar to the underlying ground truth graphs than the original \opc~algorithm, especially in high-dimensional settings. In particular, we empirically show that the results of both algorithms are more accurate and stabler when the sample size is reasonably large
and the underlying graph is sparse.
\end{abstract}

\begin{keywords}AMP chain graph, conditional independence, decomposition, separator, junction tree, augmented graph, triangulation, graphical model, Markov equivalent, structural learning.
\end{keywords}

\section{Introduction}

Probabilistic graphical models (PGMs), and their use for reasoning intelligently under
uncertainty, emerged in the 1980s within the statistical and artificial intelligence
reasoning communities. Probabilistic graphical models are now widely accepted as a powerful and
mature tools for reasoning under uncertainty. Unlike some of the ad hoc approaches taken in early experts systems, PGMs are based on the strong mathematical foundations of graph and probability theory. In fact, any PGM consists of two main components: (1) a graph that defines the structure of the model; and (2) a joint distribution over random variables of the model. 
The main advantages
of using PGMs compared to other models are that the representation is intuitive, inference can often be done efficiently and practical learning algorithms exist.~\footnote{These algorithms are fast enough in practice, even though learning, inference, and other reasoning tasks are NP-complete or worse in the worst case, because they  exploit sparsity and other features prevalent in application domains \citep{COOPER1990,kf}.}
This led PGMs to  become arguably the most important architecture for
reasoning with uncertainty in artificial intelligence \citep{kf,Neapolitan18}.  
There are many efficient algorithms for both inference
and learning available in open-source \citep{hel, Nagarajan2013,Scutari15} and commercial software (\href{https://www.hugin.com/}{Hugin}, \href{https://www.norsys.com/}{Netica}, \href{https://www.bayesfusion.com/}{GeNIe}, and \href{https://www.bayesfusion.com/}{BayesiaLab}). Moreover, their
power and efficacy has been proven through their successful application to an
enormous range of real-world problem domains. They can be used for a wide
range of reasoning tasks including prediction, monitoring, diagnosis, risk assessment and decision making \citep{sgs,Xiang2002,Jensen2007,FentonNeil}.  

One of the most basic subclasses of PGMs is Markov networks. The graphical framework of Markov networks are undirected graphs (UGs), in which each undirected edge represents
a symmetric relation i.e., direct correlation between the two variables it connects, while no edge
means that the variables are not directly correlated. The best known
and most widely used PGM class, however, is Bayesian networks. The graphical structures of Bayesian networks are directed acyclic graphs (DAGs). In a DAG the directed edges can often be seen as representing cause and effect (asymmetric) relationships e.g., see \citep{Motzek2017}.

Chain graphs (CGs) were introduced as a unification of directed and undirected graphs to model  systems
containing both symmetric and asymmetric relations. In fact, a chain graph is a type of mixed graph, admitting both directed and undirected edges, which contain no partially directed cycles. So, CGs may contain two types of edges,
the directed type that corresponds to the causal relationship in DAGs and a
second type of edge representing a symmetric relationship \citep{s2}. In
particular, $X_1$ is a direct cause of $X_2$ only if $X_1\to X_2$ (i.e., $X_1$ is a parent
of $X_2$), and $X_1$ is a (possibly indirect) cause of $X_2$ only if there is a directed
path from $X_1$ to $X_2$ (i.e., $X_1$ is an ancestor of $X_2$). So, while the interpretation of the directed edge in a CG is quite clear,
the second type of edge can represent different types of relations and, depending on how we interpret it in the graph, we say that we have different CG interpretations with different separation criteria, i.e. different ways of reading conditional independences from the graph, and different intuitive meaning behind
their edges. The three following interpretations are the best known in the literature. The first interpretation (LWF) was introduced by Lauritzen,
Wermuth and Frydenberg \citep{lw, f} to combine DAGs and undirected graphs (UGs). The second
interpretation (AMP), was introduced by Andersson, Madigan and Perlman, and also combines DAGs and UGs but with a Markov equivalence criterion that more closely resembles the one of DAGs \citep{amp}. The third interpretation,
the multivariate regression interpretation (MVR), was introduced by Cox
and Wermuth \citep{cw1, cw2} to combine DAGs and bidirected (covariance) graphs. 

This paper deals with chain graphs under the alternative Andersson-Madigan-Perlman
(AMP) interpretation \citep{amp,AMP2001}. AMP CGs are useful when we have a set of variables for which the internal
relations has no causal ordering, so the relations should be modelled as a Markov network, but
also a second set of variables that can be seen as causes for some of these variables
in the first set. The internal structure of the first set of variables can then be modelled
as a Markov network, creating a chain component in an AMP CG, and the causes as parents of
some of the variables in the chain component. Note that for AMP CGs the parents
only affect the direct children in the chain component, not all the nodes in the chain
component as in the case of LWF CGs. An example in medicine \citep{Sonntag2015} when such a
model might be appropriate is when we are modelling pain levels on different areas on
the body of a patient. The pain levels can then be seen as correlated “geographically”
over the body, and hence be modelled as a Markov network. Certain other factors do, however,
exist that alters the pain levels locally at some of these areas, such as the type of body
part the area is located on or if local anaesthetic has been administered in that area
and so on. These outside factors can then be modeled as parents affecting the pain
levels locally. AMP chain graphs are widely studied in different areas from applications in biology \citep{Sonntag2015}, to more advanced theoretical investigations \citep{r1, LPM2001,roverato05,roverato06, d,studeny09,PENA20141185,p2,Sonntag2015, Pena2016,PENA2016MAMP,p3,pena2018uai}.

Minimality is a desirable
property to ensure efficiency and usability e.g., see \citep{Pena2011}. Finding minimal separators is useful for learning and inference tasks \citep{ad,jv-mvr19}. Of course, finding these sets will take some effort, but the additional effort will be compensated by decreased computing time when using the corresponding independencies in learning and inference. Moreover, it will also increase the reliability of the results, because fewer data are needed to reliably compute a conditional dependence measure of lower order. For example, Acid and de Campos \citep{ACID2001} proposed a hybrid algorithm for learning Bayesian networks from data that uses minimal \textit{d}-separators. They showed that the use of minimal $d$-separating sets is clearly useful, not only with respect to the quality of the learned network but also in terms of time complexity of the proposed algorithm. In this paper, we address the problem of finding minimal separators in  AMP chain graphs and their applications in learning the structure of AMP CGs from data.

One important aspect of PGMs is the possibility of learning the structure of models directly
from sampled data. Two \textit{constraint-based} learning algorithms, that use a statistical analysis to test the presence of a
conditional independency, exist for learning AMP CGs:
(1) the \opc~algorithm \citep{penea12amp, PENA2016MAMP}, and (2) the answer set programming (ASP) algorithm \citep{Pena2016}.
In this paper, we show that the \opc~algorithm is \textit{order-dependent},
in the sense that the output can depend on the order in which the variables are given. We propose several modifications of the \opc~algorithm, i.e., \textbf{Stable} \textbf{PC}-like for \textbf{AMP} CGs (\spc), \textbf{Conservative} \textbf{PC}-like for \textbf{AMP} CGs (\cpc), and  a version that is both \textbf{Stable} and \textbf{Conservative} (\scpc) for learning the structure of AMP chain graphs under the faithfulness assumption that remove part or all of the order-dependence. 
\begin{figure}[!ht]
\centering
\captionsetup[subfigure]{font=footnotesize}
\centering
\subcaptionbox{Observational Data}[.33\textwidth]{%
\includegraphics[width=.25\linewidth]{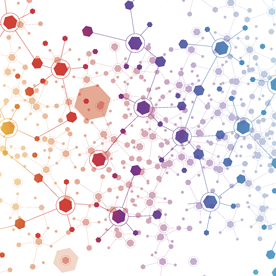}
}%
\subcaptionbox{Undirected Independence Graph Recovery}[.33\textwidth]{%
\begin{tikzpicture}[transform shape]
	\tikzset{vertex/.style = {shape=circle,inner sep=0pt,
  text width=5mm,align=center,
  draw=black,
  fill=white}}
\tikzset{edge/.style = {->,> = latex',thick}}
		\node[vertex,thick] (o) at  (1,.5) {$c$};
	\node[vertex,thick] (p) at  (-1,.5) {$e$};
	\node[vertex,thick] (q) at  (0,3) {$b$};
	\node[vertex,thick] (r) at  (-1.5,2) {$d$};
	\node[vertex,thick] (s) at  (1.5,2) {$a$};
	\node[vertex,thick] (u) at  (-2,-.5) {$f$};
	\draw[thick] (q) to (r);
	\draw[thick] (s) to (q);
	\draw[thick] (r) to (p);
	\draw[thick] (s) to (o);
	\draw[thick] (o) to (p);
	\draw[thick] (p) to (u);
	\draw[thick] (o) to (r);
\end{tikzpicture}}%
\subcaptionbox{Triangulation}[.33\textwidth]{\begin{tikzpicture}[transform shape]
	\tikzset{vertex/.style = {shape=circle,inner sep=0pt,
  text width=5mm,align=center,
  draw=black,
  fill=white}}
\tikzset{edge/.style = {->,> = latex',thick}}
	\node[vertex,thick] (e) at  (9,.5) {$c$};
	\node[vertex,thick] (d) at  (7,0.5) {$e$};
	\node[vertex,thick] (a) at  (8,3) {$b$};
	\node[vertex,thick] (b) at  (6.5,2) {$d$};
	\node[vertex,thick] (c) at  (9.5,2) {$a$};
	\node[vertex,thick] (cc) at  (6,-.5) {$f$};
	\draw[thick] (b) to (a);
	\draw[thick] (c) to (a);
	\draw[thick] (e) to (d);
	\draw[thick] (e) to (c);
	\draw[thick] (b) to (d);
	\draw[thick] (cc) to (d);
	\draw[thick] (c) to (b);
	\draw[thick] (b) to (e);
\end{tikzpicture}}
\vskip 1cm
\subcaptionbox{$p$-Separation Tree}[.33\textwidth]{\begin{tikzpicture}[scale=.75, transform shape,auto,node distance=1.5cm]
    \tikzset{edge/.style = {->,> = latex',thick}}
    \node[entity,thick] (node1) {$e$}
    [grow=up,sibling distance=3cm]
    child[grow=up,level distance=2cm,thick]  {node[attribute,thick] (ch1) {$c,d,e$}}
    child[grow=right,level distance=3cm,thick] {node[attribute,thick] {$e,f$}};
    \node[entity,thick] (rel1) [above right = of node1] {$c,d$}
    child[grow=up,level distance=2cm,thick] {node[attribute,thick] (ch2) {$b,c,d$}};
    \node[entity,thick] (node2) [above left = of rel1]	{$b,c$}
    child[grow=up,thick] {node[attribute,thick] {$a,b,c$}};
    \path[thick] (ch1) edge (rel1);
    \path[thick] (ch2) edge (node2);
    \end{tikzpicture}}%
\subcaptionbox{Local Skeleton Recovery}[.33\textwidth]{
\begin{tikzpicture}[scale=.75, transform shape]
	\tikzset{vertex/.style = {shape=circle,inner sep=0pt,
  text width=5mm,align=center,
  draw=black,
  fill=white}}
\tikzset{edge/.style = {->,> = latex',thick}}
		\node[vertex,thick] (d) at  (0,0) {$d$};
	\node[vertex,thick] (e) at  (0,1.5) {$e$};
	\node[vertex,thick] (c) at  (1.5,1.5) {$c$};
	\node[vertex,thick] (r) at  (2.5,1.5) {$e$};
	\node[vertex,thick] (s) at  (4,0) {$f$};
	\node[vertex,thick] (a) at  (0,5.5) {$a$};
	\node[vertex,thick] (m) at  (1.5,5.5) {$c$};
	\node[vertex,thick] (n) at  (0,4) {$b$};
	\node (o) at  (1,3) {$S_{bc}=\{a\}$};
	\node[vertex,thick] (x) at  (2.5,5.5) {$c$};
	\node[vertex,thick] (y) at  (4,5.5) {$b$};
	\node[vertex,thick] (z) at  (2.5,4) {$d$};
	\node (o) at  (3.5,3) {$S_{cd}=\{b\}$};
	\draw[thick] (d) to (e);
	\draw[thick] (d) to (c);
	\draw[thick] (c) to (e);
	\draw[thick] (r) to (s);
	\draw[thick] (a) to (m);
	\draw[thick] (a) to (n);
	\draw[thick] (x) to (y);
	\draw[thick] (y) to (z);
\end{tikzpicture}
}%
\subcaptionbox{Global Skeleton Recovery}[.33\textwidth]{%
\begin{tikzpicture}[transform shape]
	\tikzset{vertex/.style = {shape=circle,inner sep=0pt,
  text width=5mm,align=center,
  draw=black,
  fill=white}}
\tikzset{edge/.style = {->,> = latex',thick}}
		\node[vertex,thick] (o) at  (1,.5) {$c$};
	\node[vertex,thick] (p) at  (-1,.5) {$e$};
	\node[vertex,thick] (q) at  (0,3) {$b$};
	\node[vertex,thick] (r) at  (-1.5,2) {$d$};
	\node[vertex,thick] (s) at  (1.5,2) {$a$};
	\node[vertex,thick] (u) at  (-2,-.5) {$f$};
	\draw[thick] (q) to (r);
	\draw[thick] (s) to (q);
	\draw[thick] (r) to (p);
	\draw[thick] (s) to (o);
	\draw[thick] (o) to (p);
	\draw[thick] (p) to (u);
\end{tikzpicture}}
\vskip 1cm
\subcaptionbox{AMP CG Recovery}[.33\textwidth]{%
\begin{tikzpicture}[transform shape]
	\tikzset{vertex/.style = {shape=circle,inner sep=0pt,
  text width=5mm,align=center,
  draw=black,
  fill=white}}
\tikzset{edge/.style = {->,> = latex',thick}}
		\node[vertex,thick] (o) at  (1,.5) {$c$};
	\node[vertex,thick] (p) at  (-1,.5) {$e$};
	\node[vertex,thick] (q) at  (0,3) {$b$};
	\node[vertex,thick] (r) at  (-1.5,2) {$d$};
	\node[vertex,thick] (s) at  (1.5,2) {$a$};
	\node[vertex,thick] (u) at  (-2,-.5) {$f$};
	\draw[thick] (q) to (r);
	\draw[thick] (s) to (q);
	\draw[thick,edge] (r) to (p);
	\draw[thick] (s) to (o);
	\draw[thick,edge] (o) to (p);
	\draw[thick,edge] (p) to (u);
\end{tikzpicture}}%
\subcaptionbox{Largest Deflagged Graph Recovery}[.33\textwidth]{\begin{tikzpicture}[transform shape]
	\tikzset{vertex/.style = {shape=circle,inner sep=0pt,
  text width=5mm,align=center,
  draw=black,
  fill=white}}
\tikzset{edge/.style = {->,> = latex',thick}}
	\node[vertex,thick] (c) at  (9,.5) {$c$};
	\node[vertex,thick] (e) at  (7,0.5) {$e$};
	\node[vertex,thick] (b) at  (8,3) {$b$};
	\node[vertex,thick] (d) at  (6.5,2) {$d$};
	\node[vertex,thick] (a) at  (9.5,2) {$a$};
	\node[vertex,thick] (f) at  (6,-.5) {$f$};
	\draw[thick] (b) to (a);
	\draw[thick,edge] (a) to (c);
	\draw[thick,edge] (d) to (e);
	\draw[thick,edge] (c) to (e);
	\draw[thick,edge] (b) to (d);
	\draw[thick,edge] (e) to (f);
\end{tikzpicture}}
        \caption{An overview of \lcd~'s steps for learning the structure of the largest deflagged AMP CG from a faithful distribution.}
    \label{fig:my_label}
\end{figure}
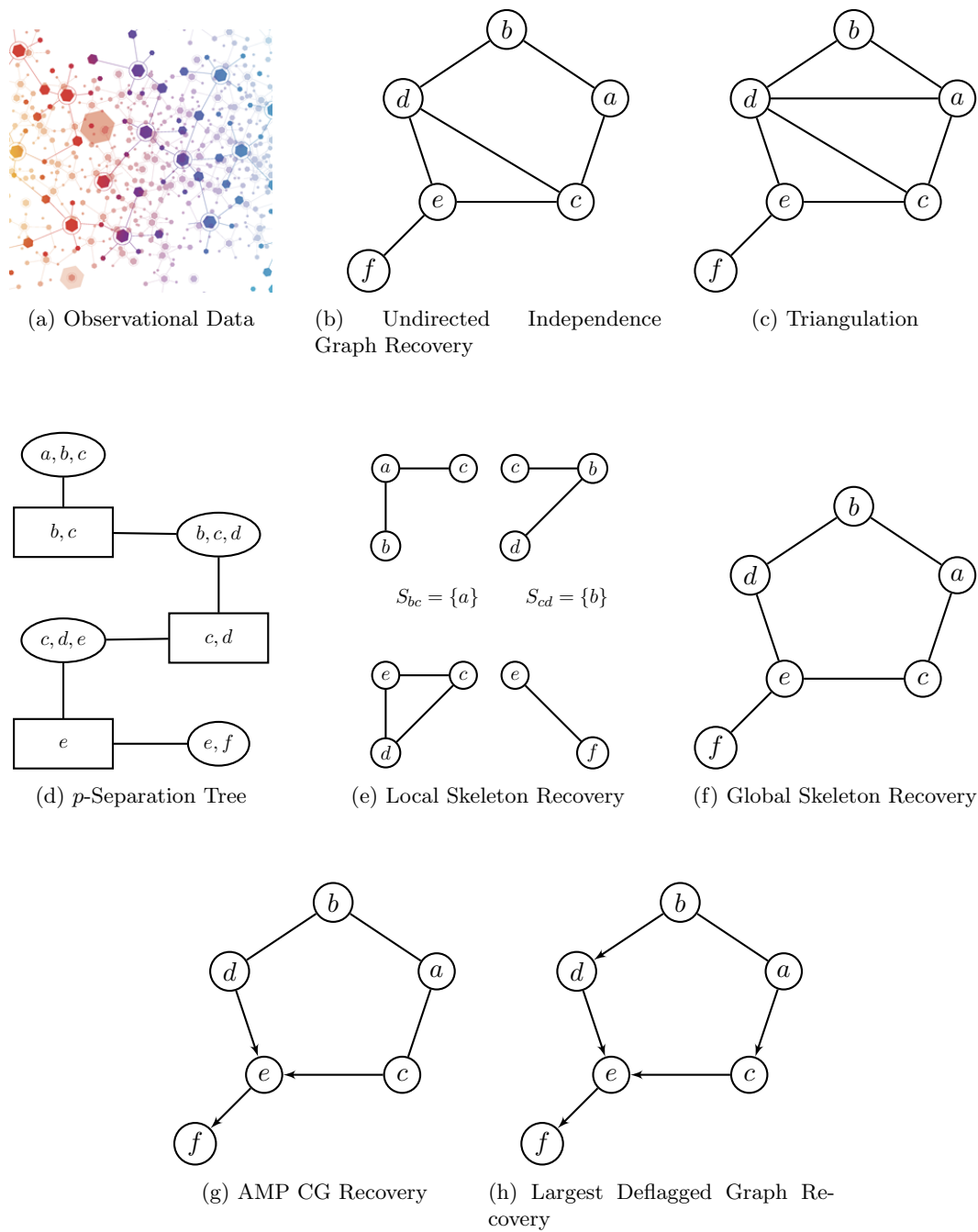

We use some of our findings regarding minimal separators in AMP CGs to prove the correctness of a new efficient algorithm for learning AMP chain graphs, called \textbf{L}earn \textbf{C}hain graphs via \textbf{D}ecomposition for \textbf{AMP} CGs (\lcd). Our proposed \lcd~algorithm, illustrated in Figure \ref{fig:my_label}, consists of five steps: (1) An undirected graphical model for the data is chosen. Any conditional independencies that hold under this model
will also hold under the selected chain graph, so this step serves to restrict the search space
in the third step. (2) A junction tree as a facilitator for decomposition of structure learning
is built from the triangulated graph obtained from the resulting graph at the end of step
(1). (3) Local skeletons are recovered in each individual node of the obtained separation
tree from the previous step. (4) The global skeleton is recovered by merging recovered local
skeletons from the previous step along with removing those edges that are deleted in any
local skeleton. (5) Arrowheads are added to some of the edges to obtain desired AMP chain
graph. The details of each step with related definitions  are provided later in the paper (Section \ref{main-alg}). This algorithm not only \textit{reduces complexity} and \textit{increases the power of computational independence tests} but also achieves a \textit{better quality} with respect to the learned structure.

The results of the experiments show that our proposed algorithm \textbf{L}earn \textbf{C}hain graphs via \textbf{D}ecomposition for \textbf{AMP} interpretation (\lcd)  consistently outperforms the ( {\sc Stable-}) \opc~algorithm\footnote{When we use parenthesis, we mean that what we write applies to both the original \opc~algorithm and the \spc~algorithm.}. Our proposed  algorithms, i.e., the \textbf{Stable} \textbf{PC}-like for \textbf{AMP} CGs (\spc) and \lcd~ are able to exploit the parallel computations for scaling up the task of learning AMP chain graphs. This will enable AMP chain graph discovery on large datasets. In fact, lower complexity, higher power of computational independence test, better learned structure quality, along with the ability of exploiting parallel computing, make our proposed algorithms more desirable and suitable for big data analysis when AMP chain graphs are being used.  Code for reproducing our results is available at {\color{blue}{\url{https://github.com/majavid/AMPCGs2019}}}.

Our main contributions are the following:

\begin{enumerate}
    \item We propose several polynomial time algorithms to solve the problem of finding minimal separating sets in AMP chain graphs (Section \ref{sec:findingminimals}).
    \item We show that the original \opc~algorithm \citep{penea12amp} is \textit{order-dependent}, in the sense that the output can depend on the order in which the variables are given.  Then, we propose modifications of the \opc~algorithm, i.e., \textbf{Stable} \textbf{PC}-like for \textbf{AMP} (\spc), \textbf{Conservative} \textbf{PC}-like for \textbf{AMP} (\cpc), and \scpc~for learning the structure of AMP chain graphs under the faithfulness assumption that remove part or all of the order-dependence (Section \ref{sec:pcalg}).
    \item We present a computationally feasible algorithm for learning the structure of AMP chain graphs via decomposition, called \lcd~, that reduces complexity and increase the power of computational independence tests (Section \ref{main-alg}).
    \item We compare the performance of our algorithms with that of the \opc~algorithm proposed in \citep{penea12amp}, in the Gaussian and discrete cases. We empirically show that our modifications of the \opc~algorithm achieve output of better quality than the original \opc~algorithm, especially in high-dimensional settings. We also show that our decomposition based algorithm, i.e., the \lcd~algorithm outperforms the ( {\sc Stable-}) \opc~algorithm in our experiments (Section \ref{evaluation}).
    \item We release supplementary material including data and an \textsf{R} package that implements the proposed algorithms.
\end{enumerate}

\section{Basic Definitions and Concepts}
In this paper, we consider graphs containing both directed (of the form $a \to b$ or, simply, $(a,b)$) and undirected (of the form $a-b$ or, simply, $\{a,b\}$) edges
and largely use the terminology of \citep{AMP2001}, where the reader can also find further
details. Below we briefly list some of the central concepts used in this paper.

	If $A\subseteq V$ is a subset of the vertex set in a graph $G=(V,E)$, the  \textit{induced subgraph} $G_A=(A,E_A)$ is a graph in which the edge set $E_A=E\cap (A\times A)$ is obtained from $G$ by keeping edges with both endpoints in $A$.

    If there is an arrow from $a$ pointing towards $b$, $a$ is said to be a \textit{parent} 
	of $b$. The set of parents of $b$ is denoted as $pa(b)$. If there is an undirected edge between $a$ and $b$, $a$ and $b$ are said to be \textit{adjacent} or \textit{neighbors}. The set of neighbors of a vertex $a$ is denoted as $ne(a)$. The expressions $pa(A)$ and $ne(A)$ denote the collection of  
	parents and neighbors of vertices in $A$ that are not themselves 
	elements of $A$. The \textit{boundary} $bd(A)$ of a subset $A$ of vertices is the set of vertices in $V\setminus A$ that are parents or neighbors to vertices in $A$. The \textit{closure} of $A$ is $cl(A)=bd(A)\cup A$.

	A \textit{directed path} of length $n$ from $a$ to $b$ is a sequence $a=a_0,\dots , a_n=b$ of 
	distinct vertices such that $(a_i,a_{i+1})\in E$, for all $i=0,\dots ,n-1$. 
	(A \textit{semidirected path} of length $n$ from $a$ to $b$ is a sequence $a=a_0,\dots , a_n=b$ of 
	distinct vertices such that either $(a_i,a_{i+1})$ or $\{a_i, a_{i+1}\} \in E$, for all $i=0,\dots ,n-1$.)
	A \textit{chain} of length $n$ from $a$ to $b$ is a sequence $a=a_0,\dots , a_n=b$ of 
	distinct vertices such that $(a_i,a_{i+1})\in E$, or $(a_{i+1},a_i)\in E$, or $\{a_i,a_{i+1}\}\in E$, for all $i=0,\dots ,n-1$. 
	A vertex $\alpha$ is said to be an \emph{ancestor} of a vertex $\beta$ if there is a directed path $\alpha \to \dots \to \beta$ from $\alpha$ to $\beta$. We define the \textit{smallest ancestral set} containing $A$ as $An(A):=an(A)\cup A$. A vertex $\alpha$ is said to be \emph{anterior} to a vertex $\beta$ if there is a chain $\mu$ from $\alpha$ to $\beta$ on which every edge is either of the form $\gamma-\delta$, or $\gamma \to\delta$ with $\delta$ between $\gamma$ and $\beta$, or $\alpha=\beta$; that is, there are no edges $\gamma\gets\delta$ pointing toward $\alpha$. We apply this definition to sets: $ant(X) = \{\alpha | \alpha \textrm{ is an anterior of } \beta \textrm{ for some } \beta \in X\}$.

A \textit{partially directed cycle} (or semi-directed cycle) in a graph $G$ is a sequence of $n$ distinct vertices $v_1,v_2,\dots,v_n (n\ge 3)$, and $v_{n+1}\equiv v_1$, such that

(a) for all $i (1\le i\le n)$ either $v_i-v_{i+1}$ or $v_i\to v_{i+1}$, and

(b) there exists a $j (1\le j\le n)$ such that $v_j\to v_{j+1}$.

An \textit{AMP chain graph} is a graph in which there are no partially directed cycles. The \textit{chain components} $\mathcal{T}$ of a chain graph are the connected components of the undirected
graph obtained by removing all directed edges from the chain graph. We define the \textit{smallest coherent set} containing $A$ as $Co(A):=\cup_\tau\{\tau\in \mathcal{T}|\tau\cap A \ne \emptyset \}$. Let $\overline{G}$ be obtained by deleting all directed edges of $G$; for $A\subseteq V$ the \textit{extended subgraph} G[A] is defined by $G[A]:=G_{An(A)}\cup \overline{G}_{Co(An(A))}$.

A triple of vertices $\{X, Y, Z\}$ is said to form a \textit{flag} in CG if the 
induced subgraph $CG_{X\cup Y\cup Z}$ is $X \to Y-Z$ or $X-Y\gets Z$. A triple of vertices $\{X, Y, Z\}$ is said to form a \textit{triplex} in CG if the 
induced subgraph $CG_{X\cup Y\cup Z}$ is either $X \to Y-Z$, $X \to Y \gets Z$, or 
$X-Y\gets Z$. A triplex is \textit{augmented} by adding the $X-Z$ edge. A set of four 
vertices $\{X, A,B,Y\}$ is said to form a \textit{bi-flag} if the edges $X \to A$, $Y\to B$, and $A-B$ are present in the induced subgraph over $\{X, A, B, Y\}$. A bi-flag 
is augmented by adding the edge $X-Y$. A \textit{minimal complex} (or simply a complex) in a chain graph is an induced subgraph of the form $a\to v_1-\cdots \cdots-v_r\gets b$. The \textit{augmented CG} $G^a$ is the undirected graph 
formed by augmenting all triplexes and bi-flags in CG and replacing all 
directed edges with undirected edges (see Fig. \ref{augmented}). The \textit{skeleton} (underlying graph) of a CG $G$ is obtained from $G$ by changing all directed edges of $G$ into undirected edges.  Vertex $Y$ is an \textit{unshielded collider} (or V-structure) in a DAG $G$ if $G$
contains the induced subgraph $U\to Y\gets V$.

\begin{figure}
    \centering
	\[\begin{tikzpicture}[transform shape]
	\tikzset{vertex/.style = {shape=circle,draw,minimum size=1.5em}}
	\tikzset{edge/.style = {->,> = latex'}}
	\node[vertex,thick] (a) at  (2,1) {$Z$};
	\node[vertex,thick] (b) at  (1,0) {$Y$};
	\node[vertex,thick] (c) at  (0,1) {$X$};
	\node (d) at  (4,-1) {$(a)$};
	\draw[edge] (a) to (b);
	\draw (c) to (b);
	
	\node[vertex,thick] (e) at  (5,1) {$Z$};
	\node[vertex,thick] (f) at  (4,0) {$Y$};
	\node[vertex,thick] (g) at  (3,1) {$X$};
	\draw[edge] (g) to (f);
	\draw (e) to (f);
	
	\node[vertex,thick] (h) at  (8,1) {$Z$};
	\node[vertex,thick] (i) at  (7,0) {$Y$};
	\node[vertex,thick] (j) at  (6,1) {$X$};
	\draw[edge] (h) to (i);
	\draw[edge] (j) to (i);
	
	\node[vertex,thick] (k) at  (12,1) {$Z$};
	\node[vertex,thick] (l) at  (11,0) {$Y$};
	\node[vertex,thick] (m) at  (10,1) {$X$};
	\node (n) at  (11,-1) {$(b)$};
	\draw (k) to (l);
	\draw (l) to (m);
	\draw (m) to (k);
	
	\node[vertex,thick] (o) at  (5,-2) {$A$};
	\node[vertex,thick] (p) at  (3,-4) {$Y$};
	\node[vertex,thick] (q) at  (3,-2) {$X$};
	\node[vertex,thick] (r) at  (5,-4) {$B$};
	\node (s) at (4, -5) {$(c)$};
	\node (t) at (2.75, -3) {$?$};
	\draw[edge] (q) to (o);
	\draw (o) to (r);
	\draw[edge] (p) to (r);
	\draw[dashed] (q) to (p);
	
	\node[vertex,thick] (u) at  (9,-2) {$A$};
	\node[vertex,thick] (v) at  (7,-4) {$Y$};
	\node[vertex,thick] (x) at  (7,-2) {$X$};
	\node[vertex,thick] (w) at  (9,-4) {$B$};
	\node (y) at (8, -5) {$(d)$};
	\draw (u) to (v);
	\draw (v) to (x);
	\draw (x) to (w);
	\draw (w) to (u);
	\draw (u) to (x);
	\draw (v) to (w);
	
	\end{tikzpicture}\]
    \caption{(a) Triplexes and (b) the corresponding augmented triplex, (c) the four configurations that define the bi-flag; (d) the corresponding augmented bi-flag. The ``?" indicates that either $X-Y\in G$, $X\to Y \in G$, $Y\to X\in G$, or $X$ and $Y$ are not adjacent in $G$.
}
    \label{augmented}
\end{figure}
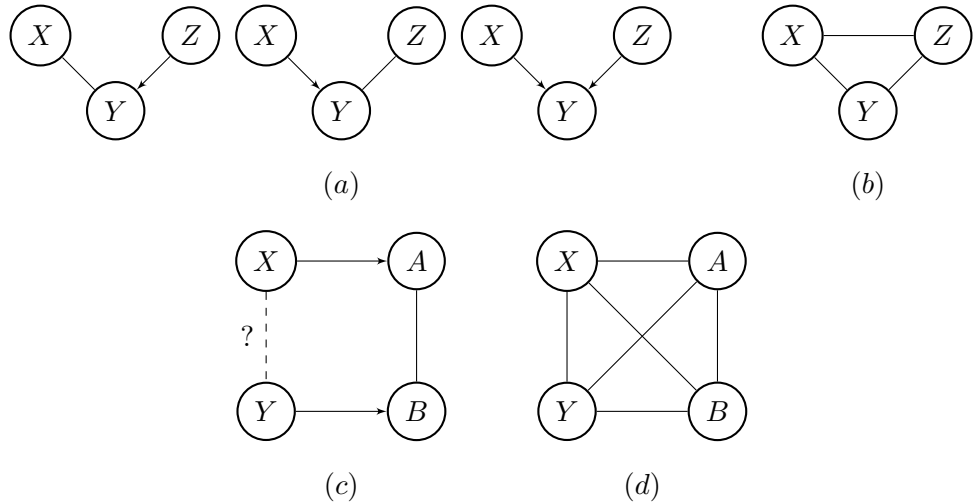

\begin{definition}\label{gMarkovAMP}
    \emph{(Global Markov property for AMP chain graphs)} 
	For any triple
	$(A, B,S)$ of disjoint subsets of $V$ such that $S$ separates $A$ from $B$
	in $(G[A\cup B\cup S])^a$, in the augmented graph  of the extended subgraph of $A\cup B\cup S$, we have $A \!\perp\!\!\!\perp B | S$ (or $\langle A,B|S\rangle$) i.e., $A$ is independent of $B$ given $S$.
\end{definition} 

An equivalent pathwise separation criterion that identifies all valid conditional independencies under the AMP Markov property was introduced in \citep{LPM2001}:
\begin{definition}\label{pseparation}
\emph{(The pathwise $p$-separation criterion for AMP chain graphs)}
    A node $B$ in a chain $\rho$ in an AMP CG $G$ is called a \emph{triplex node}  in $\rho$ if $A\to B\gets C, A\to B - C, \textrm{ or } A - B \gets C$ is a subchain of $\rho$. Moreover, $\rho$ is said to be \emph{$Z$-open} with $Z\subseteq V$ when
    \begin{itemize}
        \item every triplex node in $\rho$ is in $An(Z)$, and
        \item every non-triplex node $B$ in $\rho$ is outside $Z$, unless $A - B - C$ is a subchain of $\rho$ and $pa_G(B) \setminus Z \ne \emptyset$.
    \end{itemize}
    Let $X, Y\ne \emptyset$ and $Z$ (may be empty) denote three disjoint subsets of $V$. When there is no \textit{Z}-open chain in an AMP CG $G$ between a node in $X$ and a node in $Y$, we say that $X$ is separated from $Y$ given $Z$ in $G$ and denote it as $X\!\perp\!\!\!\perp  Y|Z$.
\end{definition}

Theorem 4.1 in \citep{LPM2001} establishes
the equivalence of the \textit{p}-separation criterion and the augmentation criterion
occurring in the AMP global Markov property for CGs.

\begin{example}
Consider the AMP CG $G$ in Fig. \ref{GMarkov}(a). The global Markov property of AMP chain graphs implies that $X \!\perp\!\!\!\perp Y | A$ (see Fig. \ref{GMarkov}). There is no \textit{A}-open chain in the AMP CG $G$ between $X$ and $Y$ because the only chain between $X$ and $Y$ i.e., $X\to A-B\gets Y$ is blocked at $B$  $(B\textrm{ is a triplex node in the chain and } B\not\in An(A))$.
\begin{figure}
    \centering
	\[\begin{tikzpicture}[transform shape]
	\tikzset{vertex/.style = {shape=circle,draw,minimum size=1.5em}}
	\tikzset{edge/.style = {->,> = latex'}}
	\node[vertex,thick] (o) at  (5,-2) {$A$};
	\node[vertex,thick] (p) at  (3,-4) {$Y$};
	\node[vertex,thick] (q) at  (3,-2) {$X$};
	\node[vertex,thick] (r) at  (5,-4) {$B$};
	\node (s) at (4, -5) {$(a)$};
	\draw[edge,thick] (q) to (o);
	\draw[thick] (o) to (r);
	\draw[edge,thick] (p) to (r);
	
	\node[vertex,thick] (u) at  (9,-2) {$A$};
	\node[vertex,thick] (v) at  (7,-4) {$Y$};
	\node[vertex,thick] (x) at  (7,-2) {$X$};
	\node (y) at (8, -5) {$(b)$};
	\draw[edge,thick] (x) to (u);
	
	\node[vertex,thick] (a) at  (13,-2) {$A$};
	\node[vertex,thick] (b) at  (11,-4) {$Y$};
	\node[vertex,thick] (c) at  (11,-2) {$X$};
	\node[vertex,thick] (d) at  (13,-4) {$B$};
	\node (e) at (12, -5) {$(c)$};
	\draw[thick] (a) to (d);
	
	\node[vertex,thick] (f) at  (17,-2) {$A$};
	\node[vertex,thick] (g) at  (15,-4) {$Y$};
	\node[vertex,thick] (h) at  (15,-2) {$X$};
	\node[vertex,thick] (i) at  (17,-4) {$B$};
	\node (j) at (16, -5) {$(d)$};
	\draw[thick] (f) to (i);
	\draw[thick,edge] (h) to (f);
	
	\node[vertex,thick] (k) at  (11,-6) {$A$};
	\node[vertex,thick] (l) at  (9,-8) {$Y$};
	\node[vertex,thick] (m) at  (9,-6) {$X$};
	\node[vertex,thick] (n) at  (11,-8) {$B$};
	\node (o) at (10, -9) {$(e)$};
	\draw[thick] (k) to (m);
	\draw[thick] (k) to (n);
	\draw[thick] (m) to (n);
	
	\end{tikzpicture}\]
    \caption{(a) The AMP CG $G$, (b) $An(X\cup Y\cup A)$, (c) the undirected edges in $Co(An(X\cup Y\cup A))$, (d) $G[X\cup Y\cup A]$, and (e) $(G[X\cup Y\cup A])^a$. 
}
    \label{GMarkov}
\end{figure}
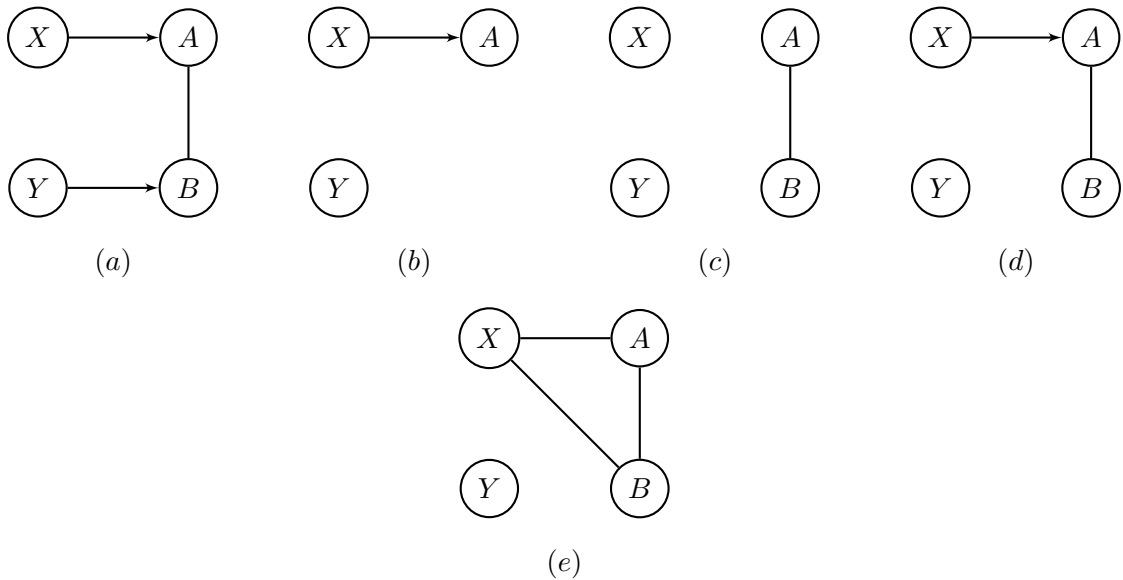
\end{example}

We say that two AMP CGs $G$ and $H$ are \textit{Markov equivalent}
or that they are in the same \textit{Markov equivalence class} if they induce the same  
conditional independence restrictions.  Two chain graphs $G$ and $H$ are Markov equivalent if and only if they have the same skeletons and the same triplexes \citep{AMP2001}. Two LWF chain graphs $G$ and $H$ are Markov equivalent if and only if they have the same skeletons and the same minimal complexes \citep{f}.  Two DAGs $G$ and $H$ are Markov equivalent if and only if they have the same skeletons and the same unshielded colliders \citep{pearl88}. The condition for AMP Markov equivalence of CGs more closely
resembles that for DAG Markov equivalence than does the condition for LWF Markov
equivalence of CGs, in the sense that triplexes involve only three vertices, while complexes can involve arbitrarily many vertices. 

We say that AMP chain graphs $G$ and $H$ belong to the same \textit{strong Markov equivalent class} iff $G$ and $H$ are Markov equivalent and contain the same flags. An AMP CG $G^*$ is said to be the \textit{AMP essential graph} of its Markov equivalence class iff for every directed edge $A\to B$ that exists in $G^*$ there exists no AMP CG $H$ s.t. $G^*$ and $H$ are Markov equivalent and $A\gets B$ is in $H$. An AMP CG $G^*$ is said to be the \textit{largest deflagged graph} of its Markov
equivalence class iff there exists no other AMP CG $H$ s.t. $G^*$ and $H$ are Markov equivalent and either $H$ contains fewer flags than $G^*$ or $G^*$ and $H$ belong to the same strong Markov equivalence class but $H$ contains more undirected edges. Any largest deflagged graph or AMP
essential graph are AMP CGs and both of these have been proven to be unique for the Markov equivalence class they represent \citep{roverato06,andersson2006}.

Let $\bar{G}_V = (V, \bar{E}_V)$ denote an undirected graph where $\bar{E}_V$ is a set of undirected edges.  For a subset $A$ of $V$, let $\bar{G}_A= (A, \bar{E}_A)$ be the subgraph induced by $A$
and $\bar{E}_A = \{e\in \bar{E}_V | e\in A\times A\} = \bar{E}_V\cap (A\times A)$. An undirected graph is called \textit{complete} if any pair of vertices is connected by an edge. For an undirected graph, we say that vertices $u$ and $v$ are separated by a set of vertices $Z$ if each path between $u$ and $v$ passes through $Z$. We say that two distinct vertex sets $X$ and $Y$ are separated by $Z$ if and
only if $Z$ separates every pair of vertices $u$ and $v$ for any $u\in X$ and $v\in Y$. We say that an undirected graph $\bar{G}_V$ is
an \textit{undirected independence graph} (UIG) for  CG $G$ if the fact that a set $Z$ separates $X$ and $Y$ in $\bar{G}_V$ implies that $Z$
$p$-separates $X$ and $Y$ in $G$. Note that the augmented graph derived from  CG $G$, $(G)^a$, is an undirected independence graph for $G$.  We say that $\bar{G}_V$ can be decomposed into subgraphs $\bar{G}_A$ and $\bar{G}_B$ if
\begin{itemize}
	\item[(1)] $A\cup B=V$, and
	\item[(2)] $C=A\cap B$ separates $V\setminus A$ and $V\setminus B$ in $\bar{G}_V$.
\end{itemize}
The above decomposition does not require that the separator $C$ be complete, which is required for weak decomposition defined in \citep{l}. In this paper, we show that a
problem of learning the structure of  CG can also be decomposed into problems for its decomposed subgraphs even if
the separator is not complete.

A \textit{triangulated (chordal)} graph is an undirected graph in which all cycles of four or more vertices have a chord, which is an edge that is not part of the cycle but connects two vertices of the cycle  (see, for example, Figure \ref{fig:AMPUIGChordal}). For an
undirected graph $\bar{G}_V$ which is not triangulated, we can add extra (``fill-in") edges to it such that it becomes a triangulated
graph, denoted by $\bar{G}_V^t$.
\begin{figure}[ht]
    \centering
    \[\begin{tikzpicture}[transform shape]
	\tikzset{vertex/.style = {shape=circle,inner sep=0pt,
  text width=5mm,align=center,
  draw=black,
  fill=white}}
\tikzset{edge/.style = {->,> = latex',thick}}
	\node[vertex,thick] (o) at  (1,.5) {$c$};
	\node[vertex,thick] (p) at  (-1,.5) {$e$};
	\node[vertex,thick] (q) at  (0,3) {$b$};
	\node[vertex,thick] (r) at  (-1.5,2) {$d$};
	\node[vertex,thick] (s) at  (1.5,2) {$a$};
	\node[vertex,thick] (u) at  (-2,-.5) {$f$};
	\node (t) at (0,-1) {$(a)$};
	\draw[edge] (q) to (r);
	\draw[thick] (s) to (q);
	\draw[edge] (r) to (p);
	\draw[edge] (s) to (o);
	\draw[edge] (o) to (p);
	\draw[edge] (p) to (u);
	
	\node[vertex,thick] (i) at  (5,0.5) {$c$};
	\node[vertex,thick] (j) at  (3,0.5) {$e$};
	\node[vertex,thick] (k) at  (4,3) {$b$};
	\node[vertex,thick] (l) at  (2.5,2) {$d$};
	\node[vertex,thick] (m) at  (5.5,2) {$a$};
	\node[vertex,thick] (ii) at  (2,-.5) {$f$};
	\node (n) at (4,-1) {$(b)$};
	\draw[thick] (l) to (k);
	\draw[thick] (j) to (i);
	\draw[thick] (m) to (k);
	\draw[thick] (m) to (i);
	\draw[thick] (j) to (l);
	\draw[thick] (j) to (ii);
	\draw[thick] (i) to (l);
	
	\node[vertex,thick] (e) at  (9,.5) {$c$};
	\node[vertex,thick] (d) at  (7,0.5) {$e$};
	\node[vertex,thick] (a) at  (8,3) {$b$};
	\node[vertex,thick] (b) at  (6.5,2) {$d$};
	\node[vertex,thick] (c) at  (9.5,2) {$a$};
	\node[vertex,thick] (cc) at  (6,-.5) {$f$};
	\node (f) at (8,-1) {$(c)$};
	\draw[thick] (b) to (a);
	\draw[thick] (c) to (a);
	\draw[thick] (e) to (d);
	\draw[thick] (e) to (c);
	\draw[thick] (b) to (d);
	\draw[thick] (cc) to (d);
	\draw[thick] (c) to (b);
	\draw[thick] (b) to (e);
	\end{tikzpicture}\]
    \caption{(a) An AMP CG $G$. (b) The augmented graph $G^a$, which is also an undirected independence graph. (c) The triangulated graph $(G^a)^t$.}
    \label{fig:AMPUIGChordal}
\end{figure}
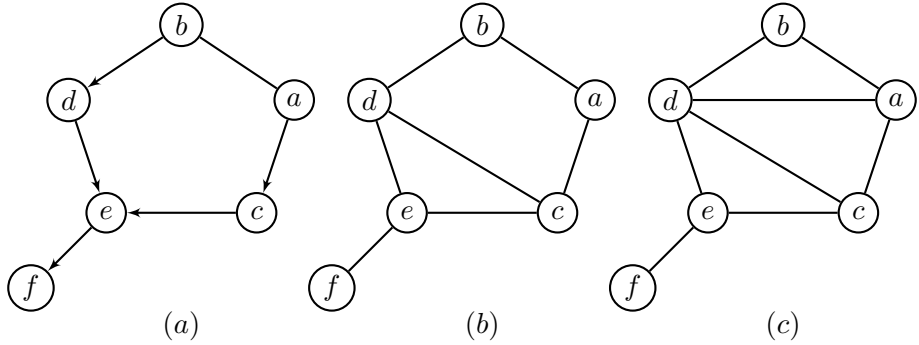

In this paper, we assume that all independencies of a
probability distribution of variables in $V$ can be checked by $p$-separations of $G$, called the faithfulness assumption \citep{sgs}. The faithfulness assumption means that all independencies and conditional independencies among variables can be represented by $G$. 

The global skeleton is an undirected graph obtained by dropping direction of  CG. A local skeleton for a subset $A$ of variables is an undirected subgraph for $A$ in which
the absence of an edge $u\erelbar{00} v$ implies that there is a subset $S$ of $A$ such that $u\!\perp\!\!\!\perp v|S$. Now, we introduce the notion of \textit{$p$-separation trees}, which is used to facilitate the representation of the decomposition. The concept is similar to the junction tree of cliques and the
independence tree introduced for DAGs as $d$-separation trees in \citep{xie}. Let $C = \{C_1, \dots, C_H \}$ be a collection of distinct variable sets such that for $h = 1,\dots ,H, C_h\subseteq V$.
Let $T$ be a tree where each node corresponds to a distinct variable set in $C$, to be displayed as an oval (see, for example, Figure \ref{fig:pseptree}). An undirected edge $e = \{C_i,C_j\}$ connecting nodes $C_i$ and $C_j$ in $T$ is labeled with a separator $S = C_i\cap C_j$, which is displayed as a rectangle.
Removing an edge $e$ or, equivalently, removing a separator $S$ from $T$ splits $T$ into two subtrees
$T_1$ and $T_2$ with node sets $C_1$ and $C_2$ respectively. We use $V_i$ to denote the union of the
vertices contained in the nodes of the subtree $T_i$ for $i = 1,2$.
\begin{figure}[ht]
    \centering
    \[\begin{tikzpicture}[auto,node distance=1.5cm]
    \tikzset{edge/.style = {->,> = latex',thick}}
    \node[entity,thick] (node1) {$e$}
    [grow=up,sibling distance=3cm]
    child[grow=up,level distance=2cm,thick]  {node[attribute,thick] (ch1) {$c,d,e$}}
    child[grow=right,level distance=3cm,thick] {node[attribute,thick] {$e,f$}};
    \node[entity,thick] (rel1) [above right = of node1] {$c,d$}
    child[grow=up,level distance=2cm,thick] {node[attribute,thick] (ch2) {$b,c,d$}};
    \node[entity,thick] (node2) [above left = of rel1]	{$b,c$}
    child[grow=up,thick] {node[attribute,thick] {$a,b,c$}};
    \path[thick] (ch1) edge (rel1);
    \path[thick] (ch2) edge (node2);
    \end{tikzpicture}\]
    \caption{The $p$-separation tree of CG $G$ in Figure \ref{fig:AMPUIGChordal}.}
    \label{fig:pseptree}
\end{figure}
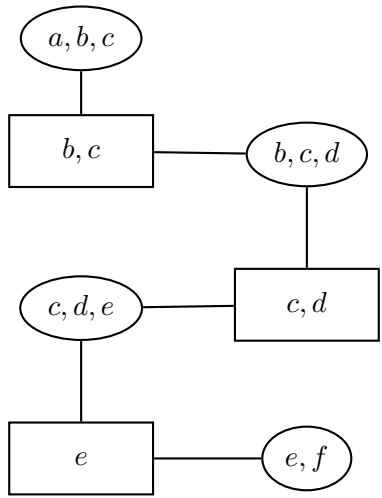

Notice that a separator is defined in terms of a tree whose nodes consist of variable sets, while
the $p$-separator is defined based on  chain graph. In general, these two concepts are not related, though for a $p$-separation tree its separator must be some corresponding $p$-separator in the underlying AMP chain graph. The definition of $p$-separation trees for AMP chain graphs is similar to that of junction trees of cliques,
see \citep{cdls,l}. Actually, it is not difficult to see that a junction tree
of  chain graph $G$ is also a $p$-separation tree. However, as in \citep{mxg}, we point out two differences here: (a) a $p$-separation tree is defined with $p$-separation and it does not require that every node be a clique or
that every separator be complete on the augmented graph; (b) junction trees are mostly used in inference engines, while our interest in $p$-separation trees is mainly derived from their power in facilitating the decomposition of structural learning.

Given an undirected graph $G=(V,E)$, a subset $S\subseteq V$ that does not contain $a$ or $b$ is said to be an \textit{$(a,b)$-separator} if all paths from $a$ to $b$ intersect $S$. A set $S$ of nodes that separates a given pair of nodes such that no proper subset of $S$ separates that pair is called a minimal separator.	Note that removing an $(a,b)$-separator disconnects a graph into two connected components, one containing $a$, and another containing $b$. Conversely, if a set $S$ disconnects a graph into a connected component including $a$ and another connected component including $b$, then $S$ is an $(a,b)$-separator. Two disjoint vertex subsets $A$ and $B$ of $V$ are adjacent if there is at least one pair of adjacent vertices $u\in A$ and $v\in B$. Let $A$ and $B$ be two disjoint non-adjacent subsets of $V$. Similarly, we define an $(A,B)$-separator to be any subset of $V\setminus(A\cup B)$ whose removal separates $A$ and $B$ in
distinct connected components. A minimal $(A,B)$-separator does not contain any other
\textit{$(A,B)$-separator}.

\section{Finding Minimal Separators in AMP Chain Graphs}\label{sec:findingminimals}
In this section we propose and solve an optimization problem related to the separation in AMP chain graphs. The basic problem is formulated as follows: given a pair of non-adjacent nodes, $x$ and
$y$, in an AMP chain graph, $G$, find a minimal set of nodes that separates $x$ and
$y$. We analyze several versions of this problem
and offer polynomial time algorithms for each. Apart from the possible theoretical interest that these problems may have \citep{tpp,ad}, generally, the solution
to the basic problem (Problem \ref{problem2amp}) represents the minimum information i.e., minimal set of variables, whose values we have to know
in order to break the mutual influence between two sets of variables, either in the
absence of any other information (Problem \ref{problem5amp}, \ref{problem6amp}), or
in the presence of some previous knowledge (Problem \ref{problem1amp}, \ref{problem3amp}, \ref{problem4amp}).
These include the following problems: 

\begin{problem}\label{problem1amp} (test for minimal separation) Given two non-adjacent nodes $X$ and $Y$ in an AMP chain graph $G$ and a set $Z$ that separates $X$ from $Y$, test if $Z$ is minimal i.e., no proper subset of $Z$ separates $X$ from $Y$.
\end{problem}
\begin{problem}\label{problem2amp} (minimal separation) Given two non-adjacent nodes $X$ and $Y$ in an AMP chain graph $G$, find a minimal separating set between $X$ and $Y$, namely, find a set $Z$ such that $Z$, and no proper subset of $Z$, separates $X$ from $Y$.
\end{problem}
\begin{problem}\label{problem3amp} (restricted separation) Given two non-adjacent nodes $X$ and $Y$ in an AMP chain graph $G$ and a set $S$ of nodes not containing $X$ and $Y$, find a subset $Z$ of $S$ that separates $X$ from $Y$.
\end{problem}
\begin{problem}\label{problem4amp} (restricted minimal separation) Given two non-adjacent nodes $X$ and $Y$ in an AMP chain graph $G$ and a set $S$ of nodes not containing $X$ and $Y$, find a subset $Z$ of $S$ which is minimal and separates $X$ from $Y$.
\end{problem}
\begin{problem}\label{problem5amp} (minimal separation of two disjoint non-adjacent sets) Given two disjoint non-adjacent sets $X$ and $Y$ in an AMP chain graph $G$, find a minimal separating set between $X$ and $Y$, namely, find a set $Z$ such that $Z$, and no proper subset of $Z$, separates $X$ from $Y$.
\end{problem}
\begin{problem}\label{problem6amp} (enumeration of all minimal separators) Given two non-adjacent nodes (or disjoint subsets) $X$ and $Y$ in an AMP chain graph $G$, enumerate all minimal separating sets between $X$ and $Y$.
\end{problem}

We prove that it is possible
to transform our problem into a separation problem, where the undirected graph in which we have to look for the minimal set separating $X$ from $Y$ depends only on $X$ and $Y$.  For each above mentioned problem, we propose and analyze an algorithm that, taking into account the previous results, solves it.

\subsection{Main Theorem: Minimal Separators in AMP Chain Graphs}
In this subsection we prove that it is possible to transform our problem into a separation problem, where the undirected graph in which we have to look for the minimal set separating $X$ from $Y$ depends only on $X$ and $Y$. Later, in the next subsections, we will apply this result to developing an efficient algorithm that solves our problems.

The next proposition shows that if we want to test a separation relationship between two disjoint sets of nodes $X$ and $Y$ in an AMP chain graph, where the separating set is included in the anterior set of $X\cup Y$, then we can test this relationship in a smaller AMP chain graph, whose set of nodes is formed only by the anteriors of $X$ and $Y$.
\begin{proposition}\label{prop1amp}
	Given an AMP chain graph $G=(V,E)$. Consider that $X, Y,$ and $Z$ are three disjoint subsets of $V,$ $Z\subseteq ant(X\cup Y)$, and $H=G_{ant(X\cup Y)}$ is the subgraph of $G$ induced by $ant(X\cup Y)$. Then ${\langle X, Y | Z\rangle}_G \Leftrightarrow {\langle X, Y | Z\rangle}_H.$
\end{proposition}
\begin{proof}
	($\Rightarrow$) The necessary condition is obvious, because a separator in a graph is also a separator in all of its subgraphs.
	
	($\Leftarrow$) Since $bd(ant(X\cup Y))=\emptyset$, so $Co(An(ant(X\cup Y)))=ant(X\cup Y)$. Let ${\langle X, Y | Z\rangle}_H$ and $Z\subseteq ant(X\cup Y)$, then $Co(An(X\cup Y\cup Z))\subseteq ant(X\cup Y)$. Consider that ${\langle X, Y \not| Z\rangle}_G$. This means that $X$ is not separated from $Y$ given $Z$ in $(G[X\cup Y\cup Z])^a$, which is a subgraph of $(G[ant(X\cup Y)])^a$.  In other words, there is a chain $C$ between $X$ and $Y$ in $H^a=(G[ant(X\cup Y)])^a=(G_{ant(X\cup Y)})^a$ that bypasses $Z$. Once again using $Z\subseteq ant(X\cup Y)$, we obtain that $X$ and $Y$ are not separated by $Z$ in $H$, in contradiction to the assumption ${\langle X, Y | Z\rangle}_H$. Therefore, it has to be ${\langle X, Y | Z\rangle}_G$.
\end{proof}

The following proposition establishes the basic result necessary to solve our optimization problems.
\begin{proposition}\label{prop2amp}
	Given an AMP chain graph $G=(V,E)$. Consider that $X, Y,$ and $Z$ are three disjoint subsets of $V$ such that ${\langle X, Y | Z\rangle}$ and ${\langle X, Y \not| Z'\rangle}, \forall Z'\subsetneq Z$. Then $Z\subseteq ant(X\cup Y).$
\end{proposition}
\begin{proof}
	Suppose that $Z\not\subseteq ant(X\cup Y)$. Define $Z'=Z\cap ant(X\cup Y)$. Then, by assumption we have ${\langle X, Y \not| Z'\rangle}$. Since $Z'\subseteq ant(X\cup Y)$, it is obvious that $Co(An(X\cup Y\cup Z'))\subseteq ant(X\cup Y)$. So, $X$ and $Y$ are not separated by $Z'$ in $(G[X\cup Y\cup  Z'])^a$, hence there is a chain $C$ between $X$ and $Y$ in $(G[X\cup Y\cup  Z'])^a$ that bypasses $Z'$ i.e., the chain $C$ is formed from nodes in $ant(X\cup Y)$ that are outside of $Z$. Since $Co(An(X\cup Y\cup Z'))\subseteq ant(X\cup Y)$, then $(G[X\cup Y\cup Z'])^a$ is a subgraph of $(G[ant(X\cup Y)])^a$. Then, the previously found chain $C$ is also a chain in $(G[ant(X\cup Y)])^a$ that bypasses $Z$, which means that $X$ and $Y$ are not separated by $Z$ in $(G[ant(X\cup Y)])^a=(G_{ant(X\cup Y)})^a$. So, $X$ and $Y$ are not $p$-separated by $Z$ in $G_{ant(X\cup Y)}$. This implies that $X$ and $Y$ are not $p$-separated by $Z$ in $G$, in contradiction to the assumption ${\langle X, Y | Z\rangle}$. Therefore, it has to be $Z\subseteq ant(X\cup Y).$ 	
\end{proof}

The next proposition shows that, by combining the results in propositions \ref{prop1amp} and \ref{prop2amp}, we can reduce our problems to a simpler one, which involves a smaller graph.
\begin{proposition}\label{prop3amp}
	Let $G=(V,E)$ be an AMP chain graph, and $X, Y\subseteq V$ are two disjoint subsets. Then the problem of finding a minimal separating set for $X$ and $Y$ in $G$ is equivalent to the problem of finding a minimal separating set for $X$ and $Y$ in the induced subgraph $G_{ant(X\cup Y)}$.
\end{proposition}
\begin{proof}
	The proof is very similar to the proof of Proposition 3 in \citep{ad,jv-uai18} and Proposition 9 in \citep{jv-pgm18}. Let $H=G_{ant(X\cup Y)}$, and let us to define sets $S_G=\{Z\subseteq V | \langle X, Y | Z\rangle_G\}$	and $S_H=\{Z\subseteq ant(X\cup Y) | \langle X, Y | Z\rangle_H\}$. Then we have to prove that $\min_{Z\in S_G}|Z|=\min_{Z\in S_H}|Z|$, and therefore, by proposition \ref{prop2amp}, the sets of minimal separators are the same. From proposition \ref{prop1amp}, we deduce that $S_H\subseteq S_G$, and therefore $\min_{Z\in S_H}|Z|\ge \min_{Z\in S_G}|Z|$.
	
	\noindent ($\Rightarrow$) Let $T=\min({Z\in S_G})$. Then $\forall T'\subsetneq T$ we have $T'\not\in S_G$, and from proposition \ref{prop2amp} we obtain $T\subseteq ant(X\cup Y)$, and now using proposition \ref{prop1amp} we get $T\in S_H$. So, we have $|T|=\min_{Z\in S_H}|Z|\ge \min_{Z\in S_G}|Z|=|T|$, hence $|T|= \min_{Z\in S_H}|Z|$.
	
	\noindent ($\Leftarrow$) Let $T=\min({Z\in S_H})$. If, $|T|=\min_{Z\in S_H}|Z|> \min_{Z\in S_G}|Z|=|Z_0|$, we have $\forall Z'\subsetneq Z_0, Z'\notin S_G$, and therefore, once again using proposition \ref{prop2amp} and \ref{prop1amp}, we get $Z_0\in S_H$, so that $|Z_0|\ge \min_{Z\in S_H}|Z|=|T|$, which is a contradiction. Thus, $|T|=\min_{Z\in S_G}|Z|$. 
\end{proof}

\begin{theorem}\label{thm1amp}
	The problem of finding a minimal separating set for $X$ and $Y$ in an AMP chain graph $G$ is equivalent to the problem of finding a minimal separating set for $X$ and $Y$ in the undirected graph $(G_{ant(X\cup Y)})^a$.
\end{theorem}
\begin{proof}
	The proof is very similar to the proof of Theorem 1 in \citep{ad,jv-uai18} and Theorem 10 in \citep{jv-pgm18}. Using the same notation from proposition \ref{prop3amp}, let $H^a$ be the augmented graph of $H=G_{ant(X\cup Y)}$, and $S_H^a=\{Z\subseteq ant(X\cup Y) | \langle X, Y | Z\rangle_{H^a} \}$. Let $Z$ be any subset of $ant(X\cup Y)$. Then taking into
	account the characteristics of anterior sets, it is clear that $H_{ant(X\cup Y\cup Z)}=H$. Then, we have
	$Z\in S_H\Leftrightarrow \langle X, Y | Z\rangle_{H}\Leftrightarrow \langle X, Y | Z\rangle_{(H_{ant(X\cup Y\cup Z)})^a}\Leftrightarrow \langle X, Y | Z\rangle_{H^a}\Leftrightarrow Z\in S_H^a.$
	Hence, $S_H=S_H^a$. Now, using proposition \ref{prop3amp}, we obtain $|T|=\min_{Z\in S_G}|Z|\Leftrightarrow |T|=\min_{Z\in S_H^a}|Z|$. 
\end{proof}

Informally, Theorem \ref{thm1amp} says that the search space of finding a minimal separating set $S$ for $X$ and $Y$ in an AMP chain graph $G$ is limited to $ant(X\cup Y)$, as shown in Figure \ref{fig:thm1amp}.
\begin{figure}[!ht]
	    \centering
	    \includegraphics[width=.35\linewidth,page=1]{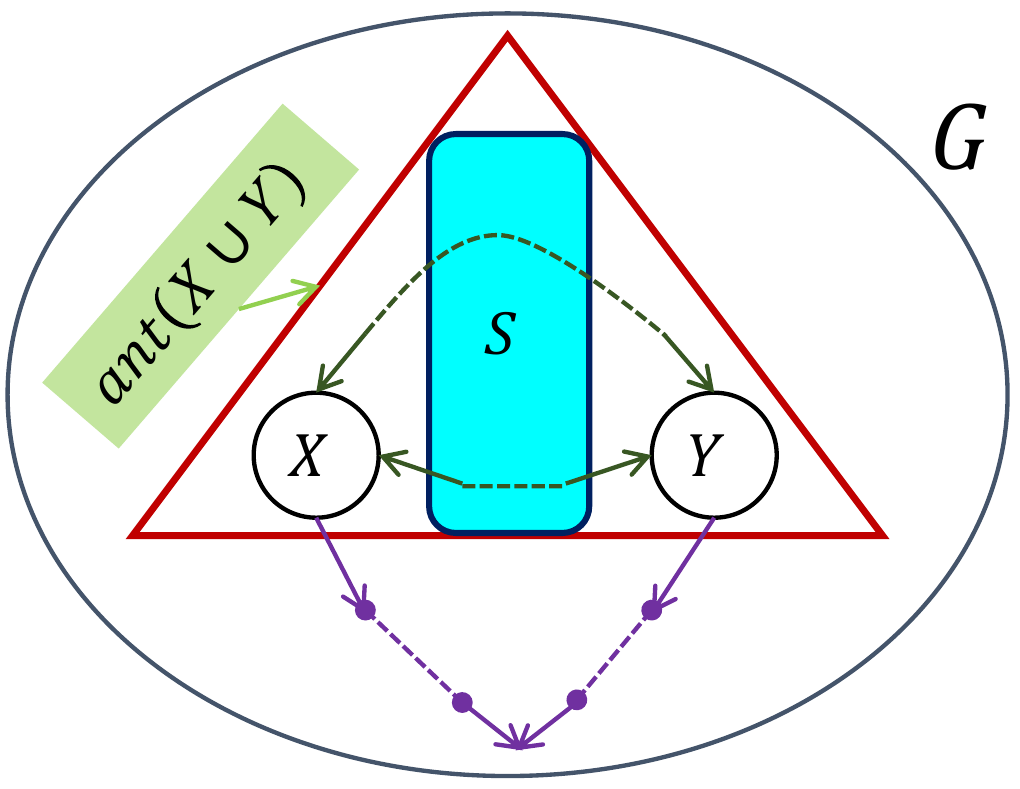}
		\captionof{figure}{Search space for finding a minimal separating set $S$ for $X$ and $Y$ in an AMP chain graph $G$.}
	    \label{fig:thm1amp}
\end{figure}
\begin{algorithm}[!ht]
\caption{Test for minimal separation (Problem \ref{problem1amp})}\label{alg1amp}
	\SetAlgoLined
	\KwIn{A set $Z$ that separates two non-adjacent nodes $X, Y$ in the AMP chain graph $G$.}
	\KwOut{If $Z$ is minimal then the algorithm returns TRUE otherwise, returns FALSE.}
	\eIf{$Z$ contains a node that is not in $ant(X\cup Y)$}{
		\Return{FALSE}\;
	}{
		\tcc{\textcolor{blue}{Building the search space according to Theorem \ref{thm1amp}.}}
		\tikzmk{A}
		Construct $G_{ant(X\cup Y)}$\;
		Construct $(G_{ant(X\cup Y)})^a$\;
		\tikzmk{B}\boxit{green!50}
		\tcc{\textcolor{blue}{Applying Theorem \ref{thm2amp} by running BFS algorithm that starts from both $X$ and $Y$.}}
		\tikzmk{A}
		Starting from $X$, run BFS. Whenever a node in $Z$
		is met, mark it if it is not already marked, and do not continue along
		that path. When BFS stops\;
		\eIf{not all nodes in $Z$ are marked}{\Return{FALSE}\;}{Remove all markings. Starting from $Y$, run BFS. Whenever a node in $Z$
			is met, mark it if it is not already marked, and do not continue along
			that path. When BFS stops\;
			\eIf{not all nodes in $Z$ are marked}{\Return{FALSE}\;}{\Return{TRUE}\;}}\tikzmk{B}\boxit{red!30}
	}
\end{algorithm}

\subsection{Algorithms for Finding Minimal Separators}
In undirected graphs we have efficient methods of testing whether a separation set is minimal, which are based on the following criterion.
\begin{theorem}\label{thm2amp}
	Given two nodes $X$ and $Y$ in an undirected graph, a separating
	set $Z$ between $X$ and $Y$ is minimal if and only if for each node $u$ in $Z$, there
	is a path from $X$ to $Y$ which passes through $u$ and does not pass through any
	other nodes in $Z$.
\end{theorem}
\begin{proof}
	See the proof of Theorem 5 in \citep{tpp}.
\end{proof}

Applying this theorem to the undirected graph described in Theorem \ref{thm1amp}, i.e.,  $(G_{ant(X\cup Y)})^a$, leads to Algorithm \ref{alg1amp} for Problem \ref{problem1amp}. The idea
is that if $Z$ is minimal then all nodes in $Z$ can be reached using Breadth
First Search (BFS) that starts from both $X$ and $Y$ without passing through any other
nodes in $Z$.

\textit{Analysis of Algorithm \ref{alg1amp} \citep{tpp}:} Let $H=G_{ant(X\cup Y)}$ and $|E_{H}^a|$ stands for the number of edges in $H^a=(G_{ant(X\cup Y)})^a$. Step 4-5 each requires $O(|E_{H}^a|)$ time. Thus, the complexity of Algorithm \ref{alg1amp} is $O(|E_{H}^a|)$.

\begin{remark}
[Characteristic operation and size measure]
The size measure used for graph algorithms in this paper is the sum of the number of vertices and the number of edges in a chain graph (for simplicity, in connected graphs, just the number of edges). 
This measure, which is used in algorithms textbooks (e.g.,~\citep{CLRS3rd}), is appropriate here, because the chain graph is given explicitly as an input.  In contrast, in heuristic search, it is usually assumed that a graph is constructed as it is searched, and the size measure that we chose would be inappropriate~\citep{Edelkamp2011,Pearl1984}.
\end{remark}
A variant of Algorithm \ref{alg1amp} solves Problem \ref{problem2amp}. Algorithm \ref{alg2amp} lists pseudocode for this variation.
\begin{algorithm}[ht]
\caption{Minimal separation (Problem \ref{problem2amp})}\label{alg2amp}
	\SetAlgoLined
	\KwIn{Two non-adjacent nodes $X, Y$ in the AMP chain graph $G$.}
	\KwOut{Set $Z$, that is a minimal separator for $X, Y$.}
	\tcc{\textcolor{blue}{Building the search space according to Theorem \ref{thm1amp}.}}
		\tikzmk{A}
	Construct $G_{ant(X\cup Y)}$\;
	Construct $(G_{ant(X\cup Y)})^a$\;
	\tikzmk{B}\boxit{green!50}
	Set $Z'$ to be $ne(X)$ (or $ne(Y)$) in $(G_{ant(X\cup Y)})^a$\;
	\tcc{$Z'$ is a separator because, according to the local Markov property of an undirected graph, a vertex is conditionally independent of all other vertices in the graph, given its neighbors \citep{l}.}
	\tcc{\textcolor{blue}{Applying Theorem \ref{thm2amp} by running BFS algorithm that starts from both $X$ and $Y$.}}
	\tikzmk{A}
	Starting from $X$, run BFS. Whenever a node in $Z'$
	is met, mark it if it is not already marked, and do not continue along
	that path. When BFS stops, let $Z''$ be the set of nodes which are marked. Remove all markings\;
	Starting from $Y$, run BFS. Whenever a node in $Z''$
	is met, mark it if it is not already marked, and do not continue along
	that path. When BFS stops, let  $Z$ be the set of nodes which are marked\;
	\tikzmk{B}\boxit{red!30}
	\Return{$Z$}\;
\end{algorithm}
\textit{Analysis of Algorithm \ref{alg2amp}:} Each one of steps 2-5 each requires $O(|E_{H}^a|)$ time. Thus, the overall complexity of Algorithm \ref{alg2amp} is $O(|E_{H}^a|)$.

\begin{theorem}\label{thm3amp}
	Given two nodes $X$ and $Y$ in an AMP chain graph $G$ and a set $S$ of nodes not
	containing $X$ and $Y$, there exists some subset of $S$ which separates $X$ and $Y$ if and only if the set $S'=S\cap ant(X\cup Y)$ separates $X$ and $Y$.
\end{theorem}
\begin{proof}
	($\Rightarrow$) Proof by contradiction. Let $S'=S\cap ant(X\cup Y)$ and ${\langle X, Y \not| S'\rangle}$. Since $S'\subseteq ant(X\cup Y)$, it is obvious that $ant(X\cup Y\cup S')=ant(X\cup Y)$. So, $X$ and $Y$ are not separated by $S'$ in $(G_{ant(X\cup Y)})^a$, hence there is a chain $C$ between $X$ and $Y$ in $(G_{ant(X\cup Y)})^a$ that bypasses $S'$ i.e., the chain $C$ is formed from nodes in $ant(X\cup Y)$ that are outside of $S$. Since $ant(X\cup Y)\subseteq ant(X\cup Y\cup S'')), \forall S''\subseteq S$ , then $(G_{ant(X\cup Y)})^a$ is a subgraph of $(G_{ant(X\cup Y\cup S)})^a$. Then, the previously found chain $C$ is also a chain in $(G_{ant(X\cup Y\cup S''})^a$ that bypasses $S''$, which means that $X$ and $Y$ are not separated by any $S''\subseteq S$ in $(G_{ant(X\cup Y\cup S})^a$, which is a contradiction.
	
		\noindent ($\Leftarrow$) It is obvious.
\end{proof}
\begin{figure}[ht]
	\centering
		\includegraphics[width=.35\linewidth,page=2]{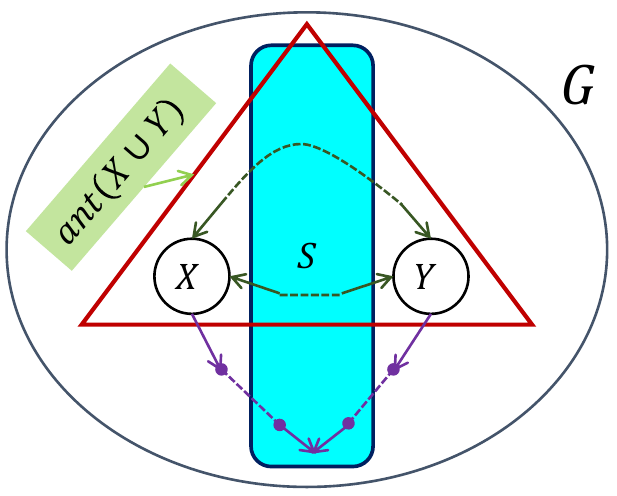}
		\captionof{figure}{Search space for finding a restricted minimal separating set $Z$ for $X$ and $Y$ in an AMP chain graph $G$, when a set of nodes $S$ not containing $X$ and $Y$ is given.}\label{fig:thm2amp}
\end{figure}
	
Informally, search space of finding a restricted minimal separating set $Z$ for $X$ and $Y$ in an AMP chain graph $G$, when a set of nodes $S$ not containing $X$ and $Y$ is given, is limited to $ant(X\cup Y)$, as shown in Figure \ref{fig:thm2amp}.
Therefore, Problem \ref{problem3amp} is solved by testing if $S'=S\cap ant(X\cup Y)$ separates $X$ and $Y$.
\begin{algorithm}[ht]
    \caption{Restricted separation (Problem \ref{problem3amp})}\label{alg3amp}
	\SetAlgoLined
	\KwIn{A set $S$ of nodes not containing $X$ and $Y$ in the AMP chain graph $G$.}
	\KwOut{If there is a subset of $S$ that separates $X$ from $Y$ then the algorithm returns $Z\subseteq S$ that separates $X$ from $Y$ otherwise, returns FALSE.}
	\tcc{\textcolor{blue}{Building the search space according to Theorem \ref{thm3amp}.}}
	\tikzmk{A}
	Construct $G_{ant(X\cup Y)}$\;
	Construct $(G_{ant(X\cup Y)})^a$\;
	Set $S'=S\cap ant(X\cup Y)$\;
	\tikzmk{B}\boxit{green!50}
	Remove $S'$ from $(G_{ant(X\cup Y)})^a$\;
	\tcc{\textcolor{blue}{Using BFS algorithm to test the separability of the candidate set $S'$.}}
	\tikzmk{A}
	Starting from $X$, run BFS\;
	\eIf{$Y$ is met}{\Return{FALSE}}{\Return{$Z=S'$}}
	\tikzmk{B}
	\boxit{red!50}
\end{algorithm}
\textit{Analysis of Algorithm \ref{alg3amp}:} This requires $O(|E_{H}^a|)$ time.

According to Theorem \ref{thm3amp}, Problem \ref{problem4amp} is solved using Algorithm \ref{alg3amp} and then, if False not returned, Algorithm \ref{alg2amp} with $Z'=S\cap ant(X\cup Y)$. The time complexity of this algorithm is also $O(|E_{H}^a|)$.

In order to solve Problem \ref{problem5amp}, i.e., to find the minimal set separating two disjoint non-adjacent subsets of nodes $X$ and $Y$ (instead
of two single nodes) in an AMP chain graph $G$, first we build the undirected graph $(G_{ant(X\cup Y)})^a$. Next, starting out from this graph, we construct a new undirected graph $Aug(G:\alpha_X,\alpha_Y)$ by adding two artificial (dummy) nodes $\alpha_X,\alpha_Y$, and connect them to those nodes that
are adjacent to some node in $X$ and $Y$, respectively. So, the separation of $X$ and $Y$ in $(G_{ant(X\cup Y)})^a$ is equivalent to the separation of $\alpha_X$ and $\alpha_Y$ in $Aug(G:\alpha_X,\alpha_Y)$. Moreover, the minimal separating set for $\alpha_X$ and $\alpha_Y$ in $Aug(G:\alpha_X,\alpha_Y)$ cannot contain nodes from $(X\cup Y)$. Therefore,
in order to find the minimal separating set for $X$ and $Y$ in $G$, it is suffice to find the minimal separating set for $\alpha_X$ and $\alpha_Y$ in $Aug(G:\alpha_X,\alpha_Y)$. So,
we have reduced this problem to one of separation for single nodes, which can be solved using Algorithm \ref{alg2amp}.

Shen and Liang in \citep{shenliang} presents an efficient algorithm for enumerating all minimal $(X,Y)$-separators, separating given non-adjacent vertices $X$ and $Y$ in an undirected connected simple graph $G = (V, E)$. This algorithm requires $O(n^3R_{XY})$ time, where $|V|=n$ and $R_{XY}$ is the number of minimal $(X,Y)$-separators. The algorithm can be generalized for enumerating all minimal $(X,Y)$-separators that separate non-adjacent vertex sets $X,Y\subseteq V$, and it requires $O(n^2(n-n_X-n_Y)R_{XY})$ time. In this case, $|X|=n_X$, $|Y|=n_Y$, and $R_{XY}$ is the number of all minimal $(X,Y)$-separators. According to Theorem \ref{thm1amp}, using this algorithm for $(G_{ant(X\cup Y)})^a$ solves Problem \ref{problem6amp}.
\begin{remark}
	Since DAGs (directed acyclic graphs) are subclass of AMP chain graphs, one can use the
	same technique to enumerate all minimal separators in DAGs.
\end{remark}

\section{\opc~Algorithm}\label{sec:pcalg}
In this section we explain the original \opc~algorithm proposed in \citep{penea12amp} briefly, and we show that this version of the \opc~algorithm is order-dependent,
in the sense that the output can depend on the order in which the variables are given. We propose modifications of the \opc~algorithm that remove (part or all of) this order-dependence.
\subsection{Order-Dependent \opc~algorithm}
The \opc~algorithm for learning AMP CGs under the faithfulness assumption proposed in \citep{penea12amp} is formally described in Algorithm \ref{algAMPoriginalPC} for the reader's convenience.

\begin{algorithm}[ht]
\caption{The order-dependent \opc~algorithm for learning AMP chain graphs \citep{penea12amp}}\label{algAMPoriginalPC}
	\SetAlgoLined
	\small\KwIn{A set $V$ of nodes and a probability distribution $p$ faithful to an unknown AMP CG $G$ and an ordering order($V$) on the variables.}
	\KwOut{A CG $H$ that is triplex equivalent to $G$.}
    Let $H$ denote the complete undirected graph over $V$\;
    \small\tcc{\textcolor{blue}{Skeleton Recovery}}
    \tikzmk{A}
\For{$i\gets 0$ \KwTo $|V_H|-2$}{
        \While{possible}{
            Select any ordered pair of nodes $u$ and $v$ in $H$ such that $u\in ad_H(v)$ and $|[ad_H(u)\cup ad_H(ad_H(u))]\setminus \{u,v\}|\ge i$, using order($V$);
            \tcc{$ad_H(x):=\{y\in V| x\erelbar{01} y, y\erelbar{01} x, \textrm{ or }x\erelbar{00}y\}$}
            \If{\textrm{there exists $S\subseteq ([ad_H(u)\cup ad_H(ad_H(u))]\setminus \{u,v\})$ s.t. $|S|=i$ and $u\perp\!\!\!\perp_p v|S$ (i.e., $u$ is independent of $v$ given $S$ in the probability distribution $p$)}}{
                Set $S_{uv} = S_{vu} = S$\;
                Remove the edge $u\erelbar{00} v$ from $H$\;
            }
        }
    }
    \tikzmk{B}
 \boxit{yellow}
    \tcc{\textcolor{blue}{Orientation phase:}}
    \tikzmk{A}
    \While{possible}{
        Apply rules R1-R4 in Figure \ref{fig:rules_ampcgs} to $H$.
    }
    Replace every edge $\erelbar{20}$ ($\erelbar{22}$) in $H$ with $\erelbar{01}$ ($\erelbar{00}$)\;
    \tikzmk{B}
 \boxit{pink}
 return $H$.
\end{algorithm}

In applications we do not have perfect conditional independence information.
Instead, we assume that we have an i.i.d. sample of size $n$ of $V = (X_1,\dots,Xp)$. In the \opc~algorithm \citep{penea12amp} all conditional independence queries are estimated by statistical conditional independence tests at some pre-specified significance level (p.value) $\alpha$. For example, if the distribution of $V$ is multivariate Gaussian, one can test for zero partial correlation, see, e.g., \citep{Kalisch07}. For this purpose, we used the $\mathsf{gaussCItest()}$ function from the \textsf{R} package \href{https://cran.r-project.org/web/packages/pcalg}{\textsf{pcalg}} throughout this paper. Let order($V$) denote an ordering on the variables in $V$. We now consider the role of
order($V$) in every step of the algorithm.

In the skeleton recovery phase of the \opc~algorithm \citep{penea12amp,PENA2016MAMP} (lines 2-10 of Algorithm \ref{algAMPoriginalPC}), the order of variables affects the estimation of the skeleton and the separating sets. In
particular, at each level of $i$, the order of variables determines the order in which pairs of adjacent
vertices and subsets $S$ of their adjacency sets are considered (see lines 4 and 5 in Algorithm \ref{algAMPoriginalPC}). The skeleton $H$ is updated after each edge removal. Hence, the adjacency sets typically
change within one level of $i$, and this affects which other conditional independencies are
checked, since the algorithm only conditions on subsets of the adjacency sets. When we have perfect conditional independence information,  all orderings on the variables lead to the same output. In the sample version, however, we typically make
mistakes in keeping or removing edges. In such cases, the resulting changes in the adjacency
sets can lead to different skeletons, as illustrated in Example \ref{ex1OrderDepAMP}.

Moreover, different variable orderings can lead to different separating sets in the skeleton recovery phase.
When we have perfect conditional independence information, this is not important, because any valid separating set leads to the
correct triplex decision in the orientation phase. In the sample version, however, different separating
sets in the skeleton recovery phase of the algorithm may yield different decisions about triplexes in the orientation phase (lines 12-15 of Algorithm \ref{algAMPoriginalPC}).
This is illustrated in Example \ref{ex2OrderDepAMP}. The examples were encountered when testing the \opc~algorithm by generating synthesized samples from the DAGs in Figure~\ref{fig:OrderDepex1amp}(a) and~\ref{fig:OrderDepex2amp}(a).

\begin{example}[Order-dependent skeleton of the \opc~algorithm.]\label{ex1OrderDepAMP}
Suppose that the distribution of $V = \{a,b,c,d,e\}$ is faithful to the DAG in Figure
\ref{fig:OrderDepex1amp}(a). This DAG encodes the following conditional independencies with minimal separating
sets: $b\perp\!\!\!\perp c|a$ and $a\perp\!\!\!\perp e|\{b,c,d\}$.

Suppose that we have an i.i.d. sample of $(a,b,c,d,e)$, and that the following
conditional independencies with minimal separating sets are judged to hold at some significance level $\alpha$: $b\perp\!\!\!\perp c|a$, $a\perp\!\!\!\perp e|d$,$a\perp\!\!\!\perp b|d$, $a\perp\!\!\!\perp c|d$, $b\perp\!\!\!\perp d|e$, and $c\perp\!\!\!\perp d|e$. Thus, the first conditional independence relation is correct, while the rest of them are false.

We now apply the skeleton recovery phase of the \opc~algorithm with two different orderings: $\textrm{order}_1(V)=(d,c,b,a,e)$ and $\textrm{order}_2(V)=(d,e,a,c,b)$. The resulting skeletons are shown in Figures \ref{fig:OrderDepex1amp}(b) and \ref{fig:OrderDepex1amp}(c), respectively. 
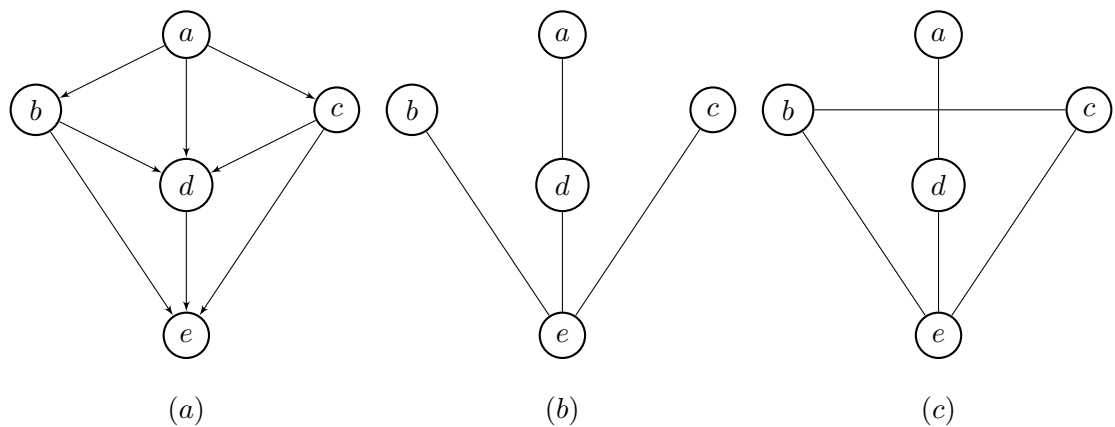
\begin{figure}[ht]
    \centering
	\[\begin{tikzpicture}[transform shape]
	\tikzset{vertex/.style = {shape=circle,draw,minimum size=1em}}
	\tikzset{edge/.style = {->,> = latex'}}
	\node[vertex,thick] (o) at  (0,0) {$e$};
	\node[vertex,thick] (p) at  (0,2) {$d$};
	\node[vertex,thick] (q) at  (0,4) {$a$};
	\node[vertex,thick] (r) at  (-2,3) {$b$};
	\node[vertex,thick] (s) at  (2,3) {$c$};
	\node (t) at (0,-1) {$(a)$};
	\draw[edge] (q) to (r);
	\draw[edge] (q) to (s);
	\draw[edge] (q) to (p);
	\draw[edge] (r) to (p);
	\draw[edge] (r) to (o);
	\draw[edge] (s) to (p);
	\draw[edge] (s) to (o);
	\draw[edge] (p) to (o);
	
	\node[vertex,thick] (i) at  (5,0) {$e$};
	\node[vertex,thick] (j) at  (5,2) {$d$};
	\node[vertex,thick] (k) at  (5,4) {$a$};
	\node[vertex,thick] (l) at  (3,3) {$b$};
	\node[vertex,thick] (m) at  (7,3) {$c$};
	\node (n) at (5,-1) {$(b)$};
	\draw (k) to (j);
	\draw (j) to (i);
	\draw (i) to (l);
	\draw (i) to (m);
	
	\node[vertex,thick] (e) at  (10,0) {$e$};
	\node[vertex,thick] (d) at  (10,2) {$d$};
	\node[vertex,thick] (a) at  (10,4) {$a$};
	\node[vertex,thick] (b) at  (8,3) {$b$};
	\node[vertex,thick] (c) at  (12,3) {$c$};
	\node (f) at (10,-1) {$(c)$};
	\draw (a) to (d);
	\draw (e) to (d);
	\draw (e) to (b);
	\draw (e) to (c);
	\draw (b) to (c);
	\end{tikzpicture}\]
    \caption{Order-dependent skeleton of the \opc~algorithm. (a) The DAG $G$, (b) the skeleton returned by  Algorithm \ref{algAMPoriginalPC} with $\textrm{order}_1(V)$, (c) the skeleton returned by  Algorithm \ref{algAMPoriginalPC} with $\textrm{order}_2(V)$. 
}
    \label{fig:OrderDepex1amp}
\end{figure}

We see that the skeletons are different, and that both are
incorrect as the edges $a\erelbar{00}b, a\erelbar{00}c, b\erelbar{00}d,$ and $c\erelbar{00}d$ are missing. The skeleton for $\textrm{order}_2(V)$ contains an additional
error, as there is an additional edge $b\erelbar{00}c$. We now go through Algorithm \ref{algAMPoriginalPC} to see what happened. We start with a complete undirected graph on $V$. When $i= 0$, variables are tested for marginal independence, and the algorithm correctly does not remove any edge. When $i=1$, there are six pairs of vertices that are
thought to be conditionally independent given a subset of size one. Table \ref{t1OrderDepAMPex1}
shows the trace table of Algorithm \ref{algAMPoriginalPC} for $i=1$ and $\textrm{order}_1(V)=(d,c,b,a,e)$.
\begin{table}[!ht]
\caption{The trace table of Algorithm \ref{algAMPoriginalPC} for $i=1$ and $\textrm{order}_1(V)=(d,c,b,a,e)$. For simplicity, we define $ADJ_H(u):=[ad_H(u)\cup ad_H(ad_H(u))]\setminus \{u,v\}$.}\label{t1OrderDepAMPex1}
\centering
\begin{tabular}{c|c|c|c|c}
 Ordered Pair $(u,v)$& $ADJ_H(u)$ & $S_{uv}$&Is $S_{uv}\subseteq ADJ_H(u)$?& Is $u\erelbar{00} v$ removed?\\
\midrule
\midrule
   $(d,c)$ & $\{a,b,e\}$&$\{e\}$&	Yes&	Yes\\%
    \midrule
  $(d,b)$  &$\{a,c,e\}$&$\{e\}$&	Yes&	Yes \\
\midrule
$(c,b)$ &$\{a,d,e\}$ &$\{a\}$&Yes&	Yes\\
\midrule
$(c,a)$ &$\{b,d,e\}$&$\{d\}$&Yes&Yes\\
\midrule
   $(b,a)$ &$\{c,d,e\}$&$\{d\}$&Yes&Yes\\
    \midrule
  $(a,e)$  &$\{d\}$&$\{d\}$&Yes&Yes\\
\bottomrule
\end{tabular}
\end{table}

Table \ref{t2OrderDepAMPex1}
shows the trace table of Algorithm \ref{algAMPoriginalPC} for $i=1$ and $\textrm{order}_2(V)=(d,e,a,c,b)$.
\begin{table}[!ht]
\caption{The trace table of Algorithm \ref{algAMPoriginalPC} for $i=1$ and $\textrm{order}_2(V)=(d,e,a,c,b)$. For simplicity, we define $ADJ_H(u):=[ad_H(u)\cup ad_H(ad_H(u))]\setminus \{u,v\}$.}\label{t2OrderDepAMPex1}
\centering
\begin{tabular}{c|c|c|c|c}
 Ordered Pair $(u,v)$& $ADJ_H(u)$ & $S_{uv}$&Is $S_{uv}\subseteq ADJ_H(u)$?& Is $u\erelbar{00} v$ removed?\\
\midrule
\midrule
   $(d,c)$ & $\{a,b,e\}$&$\{e\}$&	Yes&	Yes\\%
    \midrule
  $(d,b)$  &$\{a,c,e\}$&$\{e\}$&	Yes&	Yes \\
\midrule
$(e,a)$ &$\{b,c,d\}$ &$\{d\}$&Yes&	Yes\\
\midrule
$(a,c)$ &$\{b,d,e\}$&$\{d\}$&Yes&Yes\\
\midrule
   $(a,b)$ &$\{d,e\}$&$\{d\}$&Yes&Yes\\
    \midrule
  $(c,b)$  &$\{d,e\}$&$\{a\}$&No&No\\
  \midrule
  $(b,c)$  &$\{c,e\}$&$\{a\}$&No&No\\
\bottomrule
\end{tabular}
\end{table}

No conditional independency is found when $i=2.$
\end{example}

\begin{example}[Order-dependent separators \& triplexes of the \opc~algorithm.]\label{ex2OrderDepAMP}
Suppose that the distribution of $V = \{a,b,c,d,e\}$ is faithful to the DAG in Figure
\ref{fig:OrderDepex2amp}(a). This DAG encodes the following conditional independencies with minimal separating sets: $a\perp\!\!\!\perp d|b, a\perp\!\!\!\perp e|\{b,c\}, a\perp\!\!\!\perp e|\{c,d\}, b\perp\!\!\!\perp c, b\perp\!\!\!\perp e|d,$ and $c\perp\!\!\!\perp d$.

Suppose that we have an i.i.d. sample of $(a,b,c,d,e)$. Assume that all true conditional independencies are judged to hold except $c\perp\!\!\!\perp d$. Suppose that $c\perp\!\!\!\perp d|b$ and $c\perp\!\!\!\perp d|e$ are thought to hold. Thus, the first is correct, while the second is false. We now apply the orientation phase of the \opc~algorithm with two different orderings: $\textrm{order}_1(V)=(d,c,b,a,e)$ and $\textrm{order}_3(V)=(c,d,e,a,b)$. The resulting CGs are shown in Figures \ref{fig:OrderDepex2amp}(b) and \ref{fig:OrderDepex2amp}(c), respectively. Note that while the separating set for vertices $c$ and $d$ with $\textrm{order}_1(V)$ is $S_{dc}=S_{cd}=\{b\}$, the separating set for them with $\textrm{order}_2(V)$ is $S_{cd}=S_{dc}=\{e\}$.
\begin{figure}[!htbp]
    \centering
	\[\begin{tikzpicture}[transform shape]
	\tikzset{vertex/.style = {shape=circle,draw,minimum size=1em}}
	\tikzset{edge/.style = {->,> = latex'}}
	\node[vertex,thick] (o) at  (1,0.5) {$e$};
	\node[vertex,thick] (p) at  (-1,0.5) {$d$};
	\node[vertex,thick] (q) at  (0,3) {$a$};
	\node[vertex,thick] (r) at  (-1.5,2) {$b$};
	\node[vertex,thick] (s) at  (1.5,2) {$c$};
	\node (t) at (0,-.5) {$(a)$};
	\draw[edge,thick] (r) to (q);
	\draw[edge,thick] (s) to (q);
	\draw[edge,thick] (r) to (p);
	\draw[edge,thick] (s) to (o);
	\draw[edge,thick] (p) to (o);
	
	\node[vertex,thick] (i) at  (5,0.5) {$e$};
	\node[vertex,thick] (j) at  (3,0.5) {$d$};
	\node[vertex,thick] (k) at  (4,3) {$a$};
	\node[vertex,thick] (l) at  (2.5,2) {$b$};
	\node[vertex,thick] (m) at  (5.5,2) {$c$};
	\node (n) at (4,-.5) {$(b)$};
	\draw[edge,thick] (l) to (k);
	\draw[edge,thick] (j) to (i);
	\draw[edge,thick] (m) to (k);
	\draw[edge,thick] (m) to (i);
	\draw[thick] (j) to (l);
	
	\node[vertex,thick] (e) at  (9,0.5) {$e$};
	\node[vertex,thick] (d) at  (7,0.5) {$d$};
	\node[vertex,thick] (a) at  (8,3) {$a$};
	\node[vertex,thick] (b) at  (6.5,2) {$b$};
	\node[vertex,thick] (c) at  (9.5,2) {$c$};
	\node (f) at (8,-.5) {$(c)$};
	\draw[edge,thick] (b) to (a);
	\draw[edge,thick] (c) to (a);
	\draw[thick] (e) to (d);
	\draw[thick] (e) to (c);
	\draw[thick] (b) to (d);
	\end{tikzpicture}\]
    \caption{Order-dependent separators and triplexes of the \opc~algorithm. (a) The DAG $G$, (b) the CG returned by  Algorithm \ref{algAMPoriginalPC} with $\textrm{order}_1(V)$, (c) the CG returned by  Algorithm \ref{algAMPoriginalPC} with $\textrm{order}_3(V)$. 
}\label{fig:OrderDepex2amp}
\end{figure}
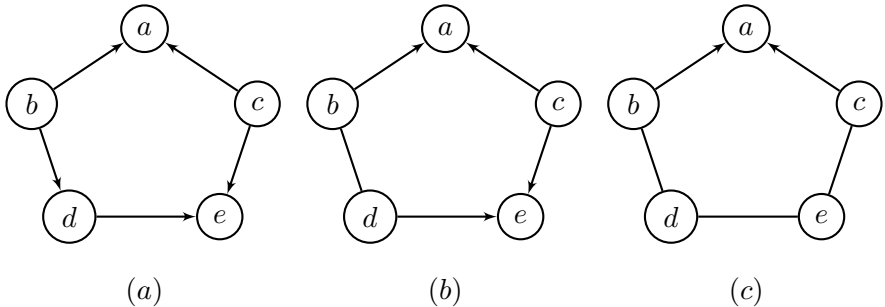

This illustrates that order-dependent separating sets in the skeleton recovery phase of the sample version of the PC-algorithm can lead to order-dependent triplexes in the orientation phase of the algorithm.
\end{example}

\subsection{Order-Independent ( {\sc Stable-}) \opc~Algorithm}
As shown in the previous section, the original \opc~algorithm is order-dependent. In this section we propose modifications of the \opc~algorithm, i.e., the \textbf{Stable} \textbf{PC}-like for \textbf{AMP} chain graphs (\spc), \textbf{Conservative} \textbf{PC}-like for \textbf{AMP} CGs (\cpc), and \scpc~for learning the structure of AMP chain graphs under the faithfulness assumption that remove part or all of the order-dependence. The order-dependence can become very
problematic for high-dimensional data, leading to highly variable results and conclusions
for different variable orderings. The second limitation of the \opc~algorithm is that the runtime of the algorithm, in the worst case, is exponential to the number of variables, and thus it is inefficient when applying to high dimensional datasets such as gene expression. We now propose several modifications of the original \opc~algorithm for learning AMP chain graphs (and hence also of the related algorithms), called \textit{stable \opc},  that remove the order-dependence in the various stages of the algorithm, analogously to what~\citep{Colombo2014} did for the original PC algorithm in the case of DAGs. The stable \opc~algorithm for AMP chain graphs can be used to parallelize the conditional independence (CI) tests at each
level of the skeleton recovery algorithm. So, the CI tests at each level can be grouped and distributed over different cores of the computer, and
the results can be integrated at the end of each level. Consequently, the runtime of our parallelized stable \opc~algorithm is much shorter than the original \opc~algorithm for learning AMP chain graphs. Furthermore, this approach enjoys the advantage of knowing the number of CI tests of each level in advance. This allows the CI tests to be evenly distributed over different cores, so that the parallelized algorithm can achieve maximum possible speedup.
In order to explain the details of the stable \opc~algorithm for AMP chain graphs, we discuss the skeleton and the orientation rules, respectively.

We first consider estimation of the skeleton in the adjacency search (skeleton recovery phase) of the \opc~algorithm for AMP chain graphs (lines 2-10 of Algorithm \ref{algAMPoriginalPC}). The pseudocode for our modification is given in Algorithm \ref{algAMPstablePC} (lines 2-13). The resulting \opc~algorithm for learning AMP chain graphs in Algorithm \ref{algAMPstablePC} is called \textit{\textbf{Stable} \textbf{PC}-like for \textbf{AMP} CGs (\spc)}.

\begin{algorithm}[ht]
\caption{The order-independent ( {\sc Stable-}) \opc~algorithm for learning AMP CGs (\spc)}\label{algAMPstablePC}
	\SetAlgoLined
	\small\KwIn{A set $V$ of nodes and a probability distribution $p$ faithful to an unknown AMP CG $G$ and an ordering order($V$) on the variables.}
	\KwOut{A CG $H$ that is triplex equivalent to $G$.}
    Let $H$ denote the complete undirected graph over $V=\{v_1,\dots,v_n\}$\;
    \small\tcc{\textcolor{blue}{Skeleton Recovery:}}
    \tikzmk{A}
\For{$i\gets 0$ \KwTo $|V_H|-2$}{
    \For{$j\gets 1$ \KwTo $|V_H|$}{
        Set $a_H(v_j)=ad_H(v_j)\cup ad_H(ad_H(v_j))$\;
        \tcc{$ad_H(x):=\{y\in V| x\erelbar{01} y, y\erelbar{01} x, \textrm{ or }x\erelbar{00}y\}$}
    }
        \While{possible}{
            Select any ordered pair of nodes $u$ and $v$ in $H$ such that $u\in ad_H(v)$ and $|a_H(u)\setminus \{u,v\}|\ge i$, using order($V$);
            
            \If{\textrm{there exists $S\subseteq (a_H(u)\setminus \{u,v\})$ s.t. $|S|=i$ and $u\perp\!\!\!\perp_p v|S$ (i.e., $u$ is independent of $v$ given $S$ in the probability distribution $p$)}}{
                Set $S_{uv} = S_{vu} = S$\;
                Remove the edge $u\erelbar{00} v$ from $H$\;
            }
        }
    }
    \tikzmk{B}
 \boxit{orange}
    \tcc{\textcolor{blue}{Orientation phase:}}
    \tikzmk{A}
    \While{possible}{
        Apply rules R1-R4 in Figure \ref{fig:rules_ampcgs} to $H$.
    }
    Replace every edge $\erelbar{20}$ ($\erelbar{22}$) in $H$ with $\erelbar{01}$ ($\erelbar{00}$)\;
    \tikzmk{B}
 \boxit{pink}
 return $H$.
\end{algorithm}
\begin{figure}[ht]
    \centering
    \includegraphics[width=.75\linewidth]{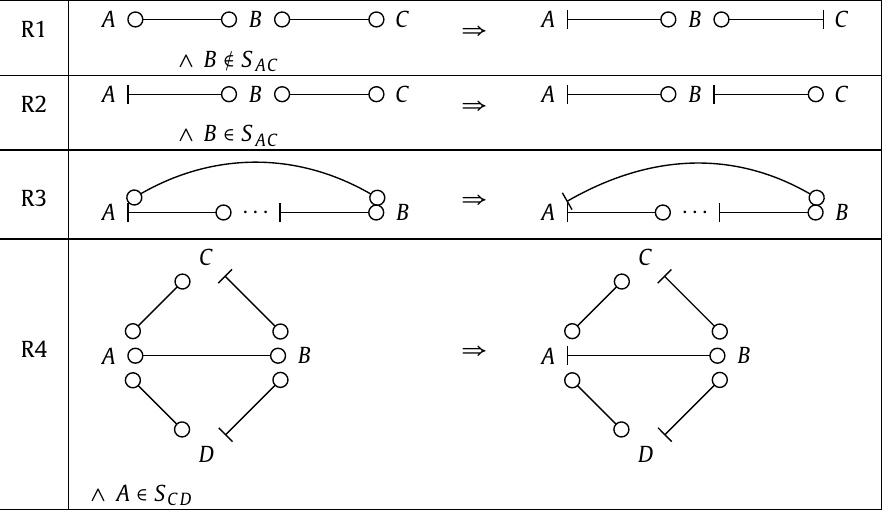}
    \caption{Orientation rules in Algorithms \ref{algAMPoriginalPC} and \ref{algAMPstablePC}: Rules R1-R4 \citep{penea12amp}}
    \label{fig:rules_ampcgs}
\end{figure}

The main difference between Algorithms \ref{algAMPoriginalPC} and \ref{algAMPstablePC} is given by the for-loop on lines
3-5 in the latter one, which computes and stores the adjacency sets $a_H(v_i)$ of all variables
after each new size $i$ of the conditioning sets. These stored adjacency sets $a_H(v_i)$ are used
whenever we search for conditioning sets of this given size $i$. Consequently, an edge deletion
on line 10 no longer affects which conditional independencies are checked for other pairs of
variables at this level of $i$.

In other words, at each level of $i$, Algorithm \ref{algAMPstablePC} records which edges should be removed, but for the purpose of the adjacency sets it removes these edges only when it goes to the
next value of $i$. Besides resolving the order-dependence in the estimation of the skeleton,
our algorithm has the advantage that it is easily parallelizable at each level of $i$   i.e., computations required for $i$-level can be performed in parallel.
The \spc~algorithm is 
correct, i.e. it returns an AMP CG the given probability distribution is faithful to (Theorem \ref{thm:correctnessPCstable}), and yields order-independent skeletons in the sample version (Theorem \ref{thm:stableskeleton}). We illustrate the algorithm in Example \ref{ex:stableskeletons}.

\begin{theorem}\label{thm:correctnessPCstable}
    Let the distribution of $V$ be faithful to an AMP CG $G$, and assume that we are given perfect conditional independence information about all pairs of variables $(u,v)$ in $V$ given subsets $S\subseteq V\setminus \{u,v\}$. Then the output of the \spc~algorithm is an AMP CG that is Markov equivalent with $G$.
\end{theorem}
\begin{proof}
    The proof of Theorem \ref{thm:correctnessPCstable} is completely analogous to the proof of Theorem 1 for the original \opc~algorithm in \citep{penea12amp}.
\end{proof}

\begin{theorem}\label{thm:stableskeleton}
    The skeleton resulting from the sample version of the \spc~algorithm for AMP CGs is order-independent.
\end{theorem}
\begin{proof}
We consider the removal or retention of an arbitrary edge $u\erelbar{00} v$ at some level $i$.
The ordering of the variables determines the order in which the edges (line 7 of Algorithm \ref{algAMPstablePC}) and the subsets $S$ of $a_H(u)$ and $a_H(v)$ (line 8 of Algorithm \ref{algAMPstablePC}) are considered. By
construction, however, the order in which edges are considered does not affect the sets $a_H(u)$ and $a_H(v)$.

If there is at least one subset $S$ of $a_H(u)$ or $a_H(v)$ such that $u\perp\!\!\!\perp_p v|S$, then any
ordering of the variables will find a separating set for $u$ and $v$ (but different orderings
may lead to different separating sets as illustrated in Example \ref{ex2OrderDepAMP}). Conversely, if there is no subset $S'$ of $a_H(u)$ or $a_H(v)$ such that $u\perp\!\!\!\perp_p v|S'$, then no ordering will find a separating set.

Hence, any ordering of the variables leads to the same edge deletions, and therefore to
the same skeleton.
\end{proof}

\begin{example}[Order-independent skeletons]\label{ex:stableskeletons}
We go back to Example \ref{ex1OrderDepAMP}, and consider the sample version of Algorithm \ref{algAMPstablePC}. The algorithm now outputs the skeleton shown in Figure \ref{fig:OrderDepex1amp}(b) for both orderings $\textrm{order}_1(V)$ and $\textrm{order}_2(V)$.

We again go through the algorithm step by step. We start with a complete undirected
graph on $V$. No conditional independence found when $i=0$. When $i= 1$, the algorithm first computes the new adjacency sets: $a_H(v)=V\setminus\{v\}, \forall v\in V$. There are six pairs of variables that are thought to be conditionally independent given a
subset of size 1 (see Table \ref{t3OrderDepAMPex1}). Since the sets $a_H(v)$ are not updated after
edge removals, it does not matter in which order we consider the ordered pairs. Any ordering leads to the removal of six edges. 

\begin{table}[!ht]
\caption{The trace table of Algorithm \ref{algAMPstablePC} for $i=1$, $\textrm{order}_1(V)=(d,c,b,a,e)$, and $\textrm{order}_2(V)=(d,e,a,c,b)$. For simplicity, we define $ADJ_H(u):=[ad_H(u)\cup ad_H(ad_H(u))]\setminus \{u,v\}$.}
\centering
\begin{tabular}{c|c|c|c|c}
 Ordered Pair $(u,v)$& $ADJ_H(u)$ & $S_{uv}$&Is $S_{uv}\subseteq ADJ_H(u)$?& Is $u\erelbar{00} v$ removed?\\
\midrule
\midrule
   $(d,c)$ & $\{a,b,e\}$&$\{e\}$&	Yes&	Yes\\%
    \midrule
  $(d,b)$  &$\{a,c,e\}$&$\{e\}$&	Yes&	Yes \\
\midrule
$(c,b)$ &$\{a,d,e\}$ &$\{a\}$&Yes&	Yes\\
\midrule
$(c,a)$ &$\{b,d,e\}$&$\{d\}$&Yes&Yes\\
\midrule
   $(b,a)$ &$\{c,d,e\}$&$\{d\}$&Yes&Yes\\
    \midrule
  $(a,e)$  &$\{b,c,d\}$&$\{d\}$&Yes&Yes\\
\bottomrule
\end{tabular}\label{t3OrderDepAMPex1}
\end{table}
\end{example}

Now, we propose a method to resolve the order-dependence in the determination of the triplexes in AMP chain graphs, by extending the approach in \citep{Ramsey:2006} for unshielded colliders recovery in DAGs. 

Our proposed {\textbf{Conservative} \textbf{PC}-like algorithm  for \textbf{AMP} CGs (\cpc)} works as follows. Let $H$ be the undirected graph resulting from the skeleton recovery phase  of the \opc~algorithm (Algorithm \ref{algAMPoriginalPC}). For all unshielded triples $(X_i, X_j, X_k)$ in $H$, determine all subsets
$S$ of $ad_H(X_i)\cup ad_H(ad_H(X_i))$ and of $ad_H(X_k)\cup ad_H(ad_H(X_k))$ that make $X_i$ and $X_k$ conditionally independent, i.e., that
satisfy $X_i\perp\!\!\!\perp_p X_k|S$. We refer to such sets as separating sets. The triple $(X_i, X_j, X_k)$ is labelled as \textit{unambiguous} if at least one such separating set is found and either $X_j$ is in all separating sets or in none of them; otherwise it is labelled as \textit{ambiguous}. If the triple is unambiguous, it is labeled and then oriented as described in Algorithm \ref{algAMPoriginalPC}. So, the orientation rules are
adapted so that only unambiguous triples are oriented. 

We refer to the combination of the \spc~and \cpc~algorithms for AMP chain graphs as the \scpc~algorithm.

\begin{theorem}\label{thm:correctnessCPCstable}
    Let the distribution of $V$ be faithful to an AMP CG $G$, and assume that we are given perfect conditional independence information about all pairs of variables $(u,v)$ in $V$ given subsets $S\subseteq V\setminus \{u,v\}$. Then the output of the \cpc~/ \scpc~algorithm is an AMP CG that is Markov equivalent with $G$.
\end{theorem}
\begin{proof}
    The skeleton of the learned CG is correct by Theorem \ref{thm:correctnessPCstable}.
    Now, we prove that for any unshielded triple $(X_i, X_j, X_k)$ in an AMP CG $G$, $X_j$ is either in all sets that \textit{p}-separate $X_i$ and $X_k$ or in none of them. Since $X_i, X_k$ are not adjacent, they are \textit{p}-separated given some subset $S\setminus\{X_i, X_k\}$ (see Algorithm \ref{alg2amp}). Based on the pathwise \textit{p}-separation criterion for AMP CGs (see Definition \ref{pseparation}), $X_j$ is a triplex node in $G$ if and only if $X_j\not\in An(S).$ So, $X_j\not\in S.$ On the other hand, if $X_j$ is a non-triplex node then $X_j\in S$, for all $S$ that \textit{p}-separate $X_i$ and $X_k$. Because in this case, $X_j\in Co(An(X_i\cup X_k\cup S))$ and so there is an undirected path $X_i\erelbar{00} X_j\erelbar{00} X_k$ in $(G[X_i\cup X_k\cup S])^a$. Any set $S\setminus\{X_i, X_k\}$ that does not contain $X_j$ will fail to \textit{p}-separate $X_i$ and $X_k$ because of this undirected path. 
    As a result, unshielded triples are all unambiguous. Since all unshielded triples are unambiguous, the orientation rules are as in the original ( {\sc Stable-}) \opc~algorithm. Therefore,  the output of the \cpc/\scpc~algorithm is an AMP CG that is Markov equivalent with $G$.
\end{proof}

\begin{theorem}\label{thm:stabletriplex}
    The decisions about triplexes in the sample version of the algorithm for AMP chain graphs recovery by \scpc~is order-independent.
\end{theorem}
\begin{proof}
The \scpc~algorithm have order-independent skeleton, by
Theorem \ref{thm:stableskeleton}. In particular, this means that their unshielded triples and adjacency sets are order-independent. The decision about whether an unshielded triple is unambiguous
and/or a triplex is based on the adjacency sets of nodes in the triple, which are order independent.
\end{proof}

\begin{example}[Order-independent decisions about triplexes]\label{ex:stabletriplex}
We consider the sample versions of the \scpc~algorithm for AMP chain graphs, using the same input as in Example \ref{ex2OrderDepAMP}. In particular, we assume that all conditional independencies induced by the AMP CG in Figure \ref{fig:OrderDepex2amp}(a)
are judged to hold except $c\perp\!\!\!\perp d$. Suppose that $c\perp\!\!\!\perp d|b$ and $c\perp\!\!\!\perp d|e$ are thought to hold. 

Denote the skeleton after the skeleton recovery phase by $H$. We consider the unshielded triple $(c,e,d)$.
First, we compute $a_H(c)=\{a,b,d,e\}$ and $a_H(d)=\{a,b,c,e\}$. We now consider all
subsets $S$ of these adjacency sets, and check whether $c\perp\!\!\!\perp d|S$. The following separating sets are found: $\{b\},\{e\}$, and $\{b,e\}$. Since $e$ is in some but not all of these separating sets, the \scpc~algorithm for AMP chain graphs determines that the triple is ambiguous, and no orientations are performed. The output of the algorithm is given in Figure \ref{fig:OrderDepex2amp}(c).
\end{example}

At this point it should be clear why the modified \opc~algorithm for AMP chain graphs is labeled ``conservative": it is more cautious than the ( {\sc Stable-}) \opc~algorithm for AMP chain graphs in drawing unambiguous conclusions about orientations. As we showed in Example \ref{ex:stabletriplex}, the output of the algorithm for AMP chain graphs recovery by \cpc~or \scpc~may
not be triplex equivalent with the true AMP CG $G$, if the resulting CG contains an ambiguous triple. 

Table \ref{Relations:modifiedalgs} summarizes all order-dependence issues explained above and the corresponding modifications of the \opc~algorithm for AMP chain graphs that removes the given order-dependence problem.

\begin{table}[!htpb]
\caption{Order-dependence issues and corresponding modifications of the \opc~algorithm
that remove the problem. ``Yes" indicates that the corresponding aspect
of the graph is estimated order-independently in the sample version.}\label{Relations:modifiedalgs}
\centering
\begin{tabular}{c|c|c}
 & skeleton & triplexes decisions\\
\midrule
\midrule
\opc~algorithm for AMP CGs & No & No\\
    \midrule
\spc & Yes & No\\
\midrule
\scpc & Yes & Yes\\
\bottomrule
\end{tabular}
\end{table}

\section{\lcd~Algorithm: Structure Learning by Decomposition}\label{main-alg}
In this section, first, we
address the issue of how to construct a $p$-separation tree from observed data, which is the heart of our decomposition-based algorithm. Then, we present an algorithm, called \textit{\lcd}, that shows how separation trees can be used to facilitate the decomposition of the
structure learning of AMP chain graphs. The theoretical results are presented first, followed by descriptions of our algorithm that is the summary of the key results in our paper.

\subsection{Constructing a \textit{p}-Separation Tree from Observed Data}\label{septree_ampcg}
As proposed in \citep{xie}, one can construct a $d$-separation tree from observed data. In this section we extend
Theorem 2 of \citep{xie}, and thereby prove that their method for constructing a separation tree
from data is valid for AMP chain graphs. To construct an undirected independence graph in which
the absence of an edge $u\erelbar{00} v$ implies $u \perp\!\!\!\perp v | V\setminus\{u,v\}$, we can start with a complete undirected graph, and then for each pair of variables $u$ and $v$, an undirected edge $u\erelbar{00} v$ is removed if $u$ and
$v$ are independent conditional on the set of all other variables \citep{xie}. For normally distributed data, the undirected independence graph can be efficiently constructed by removing an edge $u\erelbar{00} v$ if and only if the corresponding entry in the concentration matrix (inverse covariance matrix) is zero \citep[Proposition 5.2]{l}. For this purpose, performing a conditional independence test for each pair of random variables using the
partial correlation coefficient can be used. If the $p$-value of the test is smaller than the given threshold, then there will be an edge on the output graph. For discrete data, a test of conditional independence given a large number of discrete variables may
be of extremely low power. To cope with such difficulty, there are two fundamental ways to perform structure
learning: (1) \textit{Parameter estimation techniques} \citep{Banerjee2007,ravikumar2010} that utilize a factorization of the distribution according to the cliques of the graph to learn the
underlying graph. These techniques assume a certain form of the potential function, and thereby relate the structure learning problem to one of finding a sparse maximum likelihood estimator of a distribution from its samples. (2) Algorithms based on learning conditional independence relations between the variables \citep{ChowLiu,Bresler2008,Netrapalli2010,anandkumar2012} that they do not need knowledge of the underlying parametrization to learn the graph. These methods are based on comparing all possible neighborhoods of a node to find one which has the \textit{maximum influence} on the node. In \citep[Chapter 6]{ed}, \citep{Bromberg2009}, and \citep{deAbreu2010} there are other methods for UIG learning, including some for data with both continuous and discrete variables. All these methods can be used to construct separation trees from data.

\begin{theorem}\label{thm_junctree_ampcg}
	A junction tree constructed from an undirected independence graph for AMP CG $G$ is a $p$-separation tree for $G$.
\end{theorem}
\begin{proof}
    See Appendix A.
\end{proof}

A $p$-separation tree $T$ only requires that all $p$-separation properties of $T$ also hold for AMP CG $G$, but the reverse is
not required. Thus we only need to construct an undirected independence graph that may have fewer conditional
independencies than the augmented graph, and this means that the undirected independence graph may have extra edges
added to the augmented graph. As \citep{xie} observe for $d$-separation in DAGs, if all nodes of a $p$-separation tree contain only a few variables, ``the null hypothesis of the absence of an undirected edge may be tested statistically at a larger significance level."

Since there are standard algorithms for constructing junction trees from UIGs \citep[Chapter 4, Section 4]{cdls}, the construction of separation trees reduces to the construction of
UIGs. In this sense, Theorem \ref{thm_junctree_ampcg} enables us to exploit various techniques for learning UIGs to serve our purpose. More suggested methods for learning UIGs from data, in addition to the above mentioned techniques, can be found in \citep{mxg}.

\begin{example}
To construct a $p$-separation tree for the AMP CG $G$ in Figure \ref{fig:AMPUIGChordal}(a), at first an undirected independence graph
is constructed by starting with a complete graph and removing an edge $u\erelbar{00} v$ if $u \perp\!\!\!\perp v | V\setminus\{u,v\}$. An undirected graph
obtained in this way is the augmented graph of AMP CG $G$. In fact, we only need to construct an undirected independence
graph which may have extra edges added to the augmented graph. Next triangulate the undirected graph and finally obtain
the $p$-separation tree, as shown in Figure \ref{fig:AMPUIGChordal}(c) and Figure \ref{fig:pseptree} respectively.
\end{example}

\subsection{The \lcd~Algorithm for Learning AMP Chain Graphs}

By applying the following theorem to structural learning, we can split a problem of searching for $p$-separators and building the skeleton of a CG into small problems for every node of $p$-separation tree $T$.
\begin{theorem}\label{thm_main_ampcg}
	Let $T$ be a $p$-separation tree for AMP CG $G$ and $u$ and $v$ be two vertices that do not belong to the same chain component. So, vertices $u$ and $v$ are $p$-separated by $S\subseteq V$ in $G$ if and
	only if (i) $u$ and $v$ are not contained together in any node $C$ of $T$ or (ii) there exists a node $C$ that contains both $u$ and $v$ such that a subset $S'$ of $C$ $p$-separates $u$ and $v$.
\end{theorem}
\begin{proof}
    See Appendix A.
\end{proof}

According to Theorem \ref{thm_main_ampcg}, a problem of searching for a $p$-separator $S$ of $u$ and $v$ in all possible subsets of $V$ is localized to all possible subsets of nodes in a $p$-separation tree that contain $u$ and $v$. For a given $p$-separation tree $T$
with the node set $C = \{C_1,\dots , C_H \}$, we can recover the skeleton and all triplexes for an AMP CG using a constraint-based algorithm, called \lcd, that contains two main steps: (a) determining the skeleton by a divide-and-conquer approach; (b) determining triplexes and orienting some of the undirected edges into directed edges according to a set of rules applied iteratively with localized search for $p$-separators.
We elaborate on each phase of this algorithm below.

\noindent\textbf{\lcd~Description:} \textit{(a) Skeleton Recovery.}  This phase has two steps. First, 
we construct a \textit{local skeleton} for every node $C_h$ of $T$, which is constructed by starting with a complete undirected
subgraph and removing an undirected edge $u\erelbar{00} v$ if there is a subset $S$ of $C_h$ such that $u$ and $v$ are independent
conditional on $S$. For this purpose, we can use the \opc~algorithm \cite{penea12amp} or the \spc~algorithm (Algorithm \ref{algAMPstablePC}, line 2-13) in Algorithm \ref{DBalgampcg} (line 3-11). Second, in order to construct the \textit{global skeleton} (line 13-23 of Algorithm \ref{DBalgampcg}), we combine all these local skeletons together. Note that it is possible that some edges that are present in some local skeletons may be absent in other local skeletons. Also, two non-adjacent vertices $u$ and $v$ in the AMP CG $G$ that belong to the same chain component may be adjacent in the temporary global skeleton. (Note that Theorem \ref{thm_main_ampcg} only guarantees the existence of the $p$-separators for those non-adjacent vertices that do not belong to the same chain component. In Appendix A, we provide an example that shows that Theorem \ref{thm_main_ampcg} cannot be strengthened.) In order to remove the extra edges in the resulting undirected graph, we apply a removal procedure that is similar to the skeleton recovery phase of the \opc~algorithm. However, instead of the complete undirected graph we use the resulting undirected graph obtained in the previous step. \textit{(b) Orientation phase}. In this phase (line 25-28 of Algorithm \ref{DBalgampcg}), we orient undirected edges using rules R1-R4 in \citep{penea12amp, PENA2016MAMP} (illustrated in Figure \ref{fig:rules_ampcgs} for the reader's convenience). 
The whole process is formally described in Algorithm \ref{DBalgampcg}.

\begin{algorithm}[!htbp]
\caption{\lcd~: A decomposition-based recovery algorithm for AMP CGs}\label{DBalgampcg}
	\SetAlgoLined
	\KwIn{A probability distribution $p$ faithful to an unknown AMP CG $G$.}
	\KwOut{A chain graph $H$ that is triplex equivalent to the AMP CG $G$.}
    Construct a $p$-separation tree $T$ with a node set $C = \{C_1, \dots, C_I \}$ as discussed in Section \ref{septree_ampcg}\;
    Set $S=\emptyset$\;
    \tcc{\textcolor{blue}{\textbf{Local skeleton recovery:}}}
    \tikzmk{A}
    \For{$i\gets 1$ \KwTo $I$}{
    Start from a complete undirected graph $\bar{G}_i$ with         vertex set $C_i$\;
    \For{\textrm{each vertex pair $\{u,v\}\subseteq C_i$ }}{\If{$\exists S_{uv}\subseteq C_i \textrm{ such that } u \perp\!\!\!\perp v | S_{uv}$}{
    Delete the edge $u\erelbar{00} v$ in $\bar{G}_i$\;
    Add $S_{uv}$ to $S$;
    }
    }
    }
    \tikzmk{B}
 \boxit{blue!60}
    \tcc{\textcolor{blue}{\textbf{Global skeleton recovery:}}}
    \tikzmk{A}
    Initialize the edge set $\bar{E}_V$ of $\bar{G}_V$ as the union of all edge sets of $\bar{G}_i, i=1,\dots, I$\;
    Set $H=\bar{G}_V$\;
    \For{$i\gets 0$ \KwTo $|V_H|-2$}{
        \While{possible}{
            Select any ordered pair of nodes $u$ and $v$ in $H$ such that $u\in ad_H(v)$ and $|[ad_H(u)\cup ad_H(ad_H(u))]\setminus \{u,v\}|\ge i$;
            \tcc{$ad_H(x):=\{y\in V| x\erelbar{01} y, y\erelbar{01} x, \textrm{ or }x\erelbar{00}y\}$}
            \If{\textrm{there exists $S\subseteq ([ad_H(u)\cup ad_H(ad_H(u))]\setminus \{u,v\})$ s.t. $|S|=i$ and $u\perp\!\!\!\perp_p v|S$ (i.e., $u$ is independent of $v$ given $S$ in the probability distribution $p$)}}{
                Set $S_{uv} = S_{vu} = S$\;
                Remove the edge $u\erelbar{00} v$ from $H$\;
            }
        }
    }
    \tikzmk{B}
 \boxit{green!60}
    \tcc{\textcolor{blue}{\textbf{Orientation phase \citep{penea12amp}:}}}
    \tikzmk{A}
    \While{possible}{
        Apply rules R1-R4 in Figure \ref{fig:rules_ampcgs} to $H$.
        
        \small\tcc{A block is represented by a perpendicular line at the edge end such as in $\erelbar{20}$ or $\erelbar{22}$, and it means that the edge cannot be a directed edge pointing in the direction of the block. Note that $\erelbar{22}$ means that the edge must be undirected. The ends of some of the edges in the rules are labeled with a circle such as in $\erelbar{30}$ or $\erelbar{33}$. The circle represents an unspecified end, i.e. a block or nothing.}
    }
    Replace every edge $\erelbar{20}$ ($\erelbar{22}$) in $H$ with $\erelbar{01}$ ($\erelbar{00}$)\;
    \tikzmk{B}
 \boxit{red!60}
    return $H$.
    
    
\end{algorithm}

We prove that the global skeleton and all triplexes obtained by applying the decomposition in Algorithm \ref{DBalgampcg} are correct, that is, they are the same as those obtained from the joint distribution of $V$. In other words, \lcd~ returns a chain graph that is a member of a class of triplex
equivalent AMP chain graphs; see Appendix A for  proof details. Note that separators in a $p$-separation tree may not be complete in the augmented graph.
Thus the decomposition is weaker than the decomposition usually defined for parameter estimation \citep{cdls,l}.

\begin{remark}
One can apply Algorithm 3 in \citep{roverato06} to to the resulting chain graph  of Algorithm \ref{DBalgampcg} to obtain the largest deflagged graph. Also, one can apply Algorithm 1 in \citep{sonntagpena15} to the resulting chain graph  of Algorithm \ref{DBalgampcg} to obtain the AMP essential graph.
\end{remark}

\subsection{Complexity Analysis of the \lcd~Algorithm}\label{complexity}
Here we start by comparing our algorithm with the main algorithm in \citep{xie} that is designed
specifically for DAG structural learning when the underlying graph structure is a DAG. We make
this choice of the DAG specific algorithm so that both algorithms can have the same separation tree
as input and hence are directly comparable.

The same advantages mentioned by \citep{xie} for their BN structural learning algorithm hold for our algorithm when applied to AMP CGs. For the reader's convenience, we list them here. 
First, by using the $p$-separation tree, \textit{independence tests are performed only conditionally on smaller sets
contained in a node of the $p$-separation tree rather than on the full set of all other variables}. Thus our algorithm has
\textit{higher power for statistical tests}.
Second, the \textit{computational complexity can be reduced}.  The number of
conditional independence tests for constructing the equivalence class is used as characteristic operation for this complexity analysis. Decomposition of graphs is a computationally
simple task compared to the task of testing conditional independence for a large number of triples of sets of variables. The triangulation of an undirected graph is used in our algorithms to construct a $p$-separation tree from an undirected independence graph. Although the problem for optimally triangulating an undirected graph is NP-hard, sub-optimal triangulation methods \citep{bbhp} may be used provided
that the obtained tree does not contain too large nodes to test conditional independencies. Two of the best known
algorithms are lexicographic search and maximum cardinality search, and their complexities are
$O(|V||E|)$ and $O(|V|+ |E|)$, respectively \citep{bbhp}. Thus in our algorithms, \textit{conditional independence tests dominate algorithmic complexity}.

For the sake of complexity analysis, Algorithm \ref{DBalgampcg} can be divided into four parts: (1) construction of the p-separation tree, (2) local skeleton recovery (lines 3--11), (3) global skeleton recovery (lines 12--22), and (4) orientation phase (lines 23-25).  Part (1) includes the construction of the UIG, which takes at most $O(n^2)$ conditional independence tests, where $n$ is the number of variables in the data set.  Part (2) and (3) together require $O(Hm^22^m)$ as claimed in \citep[Section 6]{xie}, where $H$ is the number of $p$-separation tree nodes (usually $H \ll |V|$) and $m=\max_h|C_h|$ where $|C_h|$ denotes the number of variables in $C_h$ ($m$ usually is much less than $|V|$). Part (4) does not require any conditional independence tests.

\section{Experimental Evaluation}\label{evaluation}
In this section we evaluate the performance of our algorithms in various setups
using simulated / synthetic data sets. We first compare the performance of our proposed algorithms, i.e., \spc, \cpc, \scpc~and \lcd~with the original \opc~
learning algorithms by running them
on randomly generated AMP chain graphs. We then compare our algorithms, i.e.,  \spc~and \lcd~algorithms with the \opc~algorithm on different discrete Bayesian networks such as \href{http://www.bnlearn.com/bnrepository/}{ASIA, INSURANCE, ALARM, and HAILFINDER} that have
been widely used in evaluating the performance of structural learning algorithms. Empirical simulations show that our algorithm achieves
competitive results with the \opc~and \spc~learning algorithms; in particular, in the Gaussian case the decomposition-based algorithm outperforms the \opc~and \spc~algorithms. 
Algorithms \ref{DBalgampcg} and the \opc~and \spc~algorithms have been implemented in the R language. All code, data, and the results reported here are
based on our R implementation available at the following GitHub link {\color{blue}{\url{https://github.com/majavid/AMPCGs2019}}}. We do not consider the case of mixed continuous and discrete data in this paper, and leave this important and complex issue for future work; we only observe that this problem has been studied in the case of Markov networks and Bayesian networks, for example see \citep{Edwards2010,lauritzen2001stable,raghu2018comparison,andrews2018scoring}.

\subsection{Performance Evaluation Metrics}
We evaluate the performance of the proposed algorithms in terms of the six measurements that are commonly used by~\citep{Colombo2014,mxg,Tsamardinos2006} for constraint-based learning algorithms: 
\begin{enumerate}
    \item[(a)]  the true positive
rate (TPR)\footnote{Also known as sensitivity, recall, and hit rate.} is the ratio of  the number of correctly identified edges over total number of edges  (in true graph), i.e., $$TPR=\frac{\textrm{true positive } (TP)}{\textrm{the number of real positive cases in the data } (Pos)},$$
\item[(b)]  the false positive rate (FPR)\footnote{Also known as fall-out.} is the ratio of the number of incorrectly identified edges over total number of gaps, i.e., $$FPR=\frac{\textrm{false positive }(FP)}{\textrm{the number of real negative cases in the data }(Neg)},$$ 
\item[(c)] the true discovery rate (TDR)\footnote{Also known as precision or positive predictive value.} is the ratio of  the number of correctly identified edges over total number of edges (both in estimated graph), i.e., $$TDR=\frac{\textrm{true positive } (TP)}{\textrm{the total number of edges in the recovered CG}},$$ 
\item[(d)] accuracy (ACC) is defined as $$ACC=\frac{\textrm{true positive }(TP) +\textrm{ true negative }(TN)}{Pos+Neg},$$
\item[(e)] the structural Hamming distance (SHD)\footnote{This is the metric described by \cite{Tsamardinos2006} to  compare the
structure of the learned and the original graphs.} is the number of legitimate operations needed to change the current resulting graph to the true CG,
where legitimate operations are: (i) add or delete an edge and (ii) insert, delete or reverse an edge
orientation,  and
\item[(e)] run-time for the chain graph recovery algorithms.
\end{enumerate}

Note that we use TPR, FPR, TDR, and ACC for comparing the skeletons of a learned structure and a ground truth graph. In principle, a large TDR, TPR and ACC, a small FPR and SHD indicate good performance.

To investigate the performance of the proposed learning methods in this paper, we use the same approach that \cite{mxg} used in evaluating the performance of the LCD algorithm on LWF chain graphs. We run our algorithms on randomly generated AMP chain graphs and then we compare the results and report summary error measures in all cases.

\subsection{Performance Evaluation on Random AMP Chain Graphs (Gaussian case)}

To investigate the performance of the proposed learning methods in this paper, we use the same approach that \citep{mxg} used in evaluating the performance of the LCD algorithm on LWF chain graphs. We run our algorithms on randomly generated AMP chain graphs and then we compare the results and report summary error measures in all cases.

\subsubsection{Data Generation Procedure}
First we explain the way in which the random AMP chain graphs and random samples are generated.
Given a vertex set $V$ , let $p = |V|$ and $N$ denote the average degree of edges (including undirected
and pointing out and pointing in) for each vertex. We generate a random AMP chain graph on $V$ as
follows:
\begin{itemize}
    \item Order the $p$ vertices and initialize a $p\times p$ adjacency matrix $A$ with zeros;
    \item For each element in the lower triangle part of $A$, set it to be a random number generated from a Bernoulli distribution with probability of occurrence $s = N/(p-1)$;
    \item Symmetrize $A$ according to its lower triangle;
    \item Select an integer $k$ randomly from $\{1,\dots,p\}$ as the number of chain components;

    \item Split the interval $[1, p]$ into $k$ equal-length subintervals $I_1,\dots,I_k$ so that the set of variables falling into each subinterval $I_m$ forms a chain component $C_m$; 

    \item Set $A_{ij} = 0$ for any $(i, j)$ pair such that $i \in I_l, j \in I_m$ with $l > m$.
\end{itemize}

This procedure yields an adjacency matrix $A$ for a chain graph with $(A_{ij} = A_{ji} = 1)$ representing an undirected edge between $V_i$ and $V_j$ and $(A_{ij} =1,  A_{ji} =0)$ representing a directed edge
from $V_i$ to $V_j$. Moreover, it is not difficult to see that $\mathbb{E}[\textrm{vertex degree}] = N$, where an adjacent vertex can
be linked by either an undirected or a directed edge.
In order to sample from the artificial CGs, we first transformed
them into DAGs and then sampled from these DAGs under marginalization and conditioning as indicated in \citep{PENA20141185}. The
transformation of an AMP CG $G$ into a DAG $H$ is as follows: First, every node $X$ in $G$ gets a new parent $\epsilon ^X$
representing an error term, which by definition is never observed. Then, every undirected edge $X\erelbar{00} Y$ in $G$ is replaced by $\epsilon ^X\erelbar{01} S_{XY}\erelbar{10} \epsilon ^Y$ where $S_{XY}$ denotes a selection bias node, i.e. a node that is always observed.
Given a randomly generated chain graph $G$ with ordered chain components $C_1,\dots,C_k$, we generate a Gaussian distribution on the corresponding transformed DAG $H$ using the \href{https://www.hugin.com/}{Hugin API}.
Note that the probability distributions of samples are likely to satisfy the faithfulness assumption, but there is no guarantee i.e., samples can have additional independencies that cannot be represented by the CG $G$.

\subsubsection{Experimental Results in Low-Dimensional Settings}

\paragraph{Experimental Setting}
In our simulation, we change three parameters $p$ (the number of vertices), $n$ (sample size) and
$N$ (expected number of adjacent vertices) as follows:
\begin{itemize}
    \item $p\in\{10, 20, 30, 40, 50\}$,
    \item $n\in\{500, 1000, 5000, 10000\}$, and
    \item $N\in\{2,3\}$.
\end{itemize}

For each $(p,N)$ combination, we first generate 30 random AMP CGs. We then generate a random Gaussian distribution based on each graph and draw an
identically independently distributed
(i.i.d.) sample of size $n$ from this distribution for each possible $n$. For each sample, three different
significance levels $(\alpha = 0.005, 0.01, 0.05)$ are used to perform the hypothesis tests. The \textit{null hypothesis} $H_0$ is ``two variables $u$ and $v$ are conditionally independent given a set $C$ of variables" and alternative $H_1$ is that $H_0$ may not hold. We then compare the results to access the influence of the significance testing level on the performance of our algorithms.

\begin{figure}[!htbp]
	\centering
	\includegraphics[scale=.28,page=1]{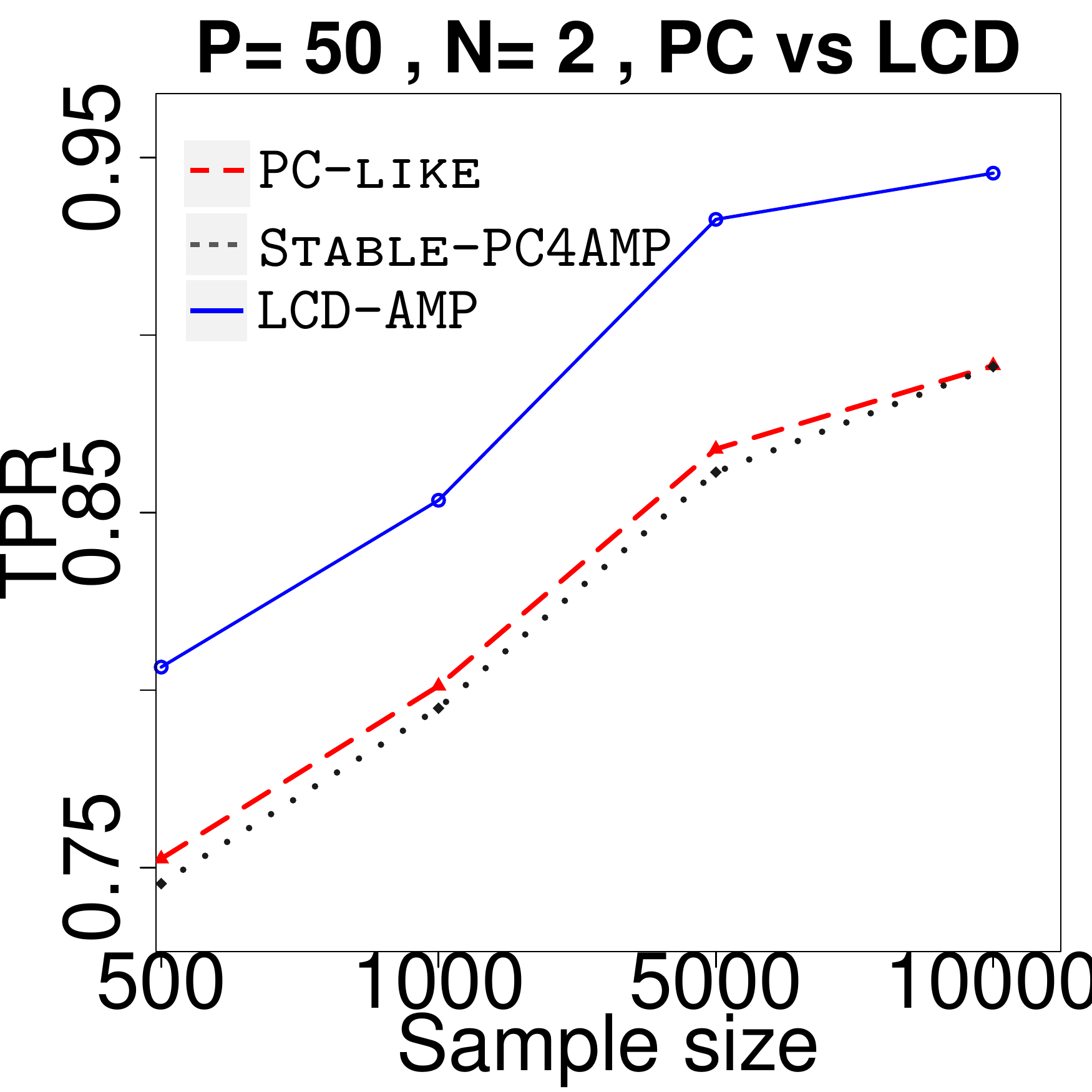}
	\includegraphics[scale=.28,page=2]{images/PCvsLCD50.pdf}
	\includegraphics[scale=.28,page=3]{images/PCvsLCD50.pdf}
	\includegraphics[scale=.28,page=4]{images/PCvsLCD50.pdf}
	\includegraphics[scale=.28,page=5]{images/PCvsLCD50.pdf}
	\includegraphics[scale=.28,page=6]{images/PCvsLCD50.pdf}
	\includegraphics[scale=.28,page=1]{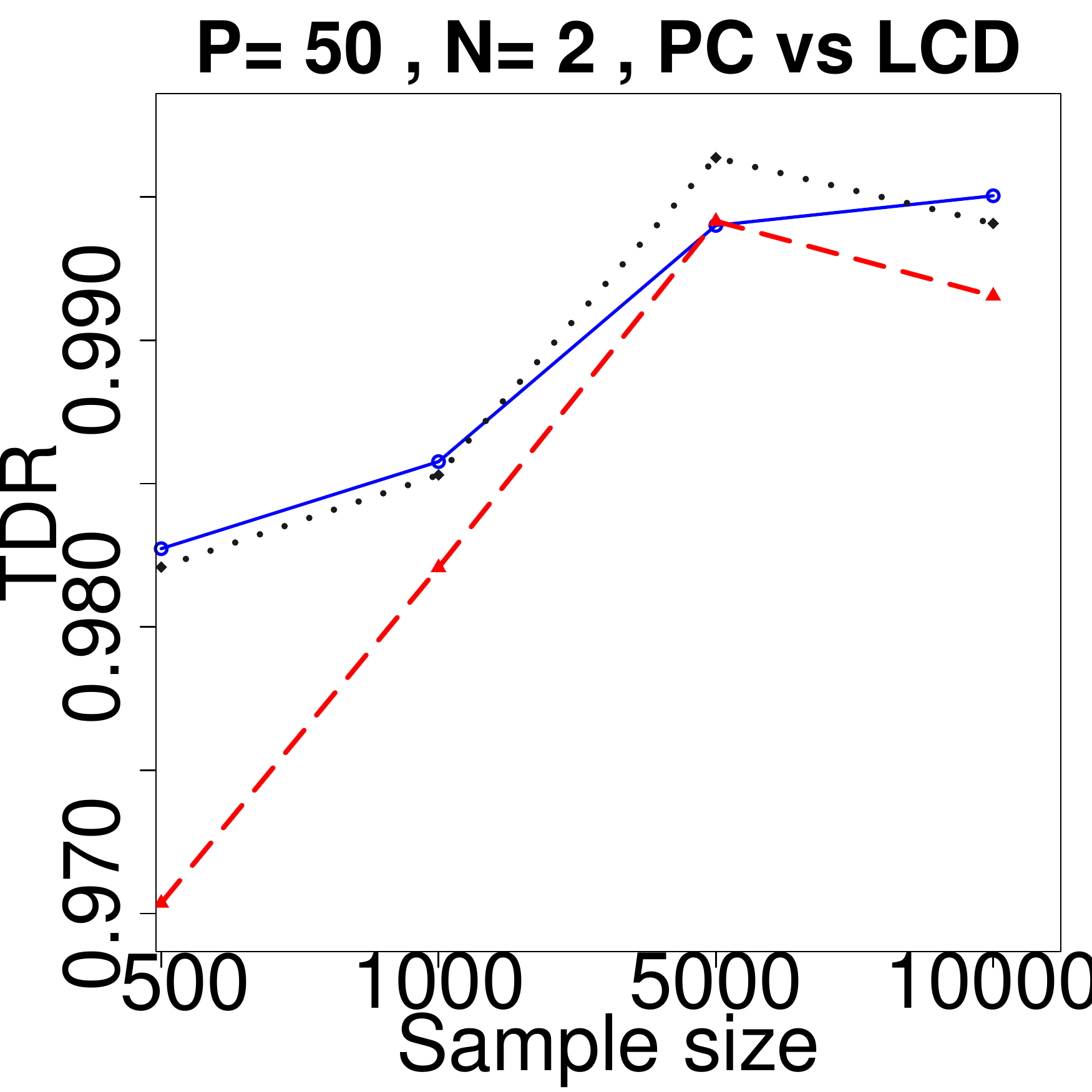}
	\includegraphics[scale=.28,page=2]{images/TDR5023pcvslcd.pdf}
	\includegraphics[scale=.28,page=3]{images/TDR5023pcvslcd.pdf}
	\includegraphics[scale=.28,page=7]{images/PCvsLCD50.pdf}
	\includegraphics[scale=.28,page=8]{images/PCvsLCD50.pdf}
	\includegraphics[scale=.28,page=9]{images/PCvsLCD50.pdf}
\caption{}
	\captionsetup{labelformat=empty}
\end{figure}
\begin{figure}[!htbp]
\ContinuedFloat
  \captionsetup{list=off}
	\centering
	\includegraphics[scale=.28,page=10]{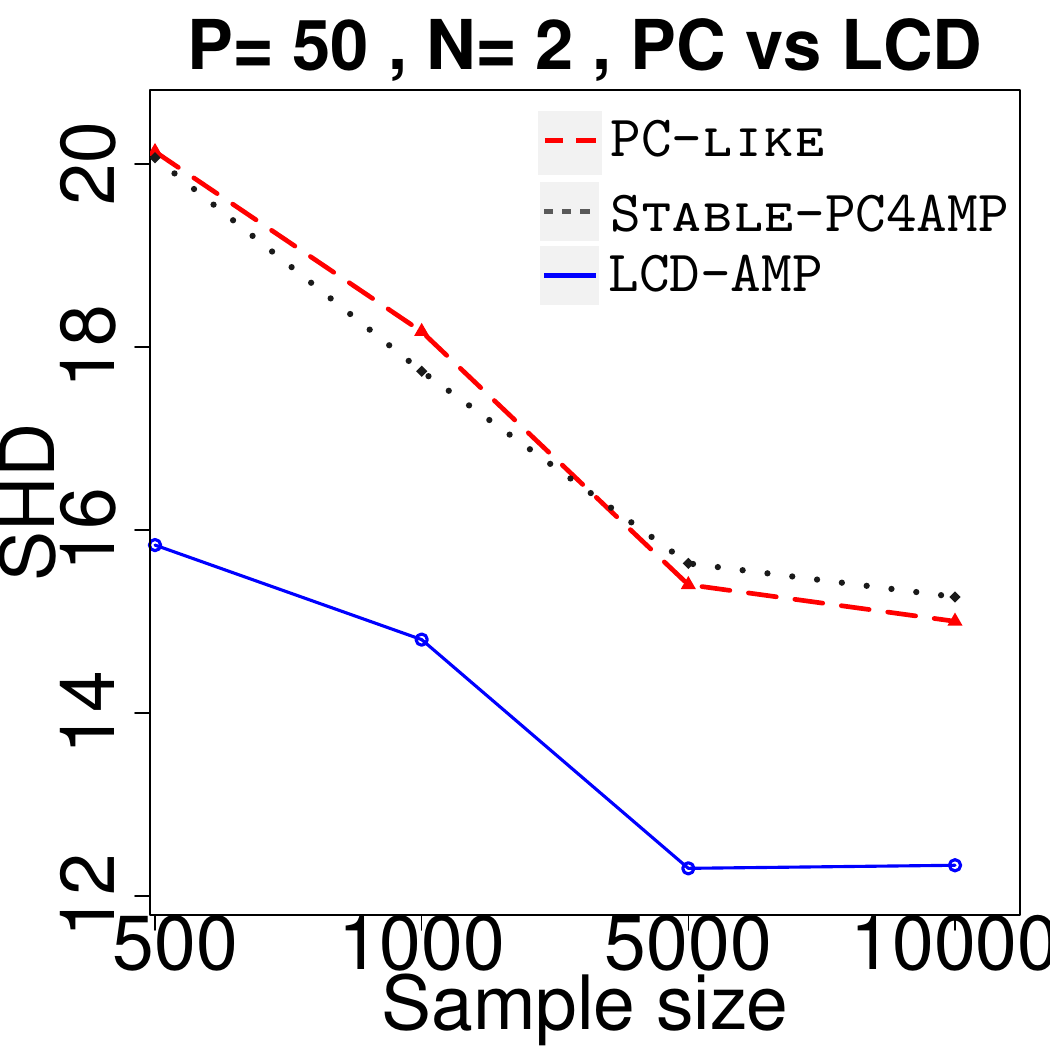}
	\includegraphics[scale=.28,page=11]{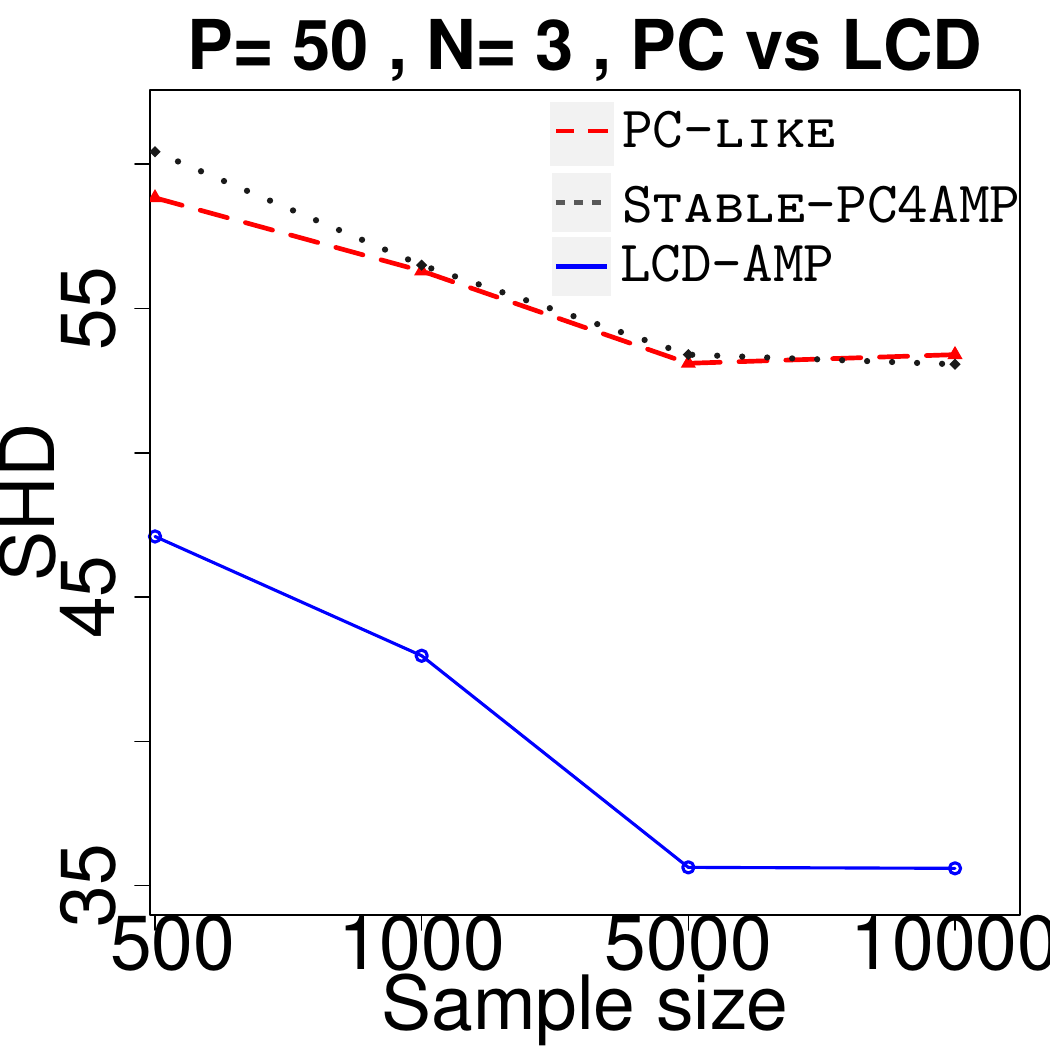}
	\includegraphics[scale=.28,page=12]{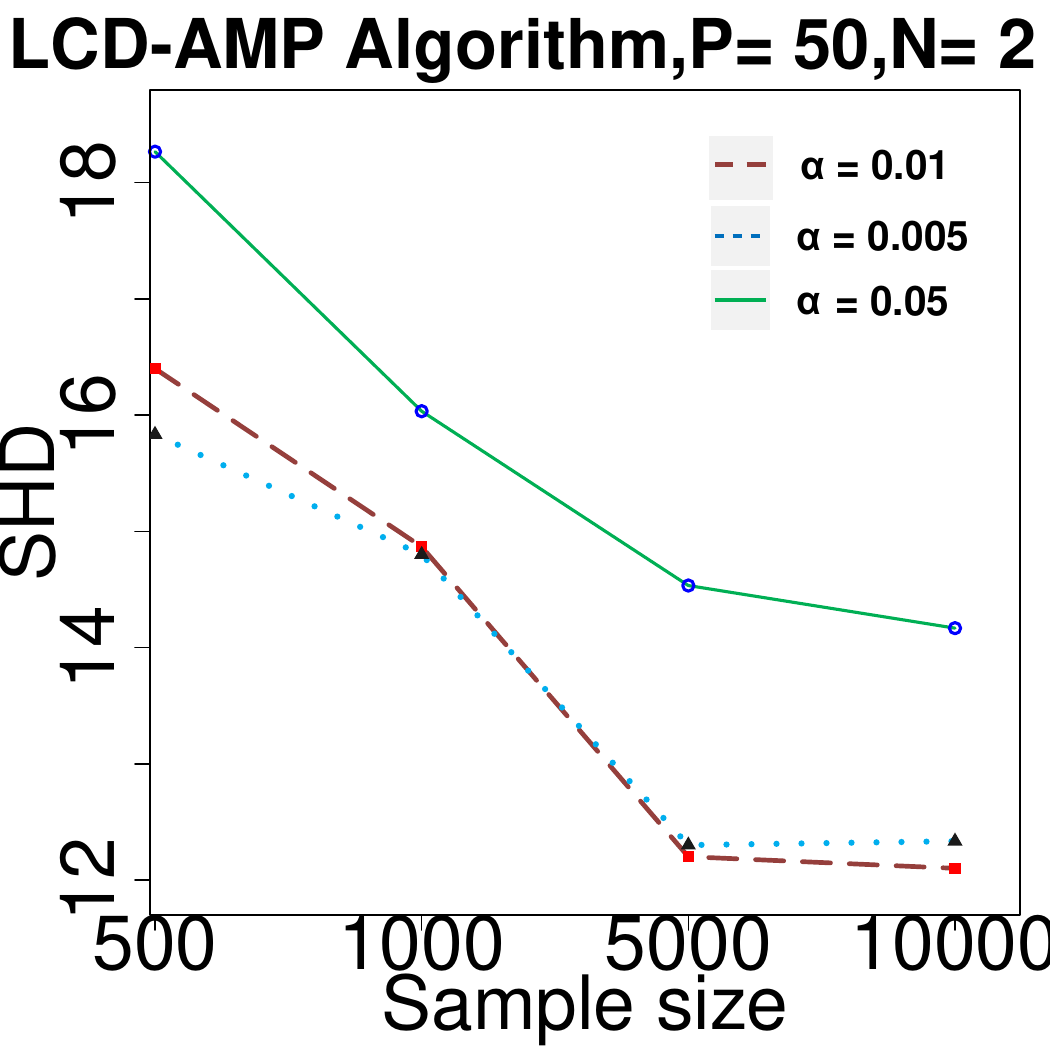}
	\caption{\footnotesize{First two columns show the performance of the decomposition based (\lcd), original \opc~and \spc~algorithms for randomly generated Gaussian chain graph models:
		average over 30 repetitions with 50 variables  correspond to N = 2, 3, and the significance level $\alpha=0.005$.  In each plot, the solid blue line corresponds to the \lcd~ algorithm,   the dashed red line corresponds to the original \opc~algorithm, and the dotted grey line corresponds to the stable \opc~(\spc) algorithm. The third column shows the performance of the decomposition-based (\lcd) algorithm for randomly generated Gaussian chain graph models:
		average over 30 repetitions with 50 variables  correspond to N = 2, and significance levels $\alpha=0.05,0.01,0.005$.  In each plot, the solid green line corresponds to $\alpha=0.05$,   the dashed brown line corresponds to $\alpha=0.01$, and the dotted blue line corresponds to $\alpha=0.005$.}}
	\label{fig:5023_05errors}
\end{figure}
\begin{figure}[ht]
    \centering
    \includegraphics[scale=.28,page=13]{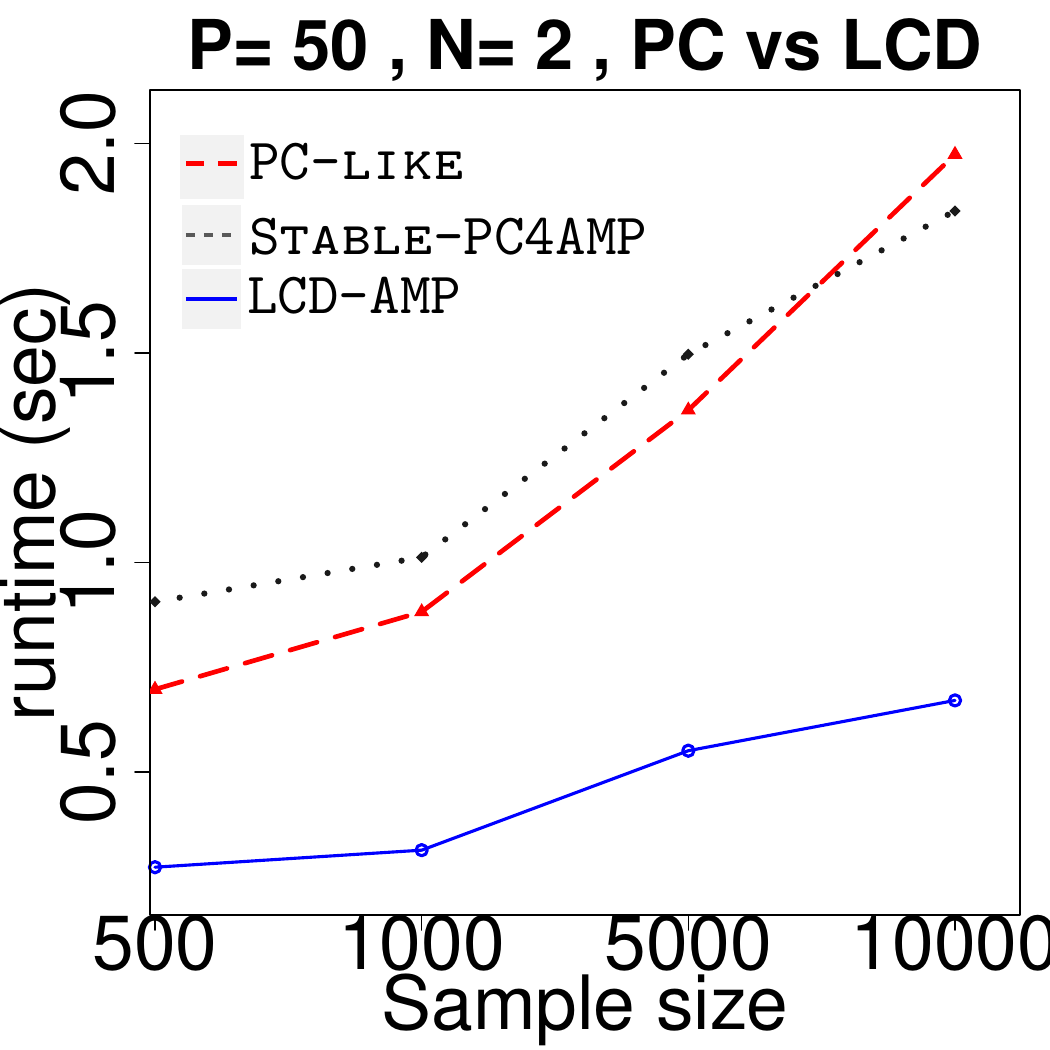}
	\includegraphics[scale=.28,page=14]{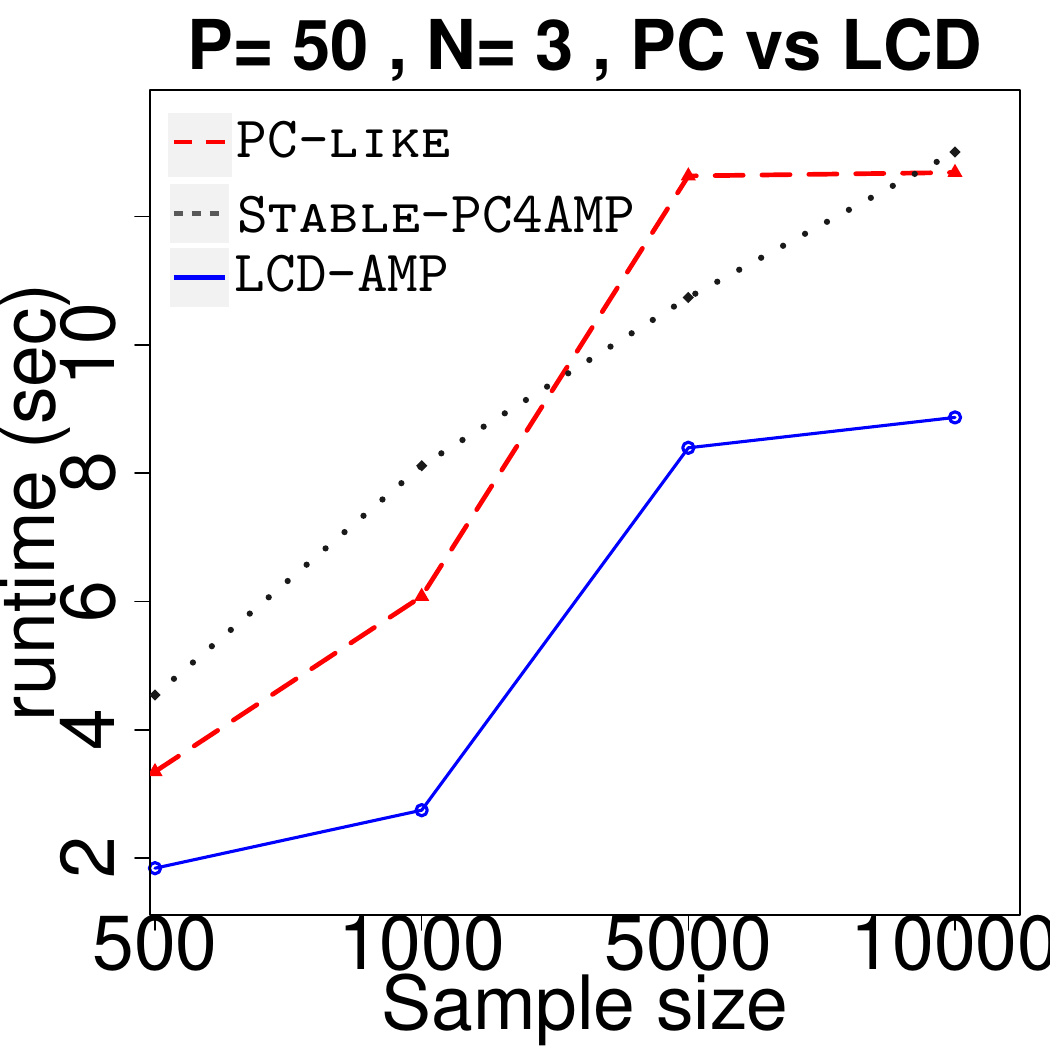}
	\includegraphics[scale=.28,page=15]{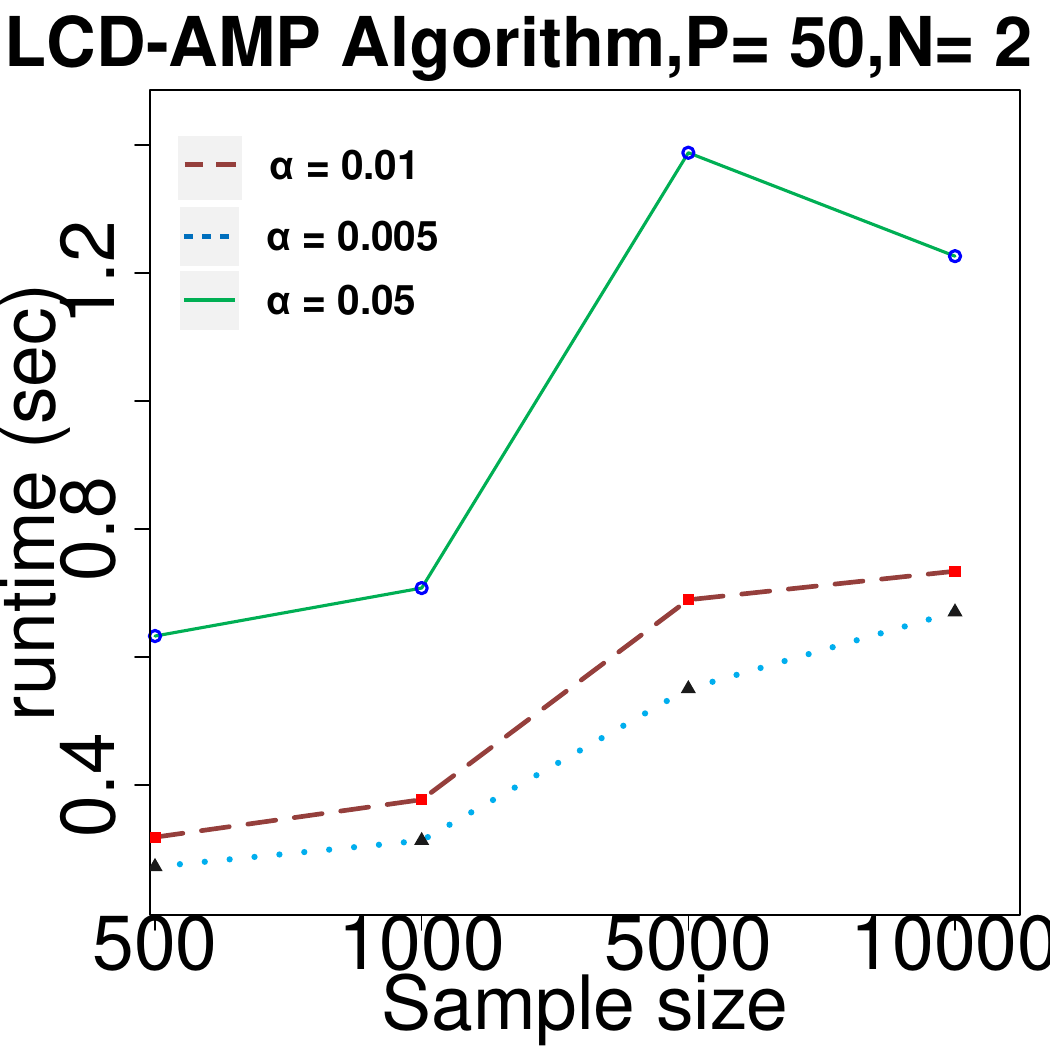}
    \caption{\footnotesize{First two columns show the running times of the decomposition-based (\lcd), original \opc~and \spc~algorithms for randomly generated Gaussian chain graph models:
		average over 30 repetitions with 50 variables  correspond to N = 2,3 and significance
		levels $\alpha=0.005$. In each plot, the solid blue line corresponds to the \lcd~ algorithm,   the dashed red line corresponds to the original \opc~algorithm, and the dotted grey line corresponds to the \spc~algorithm. The third column shows the running times of the decomposition-based (\lcd) algorithm for randomly generated Gaussian chain graph models:
		average over 30 repetitions with 50 variables  correspond to N = 2, and significance levels $\alpha=0.05,0.01,0.005$.  In each plot, the solid green line corresponds to $\alpha=0.05$,   the dashed brown line corresponds to $\alpha=0.01$, and the dotted blue line corresponds to $\alpha=0.005$.}}
    \label{fig:5023_05runtime}
\end{figure}
\begin{figure}[ht]
    \centering
    \includegraphics[scale=.28,page=1]{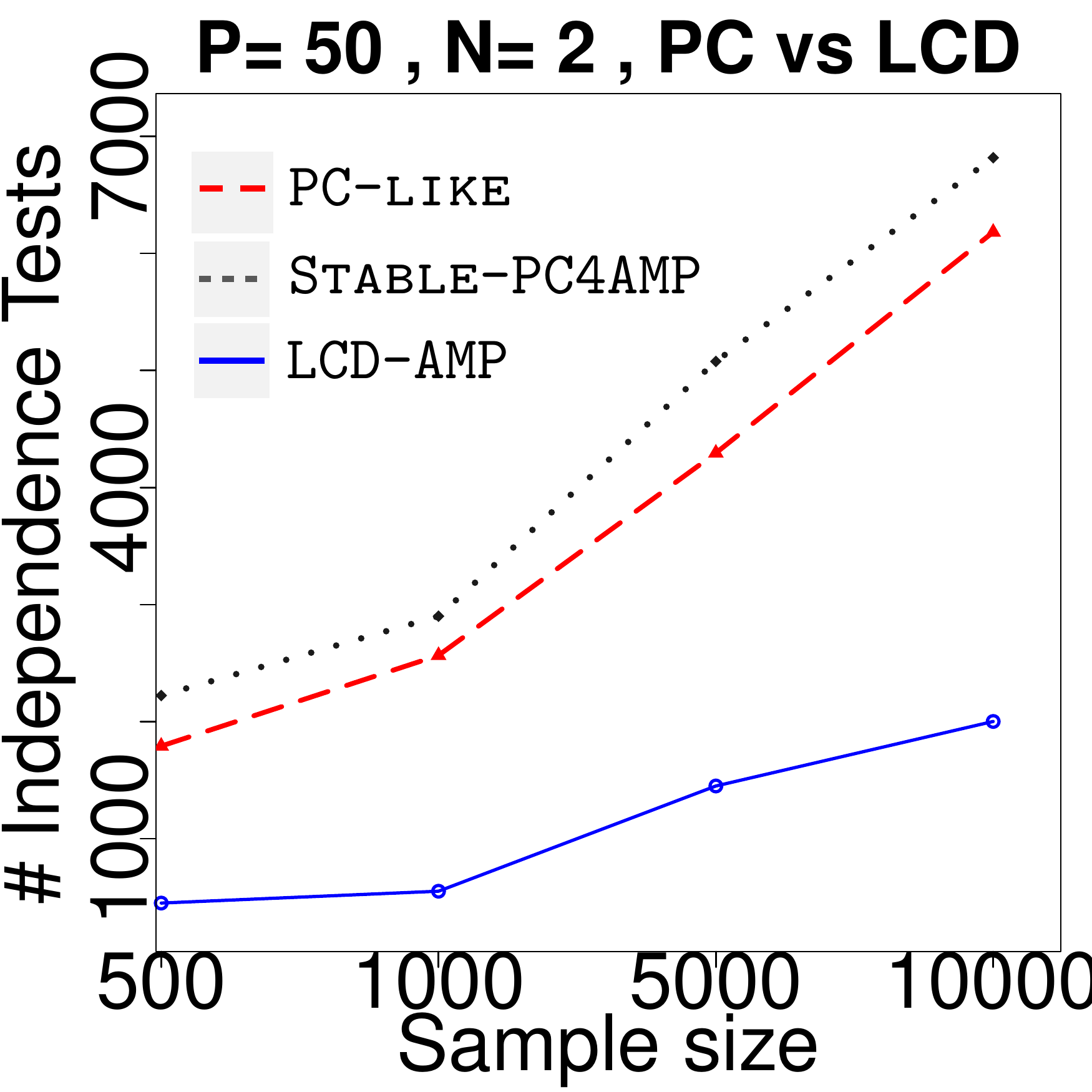}
	\includegraphics[scale=.28,page=2]{images/nIndTests_PCvsLCD5023_005.pdf}
	\includegraphics[scale=.28,page=3]{images/nIndTests_PCvsLCD5023_005.pdf}
    \caption{\footnotesize{First two columns show the number of independence tests used by the decomposition-based (\lcd), original \opc~and \spc~algorithms for randomly generated Gaussian chain graph models:
		average over 30 repetitions with 50 variables  corresponding to average degrees N = 2,3 and significance
		level $\alpha=0.005$. In each plot, the solid blue line corresponds to the \lcd~ algorithm,   the dashed red line corresponds to the original \opc~algorithm, and the dotted grey line corresponds to the \spc~algorithm. The third column shows the number of independence tests used by the decomposition-based (\lcd) algorithm for randomly generated Gaussian chain graph models:
		average over 30 repetitions with 50 variables  corresponding to average degree N = 2, and significance levels $\alpha=0.05,0.01,0.005$.  In each plot, the solid green line corresponds to $\alpha=0.05$,   the dashed brown line corresponds to  $\alpha=0.01$, and the dotted blue line corresponds to $\alpha=0.005$.}}
    \label{fig:5023_005nindtests}
\end{figure}

\paragraph{Results}
The experimental results in Figure \ref{fig:5023_05errors} shows that: 
\begin{itemize}
    \item[(a)] Both algorithms work well on sparse graphs $(N = 2,3)$.
    \item[(b)] For both algorithms, typically the TPR, TDR, and ACC increase with sample size.
    \item[(c)] The SHD and FPR decrease with sample size.
    \item[(d)] A large significance level $(\alpha=0.05)$ typically yields large
TPR, FPR, and SHD.
\item[(e)] In almost all cases, the performance of the \lcd~algorithm based on all error measures i.e., TPR, FPR, TDR, ACC, and SHD is better than the performance of the \opc~and \spc~algorithms.
\item[(f)] In most cases, error measures based on $\alpha=0.01$ and $\alpha=0.005$ are very close. Generally, our empirical results suggests that in order to obtain a better performance, we can choose a small value (say $\alpha=0.005$ or 0.01) for
the significance level of individual tests along with large sample if at all possible. However, the optimal value for a desired overall error rate may depend on the sample size, significance level, and the sparsity of the underlying graph.
\item[(g)] While the \spc~algorithm has a better TDR and FPR in comparison with the original \opc~algorithm, the original \opc~algorithm has a better TPR as observed in the case of DAGs~\citep{Colombo2014}. This can be explained by the fact that the \spc~algorithm tends to perform more tests than the original \opc~algorithm.
\item[(h)] There is no meaningful difference between the performance of the \spc~algorithm  and the original \opc~algorithm in terms of error measures ACC and SHD.
\end{itemize} 


When considering average running times versus sample sizes, as shown in Figures \ref{fig:5023_05runtime}, we observe that:
\begin{itemize}
    \item[(a)]  The average run time increases when sample size increases.
    \item[(b)] The average run times based on $\alpha=0.01$ and $\alpha=0.005$ are very close and in all cases better than $\alpha=0.05$,  while
choosing $\alpha=0.005$ yields a consistently (albeit slightly) lower average run time across all the settings.
\item[(c)] Generally, the average run time for the decomposition-based algorithm is lower than that for the ( {\sc Stable-}) \opc~algorithm.
\end{itemize}

In Figure~\ref{fig:5023_005nindtests}, the algorithms are compared by counting the number of independence tests, rather than runtime, in order to reduce the impact of different implementations (\textsf{R} packages). We observe that:
\begin{itemize}
    \item[(a)]  The average number of independence tests increases when sample size increases.
    \item[(b)] The average number of independence tests based on $\alpha=0.01$ and $\alpha=0.005$ are close and in all cases better than $\alpha=0.05$,  while
choosing $\alpha=0.005$ yields a consistently lower average number of independence tests across all the settings.
    \item[(c)] Generally, the average number of independence tests for the decomposition-based algorithm is better than that for the ({\sc Stable-}) \opc~algorithm.
\end{itemize}

These observations are consistent with the theoretical complexity analysis that we discussed in Section \ref{complexity}. In fact, our findings confirm that the decomposition-based algorithm \textit{reduces complexity} and \textit{increases the power of computational independence tests}.

\subsubsection{Experimental Results in High-Dimensional Settings}

Although the results in Figure \ref{fig:shd} show that our proposed modifications of \opc, i.e., \spc, \cpc, and \scpc~provide stabler estimations and closer to the true underlying structure in sparse high-dimensional settings for simulated Gaussian data compared with \opc, we are interested to test whether the difference is statistically significant.

\paragraph{Experimental Setting}
To show that the order-dependence of \opc~algorithm is
problematic in high-dimensional data, we compared the SHD of the original \opc~algorithm against its modifications for randomly generated Gaussian chain graph models: average over 30 repetitions with 1000 variables with $N = 2$, sample size $S=50$, and the significance level $\alpha=0.05,0.01,0.005,0.001$. We used an \textit{independent t-test} to quantitatively evaluate whether the means of SHDs in different structure discovery algorithms are different.


\paragraph{Results}
The \textit{t-test} results in Tables \ref{t:ttestPCvsModifications} and \ref{t:ttestModifiedPCs} show that: 
\begin{itemize}
    \item[(a)] Except for the p-value $\alpha=0.001$, the mean SHD of our proposed algorithms (i.e., \spc, \cpc, and \scpc) is significantly different (lower) from the mean of \opc's SHD. This confirms that our proposed modifications provide more reliable and better-learned structures in comparison with the \opc~algorithm.
    \item[(b)] The mean of \scpc's SHD is  significantly different from the mean of \spc and \cpc algorithms' SHD for the p-value $\alpha=0.05$. However, for the other p-values the difference is not meaningful.
    \item[(c)] Taken together, the quantitative t-test analysis confirms what one would expect from visual inspection of Figure \ref{fig:shd}.
\end{itemize}

In addition to \textit{t-test}, we also performed \textit{F-test} to test statistical difference between the corresponding pairwise SHD variances. Our results show that the p-value of F-test in all pairwise comparisons between all algorithms (\spc,\cpc, \scpc, and \opc) is greater than the significance level $\alpha=0.05$. In conclusion, there is no significant difference between the variances of the pairwise SHDs. The similarity of SHD variances indicates that requiring stability does not control error propagation in constraint-based algorithms, and that there remains a common source of errors to be discovered in future work.

\begin{figure}[!ht]
    \centering
    \includegraphics[scale=.5]{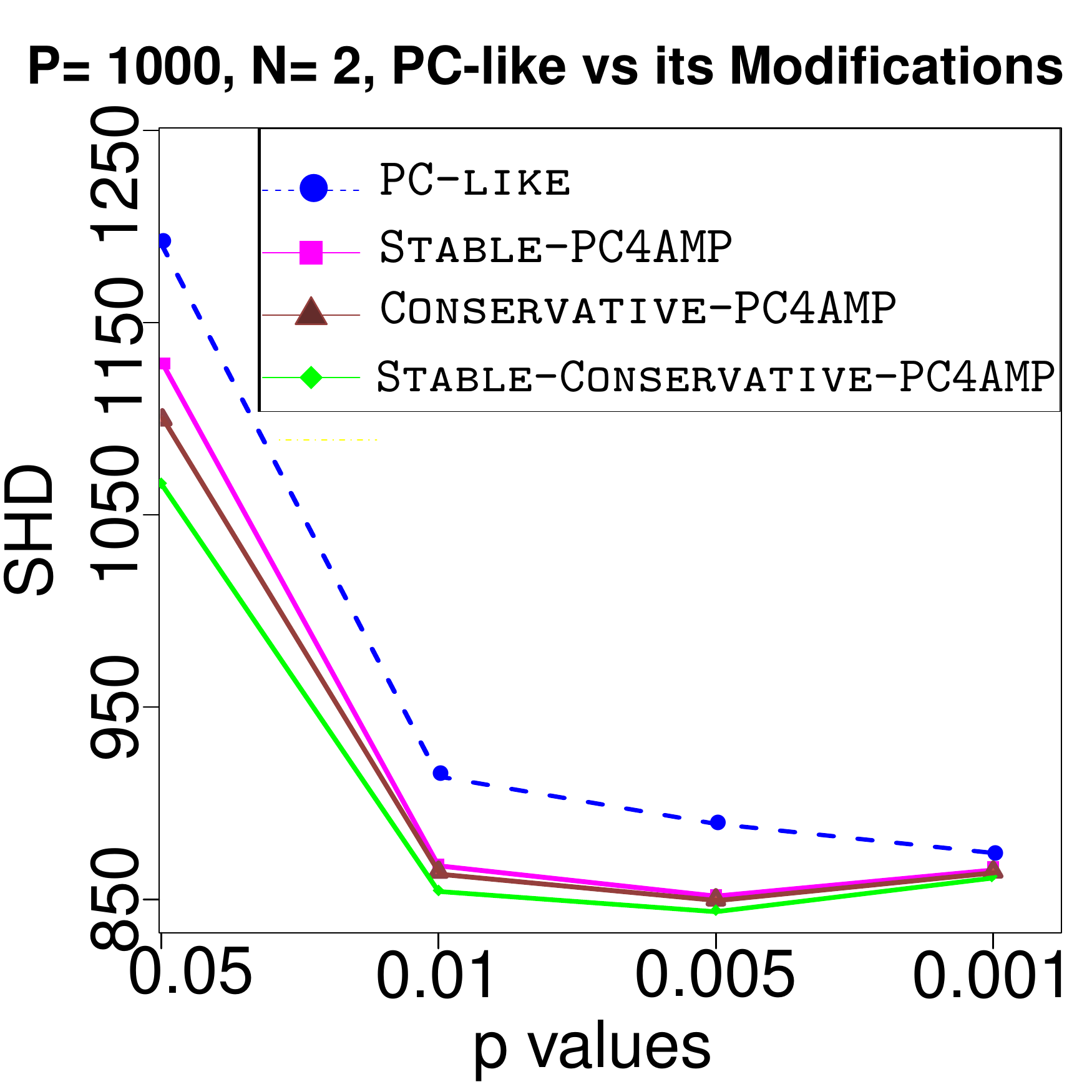}
    \caption{The SHD of the original \opc~algorithm against its modifications for randomly generated Gaussian chain graph models: average over 30 repetitions with 1000 variables  correspond to $N = 2$, sample size $S=50$, and the significance level $\alpha=0.05,0.01,0.005,0.001$.}
    \label{fig:shd}
\end{figure}

\begin{table}[!htpb]
\caption{P-values for pairwise t-tests. Bold numbers in the table mean that \opc’s average SHD is significantly different from other’s average SHD with the given p-value ($\alpha$).}\label{t:ttestPCvsModifications}
\centering
\begin{tabular}{l|c|c|c|c}
 & \opc~& \opc~& \opc~& \opc\\
 & ($\alpha=0.05$) & ($\alpha=0.01$) & ($\alpha=0.005$) & ($\alpha=0.001$)\\
\midrule
\spc & $\mathbf{7.84\mathrm{e}{-11}}$ & $\mathbf{3.338\mathrm{e}{-06}}$ & $\mathbf{0.0001399}$ & $0.1781$\\
    \midrule
\cpc & $\mathbf{9.8\mathrm{e}{-13}}$ & $\mathbf{1.682\mathrm{e}{-05}}$ & $\mathbf{0.001254}$ & 0.6045\\
\midrule
\scpc & $\mathbf{4.997\mathrm{e}{-16}}$ & $\mathbf{7.368\mathrm{e}{-07}}$ & $\mathbf{0.0001269}$ & 0.3511\\
\bottomrule
\end{tabular}
\end{table}

\begin{table}[!htpb]
\caption{P-values for pairwise t-tests. Bold numbers in the table mean that the average SHD is significantly different when executing the given pair of algorithms for the p-value $\alpha=0.05$.}\label{t:ttestModifiedPCs}
\centering
\begin{tabular}{l|c|c|c}
 &\sc \texttt{Stable-} &\sc \texttt{Conservative-} &\sc \texttt{Stable-Conservative-}\\
  & \texttt{PC4AMP} & \texttt{PC4AMP} & \texttt{PC4AMP}\\
\midrule
\sc \texttt{Stable-PC4AMP} & - & 0.116& $\mathbf{0.0003893}$ \\
    \midrule
\sc \texttt{Conservative-PC4AMP} & 0.116 & - & $\mathbf{0.04492}$\\
\midrule
\sc \texttt{Stable-Conservative-PC4AMP} & $\mathbf{0.0003893}$ & $\mathbf{0.04492}$ & -\\
\bottomrule
\end{tabular}
\end{table}

\subsection{Performance on Discrete Bayesian Networks}
Since Bayesian networks are special cases of AMP CGs, it is of interest to see whether our proposed algorithms still work well when the data are actually generated
from a Bayesian network. This matters because we often do not have the information
that the underlying graph is a DAG, which is usually untestable from data alone. For this purpose, we perform simulation studies for four well-known Bayesian networks from \href{http://www.bnlearn.com/bnrepository/}{Bayesian Network Repository} \citep{Scutari17}:  ASIA, INSURANCE, ALARM, and  HAILFINDER. We purposefully selected these networks because they have different sizes (from small to large numbers of nodes, edges, and parameters), and they are often used to evaluate structure learning
algorithms.
We briefly introduce these networks here:
\begin{itemize}
    \item ASIA \citep{asia} with 8 nodes, 8 edges, and 18 parameters, it describes the diagnosis of a patient at a chest clinic who may have just come back from a trip to Asia and may be showing dyspnea. Standard constraint-based learning algorithms are not able to recover the true structure of the network because of the presence of a functional node. 
    \item INSURANCE \citep{insurance} with 27 nodes, 52 edges, and 984 parameters, it evaluates car insurance risks.
    \item ALARM \citep{alarm} with 37 nodes, 46 edges and 509 parameters, it was designed by medical experts to provide an alarm message system for intensive care unit patients based on the output a number of vital signs monitoring devices.
    \item HAILFINDER \citep{Hailfinder} with 56 nodes, 66 edges, and 2656 parameters, it was designed to forecast severe summer hail in northeastern Colorado.
\end{itemize}

We compared the performance of our algorithms for these Bayesian networks for  significance level $\alpha=0.05$. The Structural Hamming Distance (SHD) compares the structure of the largest deflagged of the learned and the original networks, for a fair comparison. 

\begin{table}[!ht]
\caption{Results for discrete samples from the ASIA (5000 observations), INSURANCE (20000 observations), ALARM (20000 observations), and HAILFINDER (20000 observations) networks from the \textsf{bnlearn R} package respectively. Each row corresponds to the significance
level: $\alpha=0.05$. In order to learn an undirected independence graph from a given data set in the \lcd~ algorithm we used the Incremental Association with FDR (\iamb) algorithm \citep{Pena08Mb} from the \textsf{bnlearn R} package \citep{Scutari17} and the stepwise forward selection (\aic~or \bic) algorithms in \citep{Edwards2010}.}
\centering
\resizebox{!}{.3\textheight}{
\begin{tabular}{l|c|c|c|c|c}
Algorithm & TPR & TDR & FPR&ACC& SHD\\
\midrule
   \lcd~ Algorithm (\iamb) & 0.5 & 0.8 &	0.05 &	0.821 &	7\\
   \lcd~ Algorithm (\aic) & 0.75& \textbf{1} &	\textbf{0} &	0.929 &	3\\
   \lcd~ Algorithm (\bic) & \textbf{0.875}& \textbf{1} &	\textbf{0} &	\textbf{0.964} &	\textbf{1}\\
    \midrule
 \spc~Algorithm  &0.5& \textbf{1} &	\textbf{0}&	0.8571&	5\\
\midrule
Original \opc~Algorithm  &0.5& \textbf{1} &	\textbf{0}&	0.8571&	5\\
\midrule
\midrule
\lcd~ Algorithm (\iamb)& \textbf{0.558} & 0.935 &	0.0067 &	\textbf{0.929} &	\textbf{33}\\
\lcd~ Algorithm (\aic) & 0.385 & 0.952	& 0.0033 & 0.906 &	42  \\
\lcd~Algorithm (\bic) & 0.538 & 0.875 &	0.0134 &	0.920 &	36 \\
\midrule
\spc~Algorithm & 0.173 & \textbf{1} &	\textbf{0} &	0.877 &	43\\
\midrule
Original \opc~Algorithm  &0.346 & \textbf{1} &	\textbf{0}& 0.903 &	41\\
\midrule
\midrule
\lcd~ Algorithm (\iamb)& \textbf{0.783} & 0.878 &	0.0081 &	0.977 &	24\\
   \lcd~ Algorithm (\aic) & 0.696 & 0.914 & 0.0048 & 0.974 & 27 \\
\lcd~ Algorithm (\bic) & 0.760 & 0.921 &	0.0048 &	\textbf{0.979} &	20\\
    \midrule
\spc~Algorithm & 0.587 & \textbf{1} &	\textbf{0} &	0.971 &	25\\
\midrule
Original \opc~Algorithm  &0.696& \textbf{1} &	\textbf{0}& \textbf{0.979} &	\textbf{18}\\
\midrule
\midrule
\lcd~ Algorithm (\iamb)&0.515 & 0.971 &	0.00068 &	0.979 &	40\\
\lcd~ Algorithm (\aic) &  & &  & & \\
\lcd~ Algorithm  (\bic) & \textbf{0.803} & 0.930 &	0.0027 &	0.989 &	\textbf{38}\\
\midrule
\spc~Algorithm &0.394& \textbf{1} &	\textbf{0}&	0.974&	46\\
\midrule
Original \opc~Algorithm  &0.455&	0.811 &0.0047 &	0.972 &	49\\
\bottomrule
\end{tabular}}\label{t:discreteAMPCGs}
\end{table}

\subsubsection{Experimental Results}
The results of comparing all learning methods in Table \ref{t:discreteAMPCGs} indicate that the performance of {\lcd~} algorithm in many cases is better than that of the \opc~and \spc~algorithms. In particular, we observed: 
\begin{itemize}
    \item[(a)] Although the performance of our \lcd~ algorithm, overall, is better than the \opc~and \spc~algorithms, it is highly variable depending on the procedure that is used for the UIG discovery, especially in TPR and SHD. One of the most important implications of this observation is that there is much room for improvement to the UIG recovery algorithms and decomposition-based learning algorithms, and hopefully the present paper will inspire other researchers to address this important class of algorithms. In general, the more accurate the UIG discovery algorithm, the more robust the result. In our experiments, generally, the \iamb~ algorithm \citep{Pena08Mb} and the stepwise forward selection (FWD-BIC) algorithm \citep{Edwards2010} are more effective as a preliminary step (UIG recovery) towards understanding the overall dependence structure of high-dimensional discrete data.
    \item[(b)] The \opc~and \spc~algorithms tend to have better TDR and FPR. This comes at the expense, however, of much worse TPR. This suggests that the \opc~and \spc~algorithms tend to add too many  edges to the skeleton of the learned graph.
\end{itemize}


\section{Related Work}
The contributions of this paper regarding finding minimal separators and structure learning algorithms intersect with various works in the literature as follows. 
\subsection{Finding Minimal Separators in Probabilistic Graphical Models} 
A challenging task of model testing is to detect for any given pair of nodes a minimal or minimum separator.  Nontrivial algorithms for testing and for finding a minimal $d$-separator in a DAG were first proposed in \citep{ad,tpp}. An algorithm for learning the structure of Bayesian networks
from data, based on the idea of finding minimal
$d$-separating sets, proposed in \citep{ACID2001}. In \citep{van2019finding}, the authors show that testing and finding a minimal separator in DAGs can be done in linear time. As shown in \citep{van2019finding,van2019separators}, (minimal) separating sets have important applications in causal inference tasks like finding (minimal) covariate adjustment sets or conditional
instrumental variables. In \citep{jv-pgm18,jv-uai18}, the authors proposed algorithms for testing and for finding a minimal separator in an LWF CG and an MVR CG, respectively. In this paper, we proposed algorithms for testing and finding minimal separators in AMP chain graphs (see Section \ref{sec:findingminimals}). 

\subsection{\opc~Algorithms for Probabilistic Graphical Models}
The PC algorithm  proposed by \textbf{P}eter Spirtes and \textbf{C}lark Glymour \citep{sgs} learns the
Bayesian network structure from data by testing for conditional
independence between various sets of variables.
Given the results of these tests, a network pattern is constructed so that the Markov
property holds and $d$-separation confirms the resulting graph mirroring those conditional independencies found in the data. The PC algorithm consists of  two phases: In the first phase, an undirected graph is learned. This is known as the skeleton of the
Bayesian network. In the second phase, arrowheads are added to some of the edges
where they can be inferred. The output graph may not be fully oriented and is called a
pattern. When the pattern contains undirected edges, these indicate that the data are
consistent with models in which either orientation is possible. 

The PC algorithm is known to be order-dependent,
in the sense that the output can depend on the order in which the variables are given.
This order-dependence can be very pronounced in high-dimensional settings, where it can lead to highly
variable results. In order to resolve  the order-dependence problem, Colombo and Maathuis \citep{Colombo2014} proposed several modifications of the PC algorithm that remove part or all of this order-dependence. 

\opc~algorithms currently exist for all three chain graph interpretations \citep{javidian2019properties,penea12amp,sp} where the different phases are slightly altered according to the interpretation but the basic ideas are kept the same. The first phase finds the
adjacencies (skeleton), the second orients the edges that must be oriented the same in
every CG in the Markov equivalence class and the third phase transforms this graph into a CG.  Order-independent versions of the \opc~algorithm for LWF CGs and MVR CGs were proposed in \citep{javidian2019properties,javidian2019orderindepMVR}, respectively. In this paper, we proved that the proposed \opc~algorithm for AMP CGs in \citep{penea12amp} is order-dependent. Then, we proposed several modifications of the \opc~algorithm that remove part or all of this order-dependence, but the proposed algorithms do not change the result when perfect conditional independence information is used (see Section \ref{sec:pcalg}).

\subsection{Decomposition Based Learning (LCD-Like) Algorithms  for PGMs}
Structure learning of Bayesian networks via decomposition was proposed in \citep{xie}. This approach starts with finding a decomposition of the entire variable set into subsets, on each of which the
local skeleton is then recovered. In the next phase, the adjacency graph (global skeleton) is reconstructed by merging the decomposed graphs (local skeletons) together. In the last phase, arrowheads are added to some of the edges where they can be inferred in an
efficient manner with lower complexity than the PC algorithm \citep{xie}.

Following the same idea, a decomposition-based algorithm called LCD  (\textbf{L}earn \textbf{C}hain graphs via \textbf{D}ecomposition) was proposed in \citep{mxg,javidian2019properties}  to learn LWF CGs and MVR CGs, respectively; where the different phases are slightly altered according to the interpretation but the basic ideas in \citep{xie} are kept the same.
In this paper, we developed an LCD-like algorithm, called \lcd~, for learning the structure of AMP chain graphs  based on the idea of decomposing the learning problem into a set of smaller scale problems on its decomposed subgraphs. Similarities and differences between \lcd~ and other LCD-like algorithms are discussed in Section \ref{main-alg}.

\section{Conclusion}
This paper addresses two main problems in the context of AMP chain graphs (CGs): finding minimal separators and structure learning.  
We first studied and solved the problem of finding
minimal separating sets for pairs of variables in an AMP CGs. We also studied some extensions of the basic problem that include finding a
minimal separator from a restricted set of nodes, finding a minimal separator for two given disjoint sets, testing whether a given separator is minimal, and listing all minimal separators, given two non-adjacent nodes (or disjoint subsets) $X$ and $Y$. Applications
of this research include: (1) learning chain graphs from data and (2) problems related to the
selection of the variables to be instantiated when using chain graphs for inference tasks, a topic for future work. 

Experimental evaluations in the Gaussian case show that both ({\sc Stable}-) \opc~ and \lcd~algorithms yield good results when the underlying graph is sparse; this holds also in the discrete case, according to experiments with standard benchmark Bayesian networks.  This is important because Bayesian networks are special cases of AMP CGs and we often do not know the information that the true underlying structure is a DAG, which is not usually testable from data. The \lcd~algorithm achieves competitive results with the \opc~and \spc~learning algorithms in both the Gaussian and discrete cases.
In fact, our \lcd~ usually outperforms the \opc~and \spc~algorithms in all five performance metrics i.e., TPR, FPR, TDR, ACC, and SHD.  

The local skeletons of our \lcd~ algorithm and CI tests at each level of the skeleton recovery of the \spc~algorithm can be learned independently from each other, and later merged and reconciled to produce a coherent AMP chain graph. This allows the parallel implementations for scaling up the task of learning AMP chain graphs from data containing more than hundreds of variables, which is crucial for big data analysis tasks. 
The correctness proof of the decomposition-based algorithm (i.e., \lcd) is built upon our results on separating sets. This algorithm exhibits reduced complexity, as measured by run time and number of conditional independence tests, enhances the power of conditional independence tests by reducing the number of separating sets that need to be considered, and, according to  our experimental evaluation, achieves better quality with respect to the learned structure. 

A direction for future work is the design of a hybrid algorithm for learning AMP chain graphs that exploits minimal separators directly, as done in~\citep{ACID2001} for learning Bayesian networks. Another natural continuation of the work presented here would be to develop a learning algorithm with weaker assumptions than the faithfulness assumption. This could for example be a learning
algorithm that only assumes that the probability distribution satisfies the \textit{composition property}. It should be mentioned that \citep{psn} developed an algorithm for learning LWF CGs under the composition property. However, \citep{Addendum} proved that the same technique cannot be used for AMP chain graphs. We believe that our decomposition-based approach is extendable to the structural learning of marginal AMP chain graphs \citep{PENA2016MAMP} and ancestral graphs \citep{rs}. Also, a potential continuation of the work presented here would be to develop a learning algorithm via decomposition for marginal AMP chain graphs and ancestral graphs under the faithfulness assumption.
As we mentioned before, our \lcd~algorithm works better than the ({\sc Stable-}) \opc~in many settings. The reason is that \lcd~algorithm takes advantage of local computations that makes it robust against the choice of learning parameters. In Bayesian networks, the concept that enables us to take advantage of local computation is \textit{Markov blanket}.
Recently, \citep{uai2020LWFCGs} extended the concept of Markov blankets to LWF CGs and proved what variables make up the Markov blanket of a target variable in  an  LWF  CG. Characterizing Markov blankets in AMP CGs and designing a Markov blanket based algorithm for learning AMP CGs nother interesting direction for future work.

\section*{Acknowledgment}
We are grateful to Professor Jose M. Pe\~na and Dr. Dag Sonntag for providing
us with code that helped in the data generating procedure. This work has been partially supported by AFRL and DARPA (FA8750-16-2-0042). This work is also partially supported by an ASPIRE grant from the Office of the Vice President for Research at the University of South Carolina.

\section*{Appendix A. Proofs of Theorems \ref{thm_junctree_ampcg} and \ref{thm_main_ampcg}}

In Theorem \ref{thm3amp}, we showed that if we find a separator over $S$ in $(G_{ant(u\cup v)})^a$ then it is a $p$-separator in $G$. On the other hand, if there exists a $p$-separator over $S$ in $G$ then there must exist a separator over $S$ in $(G_{ant(u\cup v)})^a$ by removing all nodes which are not in $ant(u\cup v)$ from it. This observation  yield the following results.

\begin{lemma}\label{lem1amp}
	Let $u$ and $v$ be two non-adjacent vertices in AMP CG $G$, and let $\rho$ be a chain from $u$ to $v$. If $\rho$ is not contained in
	$ant(u\cup v)$, then $\rho$ is blocked by any subset $S$ of $ant(u\cup v)\setminus\{u,v\}$.
\end{lemma}
\begin{proof}
	Since $\rho \not\subseteq ant(u\cup v)$, there is a sequence from $s$ (may be $u$) to $y$ (may be $v$) in $\rho=(u,\dots,s,t,\dots,x,y,\dots,v)$ such that $s$ and $y$ are contained in $ant(u\cup v)$ and all vertices from $t$ to $x$ are out of $ant(u\cup v)$.Then the edges $s-t$ and $x-y$ must be oriented as $s\to t$ and $x\gets y$, otherwise $t$ or $x$ belongs to $ant(u\cup v)$. Thus there exist at least one triplex between $s$ and $y$ on $\rho$. The middle vertex $w$ of the triplex closest to $s$ between $s$ and $y$ is not contained in
	$ant(u\cup v)$, and any descendant of $w$ is not in $ant(u\cup v)$. So $\rho$ is blocked
	by this triplex, and it cannot be activated conditionally on any vertex in $S$ where $S\subseteq ant(u\cup v)\setminus\{u,v\}$.
\end{proof}

\begin{lemma}\label{lem2amp}
	Let $T$ be a $p$-separation tree for the AMP CG $G$. For any vertex $u$ there exists at least one node of $T$ that contains $u$ and $pa(u)$.
\end{lemma}
\begin{proof}
	If $pa(u)$ is empty, the result is trivial. Otherwise let $C$ denote the node of $T$ which contains $u$ and the most elements
	of $u$'s parent. Since no set can separate $u$ from a parent, there must be a node of $T$ that contains $u$ and the parent. If $u$ has only
	one parent, then we obtain the lemma. If $u$ has two or more parents, we choose two arbitrary elements $v$ and $w$ of $u$'s parent that are not contained in a single node of $T$ but are contained in two different nodes of $T$,
	say $\{u,v\}\subseteq C$ and  $\{u,w\}\subseteq C'$ respectively, since all vertices in $V$ appear in $T$. On the chain from $C$ to $C'$ in $T$, all separators must contain $u$, otherwise they cannot separate $C$ from $C'$. However, any separator containing $u$ cannot
	separate $v$ and $w$ because $v\to u\gets w$ is an active triplex between $v$ and $w$ in $G$. Thus we got a contradiction.
\end{proof}

\begin{lemma}\label{lem3amp}
	Let $T$ be a $p$-separation tree for AMP CG $G$ and $C$ a node of $T$. If $u$ and $v$ are two vertices in $C$ that
	are non-adjacent in $G$ and belong to two different chain components, then there exists a node $C'$ of $T$ containing $u, v$ and a set $S$ such that $S$ $p$-separates $u$ and $v$ in $G$.
\end{lemma}
\begin{proof}
    Assume that $u$ and $v$ are two vertices in $G$ that are non-adjacent and belong to two different chain components.
    Without loss of generality, we can suppose that $v$ is not a descendant of the vertex $u$ in $G$, i.e., $v\not\in nd(u)$. According to the pairwise Markov property for AMP chain graphs in \citep{AMP2001}, $u\perp\!\!\!\perp v|pa(u).$ By Lemma \ref{lem2amp}, there is a
	node $C_1$ of $T$ that contains $u$ and $pa(u)$. If $v\in C_1$, then $S$ defined as the parents of $u$ $p$-separates $u$ from $v$.
	
	If $v\not\in C_1$, choose the node $C_2$ that is the closest node in $T$ to the node $C_1$ and that contains $u$ and $v$. Consider that there is at least one parent $p$ of $u$ that is not contained
	in $C_2$. Thus there is a separator $K$ connecting $C_2$ toward $C_1$ in $T$ such that $K$ $p$-separates $p$ from all vertices in $C_2\setminus K$. Note that on the chain from $C_1$ to $C_2$ in $T$, all
	separators must contain $u$, otherwise they cannot separate $C_1$ from $C_2$. So, we have $u\in K$ but $v\not\in K$ (if $v\in K$, then $C_2$ is not the closest node of $T$ to the node $C_1$). In fact, for every parent $p'$ of $u$ that is contained in $C_1$ but not in $C_2$,  $K$ separates $p'$ from all vertices in $C_2\setminus K$, especially the vertex $v$.
	
	Define $S=[ant(u\cup v)\cap (K\cup\{p\in pa(u)|p\in C_2\})]\setminus \tau_u$, where $\tau_u$ is the chain component that includes $u$.  It is not difficult to see that $S$ is a subset of $C_2$. We need to show that $u$ and $v$ are $p$-separated by $S$, that is, every chain between $u$
	and $v$ in $G$, say $\rho$, is blocked by $S$.
	
	If $\rho$ is not contained in $ant(u\cup v)$, then we obtain from Lemma \ref{lem1amp} that $\rho$ is blocked by $S$.
	
	When $\rho$ is contained in $ant(u\cup v)$, let $x$ be adjacent to $u$ on $\rho$, that is, $\rho =(u, x, y, \dots , v)$. We consider the three possible orientations of the edge between $u$ and $x$.  We now show that $\rho$ is blocked in all three cases by $S$.
	
	\begin{itemize}
		\item[i:] $u\gets x$, so it is obvious that $x$ is not a triplex node and we have two possible sub-cases:
			\begin{enumerate}
				\item $x\in C_2$. In this case the chain $\rho$ is blocked at $x$.
				\item $x\not\in C_2$. In this case $K$ $p$-separates $x$ from $v$. Theorem \ref{thm3amp} guarantees that the set $S'=K\cap ant(x\cup v)$ also $p$-separates $x$ from $v$. Note that $S'\cap \tau_u=\emptyset$ to prevent a partially directed cycle, and $S'\subseteq S$. So, $S$ $p$-separates $x$ from $v$ i.e., the chain between $v$ and $x$ is blocked by $S$.  Hence the chain $\rho$ is blocked by $S$.
			\end{enumerate}
		\item[ii:] $u\to x$. We have the following sub-cases:
			\begin{enumerate}
				\item $x\in ant(u)$. This case is impossible because a partially directed cycle would occur.
				\item $x\in an(v)$. This case is impossible because $v$ cannot be a descendant of $u$.
			\end{enumerate}
		\item[iii:] $u\erelbar{00} x$, so $x\in \tau_u$. In this case the chain $\rho$ between $u$ and $v$ has a triplex node at $y\in\tau_u$  that is not in $S$. So, the chain $\rho$ is blocked at $y$ and cannot be activated by $S$. 
		\end{itemize}
\end{proof}

\begin{proof}[Proof of Theorem \ref{thm_junctree_ampcg}]
	From \citep{cdls}, we know that any separator $S$ in junction tree $T$ separates $V_1\setminus S$ and $V_2\setminus S$ in the triangulated graph $\bar{G}_V^t$, where $V_i$ denotes the variable set of the subtree $T_i$ induced by removing the edge with
	a separator $S$ attached, for $i = 1, 2$. Since the edge set of $\bar{G}_V^t$ contains that of undirected independence graph $\bar{G}_V$ for $G$, $V_1\setminus S$ and $V_2\setminus S$ are also separated in $\bar{G}_V$. Since $\bar{G}_V$ is an undirected independence graph for $G$, using  the definition of $p$-separation tree we obtain that $T$ is a $p$-separation tree for $G$.
\end{proof}

\begin{proof}[Proof of Theorem \ref{thm_main_ampcg}] 
	\noindent ($\Rightarrow$) If condition (i) is the case, nothing remains to prove. Otherwise, Lemma \ref{lem3amp} implies condition (ii).
	
	\noindent ($\Leftarrow$) Assume that $u$ and $v$ are not contained together in any chain component and any node $C$ of $T$. Also, assume that $C_1$ and $C_2$ are two nodes of $T$ that contain $u$ and $v$, respectively. Consider that $C_1'$ is the most distant node from $C_1$, between $C_1$ and $C_2$, that contains $u$ and $C_2'$ is the most distant node from $C_2$, between $C_1$ and $C_2$, that contains $v$. Note that it is possible that $C_1'=C_1$ or $C_2'=C_2$. By the condition (i) we know that $C_1'\ne C_2'$. The sufficiency of condition (i) is given by the definition of the $p$-separation tree, because any separator between $C_1'$ and $C_2'$ $p$-separates $u$ from $v$.
	
	The sufficiency of conditions (ii) is trivial by the definition of $p$-separation.
\end{proof}

The following example shows that Theorem \ref{thm_main_ampcg} cannot be strengthened.
\begin{figure}[!ht]
\centering
\captionsetup[subfigure]{font=footnotesize}
\centering
\subcaptionbox{}[.5\textwidth]{%
\begin{tikzpicture}[transform shape]
	\tikzset{vertex/.style = {shape=circle,inner sep=0pt,
  text width=5mm,align=center,
  draw=black,
  fill=white}}
\tikzset{edge/.style = {->,> = latex',thick}}
	\node[vertex,thick] (f) at  (1,1) {$f$};
	\node[vertex,thick] (g) at  (3,1) {$g$};
	\node[vertex,thick] (e) at  (1,3) {$e$};
	\node[vertex,thick] (h) at  (3,3) {$h$};
	\node[vertex,thick] (b) at  (0,0) {$b$};
	\node[vertex,thick] (a) at  (0,4) {$a$};
	\node[vertex,thick] (c) at  (4,0) {$c$};
	\node[vertex,thick] (d) at  (4,4) {$d$};
	\draw[edge] (a) to (e);
	\draw[edge] (d) to (h);
	\draw[edge] (b) to (f);
	\draw[edge] (c) to (g);
	\draw[thick] (e) to (h);
	\draw[thick] (e) to (f);
	\draw[thick] (f) to (g);
	\draw[thick] (g) to (h);
\end{tikzpicture}}%
\subcaptionbox{}[.5\textwidth]{\begin{tikzpicture}[transform shape]
	\tikzset{vertex/.style = {shape=circle,inner sep=0pt,
  text width=5mm,align=center,
  draw=black,
  fill=white}}
\tikzset{edge/.style = {->,> = latex',thick}}
	\node[vertex,thick] (f) at  (1,1) {$f$};
	\node[vertex,thick] (g) at  (3,1) {$g$};
	\node[vertex,thick] (e) at  (1,3) {$e$};
	\node[vertex,thick] (h) at  (3,3) {$h$};
	\node[vertex,thick] (b) at  (0,0) {$b$};
	\node[vertex,thick] (a) at  (0,4) {$a$};
	\node[vertex,thick] (c) at  (4,0) {$c$};
	\node[vertex,thick] (d) at  (4,4) {$d$};
	\draw[thick] (a) to (e);
	\draw[thick] (d) to (h);
	\draw[thick] (b) to (f);
	\draw[thick] (c) to (g);
	\draw[thick] (e) to (h);
	\draw[thick] (e) to (f);
	\draw[thick] (f) to (g);
	\draw[thick] (g) to (h);
	\draw[thick,red] (a) to (d);
	\draw[thick,red] (a) to (h);
	\draw[thick,red] (e) to (d);
	\draw[thick,red] (a) to (b);
	\draw[thick,red] (a) to (f);
	\draw[thick,red] (b) to (e);
	\draw[thick,red] (b) to (c);
	\draw[thick,red] (b) to (g);
	\draw[thick,red] (f) to (c);
	\draw[thick,red] (c) to (d);
	\draw[thick,red] (c) to (h);
	\draw[thick,red] (g) to (d);
\end{tikzpicture}}

\subcaptionbox{}[.5\textwidth]{\begin{tikzpicture}[transform shape]
	\tikzset{vertex/.style = {shape=circle,inner sep=0pt,
  text width=5mm,align=center,
  draw=black,
  fill=white}}
\tikzset{edge/.style = {->,> = latex',thick}}
	\node[vertex,thick] (f) at  (1,1) {$f$};
	\node[vertex,thick] (g) at  (3,1) {$g$};
	\node[vertex,thick] (e) at  (1,3) {$e$};
	\node[vertex,thick] (h) at  (3,3) {$h$};
	\node[vertex,thick] (b) at  (0,0) {$b$};
	\node[vertex,thick] (a) at  (0,4) {$a$};
	\node[vertex,thick] (c) at  (4,0) {$c$};
	\node[vertex,thick] (d) at  (4,4) {$d$};
	\draw[thick] (a) to (e);
	\draw[thick] (d) to (h);
	\draw[thick] (b) to (f);
	\draw[thick] (c) to (g);
	\draw[thick] (e) to (h);
	\draw[thick] (e) to (f);
	\draw[thick] (f) to (g);
	\draw[thick] (g) to (h);
	\draw[thick,red] (a) to (d);
	\draw[thick,red] (a) to (h);
	\draw[thick,red] (e) to (d);
	\draw[thick,red] (a) to (b);
	\draw[thick,red] (a) to (f);
	\draw[thick,red] (b) to (e);
	\draw[thick,red] (b) to (c);
	\draw[thick,red] (b) to (g);
	\draw[thick,red] (f) to (c);
	\draw[thick,red] (c) to (d);
	\draw[thick,red] (c) to (h);
	\draw[thick,red] (g) to (d);
	\draw[thick,blue] (f) to (h);
	\draw[thick,blue] (f) to [out=35,in=245] (d);
	\draw[thick,blue] (b) to [out=65,in=205] (h);
	\draw[thick,blue] (b) to [out=120,in=150] (d);
\end{tikzpicture}}%
\subcaptionbox{}[.5\textwidth]{\begin{tikzpicture}[auto,node distance=1cm]
    \tikzset{edge/.style = {->,> = latex',thick}}
    \node[entity,thick] (node1) {$a,b,d,e,f,h$}
    [grow=up]
    child[grow=up,thick]  {node[attribute,thick] (ch1) {$b,d,f,h$}};
    \node[entity,thick] (node2) [above = of ch1]	{$b,c,d,f,g,h$};
    \path[thick] (ch1) edge (node2);
    \end{tikzpicture}}
        \caption{(a) AMP CG $G$, (b) augmented graph $G^a$, (c) triangulated graph $(G^a)^t$, and (d) \textit{p}-separation tree $T$.}
    \label{fig:counterexampleAMPCGs}
\end{figure}
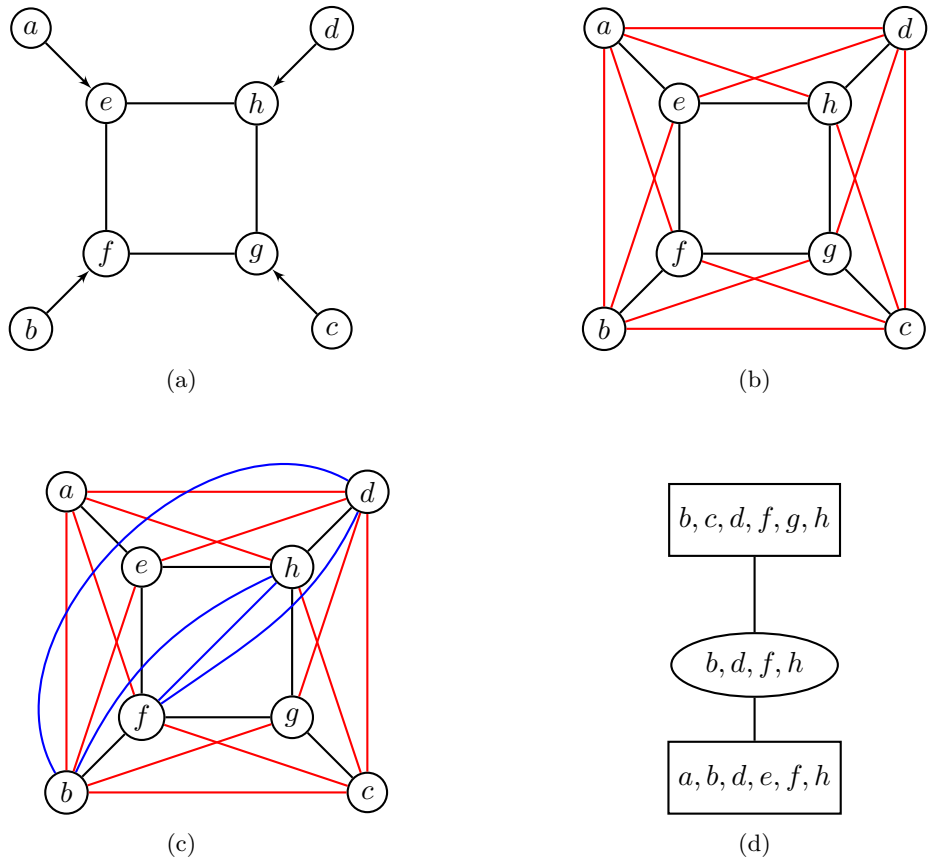

\begin{example}
Consider the AMP CG $G$ in Figure \ref{fig:counterexampleAMPCGs}(a). Vertices $f$ and $h$ are not adjacent but both of them belong to the same chain component. As one can see in the Figure \ref{fig:counterexampleAMPCGs}(d), vertices $f$ and $h$ belong to nodes tree $C_1=\{b,c,d,f,g,h\}$ and $C_2=\{a,b,d,e,f,h\}$. However, none of them contains a subset of $V_G$ that $p$-separates $f$ from $h$.
\end{example}

\begin{proof}
    [Correctness of Algorithm \ref{DBalgampcg}] By the definition of \textit{p}-separation trees and Theorem \ref{thm_main_ampcg}, the initializations at local and global skeleton recovery phases guarantee that no edge is created between any two variables which are not in the same node of the
	$p$-separation tree. Also, deleting edges at local and global skeleton recovery phases guarantees that any other edge
	between two $p$-separated variables can be deleted in some local skeleton or in the removal procedure at the global skeleton recovery phase. Thus the global skeleton obtained after line 22 is
	correct. Note that, in an AMP CG, every missing edge corresponds to at least one independency in the corresponding
	independence model. Therefore, each augmented edge $u\erelbar{00} v$ in the undirected independence graph must be deleted at some subgraph over a node of the $p$-separation tree or at some point of the removal procedure of the global skeleton recovery. The proof of the correctness of orientation rules R1-R4  can be found in \citep{penea12amp}.
\end{proof}
\newpage
\bibliography{bibliography}

\begin{thebibliography}{83}
\providecommand{\natexlab}[1]{#1}
\providecommand{\url}[1]{\texttt{#1}}
\expandafter\ifx\csname urlstyle\endcsname\relax
  \providecommand{\doi}[1]{doi: #1}\else
  \providecommand{\doi}{doi: \begingroup \urlstyle{rm}\Url}\fi

\bibitem[Abramson et~al.(1996)Abramson, Brown, Edwards, Murphy, and
  Winkler]{Hailfinder}
Bruce Abramson, John Brown, Ward Edwards, Allan Murphy, and Robert~L. Winkler.
\newblock Hailfinder: A {B}ayesian system for forecasting severe weather.
\newblock \emph{International Journal of Forecasting}, 12\penalty0
  (1):\penalty0 57 -- 71, 1996.
\newblock Probability Judgmental Forecasting.

\bibitem[Acid and de~Campos(1996)]{ad}
Silvia Acid and Luis~M. de~Campos.
\newblock An algorithm for finding minimum d-separating sets in belief
  networks.
\newblock \emph{Proceedings of the Twelfth international conference on
  Uncertainty in artificial intelligence}, pages 3--10, 1996.

\bibitem[Acid and de~Campos(2001)]{ACID2001}
Silvia Acid and Luis~M. de~Campos.
\newblock A hybrid methodology for learning belief networks: Benedict.
\newblock \emph{International Journal of Approximate Reasoning}, 27\penalty0
  (3):\penalty0 235--262, 2001.

\bibitem[Anandkumar et~al.(2012)Anandkumar, Tan, Huang, and
  Willsky]{anandkumar2012}
Animashree Anandkumar, Vincent Y.~F. Tan, Furong Huang, and Alan~S. Willsky.
\newblock High-dimensional structure estimation in {I}sing models: Local
  separation criterion.
\newblock \emph{Ann. Statist.}, 40\penalty0 (3):\penalty0 1346--1375, 06 2012.
\newblock \doi{10.1214/12-AOS1009}.
\newblock URL \url{https://doi.org/10.1214/12-AOS1009}.

\bibitem[Andersson and Perlman(2006)]{andersson2006}
Steen~A. Andersson and Michael~D. Perlman.
\newblock Characterizing {M}arkov equivalence classes for {AMP} chain graph
  models.
\newblock \emph{The Annals of Statistics}, 34\penalty0 (2):\penalty0 939--972,
  04 2006.

\bibitem[Andersson et~al.(1996)Andersson, Madigan, and Perlman]{amp}
Steen~A. Andersson, David Madigan, and Michael~D. Perlman.
\newblock An alternative {M}arkov property for chain graphs.
\newblock In Eric Horvitz and Finn~V. Jensen, editors, \emph{Proceedings of the
  Twelfth Conference on Uncertainty in artificial intelligence}, pages 40--48,
  1996.

\bibitem[Andersson et~al.(2001)Andersson, Madigan, and Perlman]{AMP2001}
Steen~A. Andersson, David Madigan, and Michael~D. Perlman.
\newblock Alternative {M}arkov properties for chain graphs.
\newblock \emph{Scandinavian Journal of Statistics}, 28\penalty0 (1):\penalty0
  33--85, 2001.

\bibitem[Andrews et~al.(2018)Andrews, Ramsey, and Cooper]{andrews2018scoring}
Bryan Andrews, Joseph Ramsey, and Gregory~F Cooper.
\newblock Scoring {B}ayesian networks of mixed variables.
\newblock \emph{International journal of data science and analytics},
  6\penalty0 (1):\penalty0 3--18, 2018.

\bibitem[Banerjee et~al.(2008)Banerjee, Ghaoui, and d'Aspremont]{Banerjee2007}
Onureena Banerjee, Laurent~El Ghaoui, and Alexandre d'Aspremont.
\newblock Model selection through sparse maximum likelihood estimation for
  multivariate {G}aussian or binary data.
\newblock \emph{Journal of Machine Learning Research}, 9:\penalty0 485--516,
  2008.

\bibitem[Beinlich et~al.(1989)Beinlich, Suermondt, Chavez, and Cooper]{alarm}
Ingo~A. Beinlich, H.~J. Suermondt, R.~Martin Chavez, and Gregory~F. Cooper.
\newblock The alarm monitoring system: A case study with two probabilistic
  inference techniques for belief networks.
\newblock In Jim Hunter, John Cookson, and Jeremy Wyatt, editors, \emph{AIME
  89}, pages 247--256, Berlin, Heidelberg, 1989. Springer Berlin Heidelberg.

\bibitem[Berry et~al.(2004)Berry, Blair, Heggernes, and Peyton]{bbhp}
Anne Berry, Jean Blair, Pinar Heggernes, and Barry Peyton.
\newblock Maximum cardinality search for computing minimal triangulations of
  graphs.
\newblock \emph{Algorithmica}, 39:\penalty0 287--298, 2004.

\bibitem[Binder et~al.(1997)Binder, Koller, Russell, and Kanazawa]{insurance}
John Binder, Daphne Koller, Stuart Russell, and Keiji Kanazawa.
\newblock Adaptive probabilistic networks with hidden variables.
\newblock \emph{Machine Learning}, 29\penalty0 (2):\penalty0 213--244, Nov
  1997.

\bibitem[Bresler et~al.(2008)Bresler, Mossel, and Sly]{Bresler2008}
Guy Bresler, Elchanan Mossel, and Allan Sly.
\newblock Reconstruction of {M}arkov random fields from samples: Some
  observations and algorithms.
\newblock In Ashish Goel, Klaus Jansen, Jos{\'e} D.~P. Rolim, and Ronitt
  Rubinfeld, editors, \emph{Approximation, Randomization and Combinatorial
  Optimization. Algorithms and Techniques}, pages 343--356, Berlin, Heidelberg,
  2008. Springer Berlin Heidelberg.
\newblock ISBN 978-3-540-85363-3.

\bibitem[Bromberg et~al.(2009)Bromberg, Margaritis, and Honavar]{Bromberg2009}
Facundo Bromberg, Dimitris Margaritis, and Vasant Honavar.
\newblock Efficient {M}arkov network structure discovery using independence
  tests.
\newblock \emph{J. Artif. Int. Res.}, 35\penalty0 (1):\penalty0 449–484, July
  2009.
\newblock ISSN 1076-9757.

\bibitem[{Chow} and {Liu}(1968)]{ChowLiu}
C.~{Chow} and C.~{Liu}.
\newblock Approximating discrete probability distributions with dependence
  trees.
\newblock \emph{IEEE Transactions on Information Theory}, 14\penalty0
  (3):\penalty0 462--467, May 1968.
\newblock ISSN 1557-9654.
\newblock \doi{10.1109/TIT.1968.1054142}.

\bibitem[Colombo and Maathuis(2014)]{Colombo2014}
Diego Colombo and Marloes~H. Maathuis.
\newblock Order-independent constraint-based causal structure learning.
\newblock \emph{The Journal of Machine Learning Research}, 15\penalty0
  (1):\penalty0 3741--3782, 2014.

\bibitem[Cooper(1990)]{COOPER1990}
Gregory~F. Cooper.
\newblock The computational complexity of probabilistic inference using
  {B}ayesian belief networks.
\newblock \emph{Artificial Intelligence}, 42\penalty0 (2):\penalty0 393 -- 405,
  1990.
\newblock ISSN 0004-3702.
\newblock \doi{https://doi.org/10.1016/0004-3702(90)90060-D}.
\newblock URL
  \url{http://www.sciencedirect.com/science/article/pii/000437029090060D}.

\bibitem[Cormen et~al.(2009)Cormen, Leiserson, Rivest, and Stein]{CLRS3rd}
Thomas~H. Cormen, Charles~E. Leiserson, Ronald~L. Rivest, and Clifford Stein.
\newblock \emph{Introduction to Algorithms, Third Edition}.
\newblock The MIT Press, 3rd edition, 2009.
\newblock ISBN 0262033844, 9780262033848.

\bibitem[Cowell et~al.(1999)Cowell, Dawid, Lauritzen, and Spiegelhalter]{cdls}
R.~Cowell, A.~P. Dawid, S.~Lauritzen, and D.~J. Spiegelhalter.
\newblock \emph{Probabilistic networks and expert systems. Statistics for
  Engineering and Information Science}.
\newblock Springer-Verlag, 1999.

\bibitem[Cox and Wermuth(1993)]{cw1}
D.~R. Cox and Nanny Wermuth.
\newblock Linear dependencies represented by chain graphs.
\newblock \emph{Statistical Science}, 8\penalty0 (3):\penalty0 204--218, 1993.

\bibitem[Cox and Wermuth(1996)]{cw2}
D.~R. Cox and Nanny Wermuth.
\newblock \emph{Multivariate Dependencies-Models, Analysis and Interpretation}.
\newblock Chapman and Hall, 1996.

\bibitem[de~Abreu et~al.(2010{\natexlab{a}})de~Abreu, Labouriau, and
  Edwards]{Edwards2010}
Gabriel de~Abreu, Rodrigo Labouriau, and David Edwards.
\newblock High-dimensional graphical model search with the graphd {R} package.
\newblock \emph{Journal of Statistical Software, Articles}, 37\penalty0
  (1):\penalty0 1--18, 2010{\natexlab{a}}.
\newblock ISSN 1548-7660.
\newblock \doi{10.18637/jss.v037.i01}.
\newblock URL \url{https://www.jstatsoft.org/v037/i01}.

\bibitem[de~Abreu et~al.(2010{\natexlab{b}})de~Abreu, Labouriau, and
  Edwards]{deAbreu2010}
Gabriel de~Abreu, Rodrigo Labouriau, and David Edwards.
\newblock Selecting high-dimensional mixed graphical models using minimal {AIC}
  or {BIC} forests.
\newblock \emph{BMC Bioinformatics}, 11\penalty0 (18), 2010{\natexlab{b}}.
\newblock ISSN 1548-7660.
\newblock \doi{10.18637/jss.v037.i01}.
\newblock URL \url{https://www.jstatsoft.org/v037/i01}.

\bibitem[Drton(2009)]{d}
Mathias Drton.
\newblock Discrete chain graph models.
\newblock \emph{Bernoulli}, 15\penalty0 (3):\penalty0 736--753, 2009.

\bibitem[Edelkamp and Schroedl(2011)]{Edelkamp2011}
Stefan Edelkamp and Stefan Schroedl.
\newblock \emph{Heuristic Search: Theory and Applications}.
\newblock Morgan Kaufmann Publishers Inc., San Francisco, CA, USA, 2011.

\bibitem[Edwards(2000)]{ed}
D.~Edwards.
\newblock \emph{Introduction to Graphical Modelling. 2nd Ed.}
\newblock Springer-Verlag, New York, 2000.

\bibitem[Fenton and Neil(2018)]{FentonNeil}
Norman Fenton and Martin Neil.
\newblock \emph{Risk Assessment and Decision Analysis with {B}ayesian
  Networks}.
\newblock Chapman and Hall/CRC, New York, 2nd edition, 2018.

\bibitem[Frydenberg(1990)]{f}
Morten Frydenberg.
\newblock The chain graph {M}arkov property.
\newblock \emph{Scandinavian Journal of Statistics}, 17\penalty0 (4):\penalty0
  333--353, 1990.

\bibitem[H{\o}jsgaard et~al.(2012)H{\o}jsgaard, Edwards, and Lauritzen]{hel}
S.~H{\o}jsgaard, D.~Edwards, and S.~Lauritzen.
\newblock \emph{Graphical Models with R}.
\newblock Springer, 2012.

\bibitem[Javidian and Valtorta(2018{\natexlab{a}})]{jv-pgm18}
M.~A. Javidian and M.~Valtorta.
\newblock Finding minimal separators in {LWF} chain graphs.
\newblock In \emph{The 9th International Conference on Probabilistic Graphical
  Models (PGM 2018)}, pages 193--200, 2018{\natexlab{a}}.

\bibitem[Javidian and Valtorta(2018{\natexlab{b}})]{jv-uai18}
M.~A. Javidian and M.~Valtorta.
\newblock Finding minimal separators in ancestral graphs.
\newblock In \emph{Seventh Causal Inference Workshop at the 34th Conference on
  Artifical Intelligence (UAI-18)}, 2018{\natexlab{b}}.

\bibitem[Javidian and Valtorta(2019)]{jv-mvr19}
M.~A. Javidian and M.~Valtorta.
\newblock A decomposition-based algorithm for learning the structure of {MVR}
  chain graphs.
\newblock \url{https://arxiv.org/abs/1806.00882}, 2019.

\bibitem[Javidian et~al.(2020)Javidian, Valtorta, and Jamshidi]{uai2020LWFCGs}
M.~A. Javidian, M.~Valtorta, and P.~Jamshidi.
\newblock Learning {LWF} chain graphs: {A} {M}arkov blanket discovery approach.
\newblock In \emph{Proceedings of the Thirty-Six Conference on Uncertainty in
  Artificial Intelligence}, UAI'20. AUAI Press, 2020.

\bibitem[Javidian(2019)]{javidian2019properties}
Mohammad~Ali Javidian.
\newblock \emph{Properties, Learning Algorithms, and Applications of Chain
  Graphs and {B}ayesian Hypergraphs}.
\newblock PhD thesis, University of South Carolina, 2019.

\bibitem[Javidian et~al.(2019)Javidian, Valtorta, and
  Jamshidi]{javidian2019orderindepMVR}
Mohammad~Ali Javidian, Marco Valtorta, and Pooyan Jamshidi.
\newblock Order-independent structure learning of multivariate regression chain
  graphs.
\newblock In \emph{International Conference on Scalable Uncertainty
  Management}, pages 324--338. Springer, 2019.

\bibitem[Jensen and Nielsen(2007)]{Jensen2007}
Finn~V. Jensen and Thomas~D. Nielsen.
\newblock \emph{{B}ayesian Networks and Decision Graphs}.
\newblock Springer, 2nd edition, 2007.

\bibitem[Kalisch and B\"{u}hlmann(2007)]{Kalisch07}
Markus Kalisch and Peter B\"{u}hlmann.
\newblock Estimating high-dimensional directed acyclic graphs with the
  {PC}-algorithm.
\newblock \emph{J. Mach. Learn. Res.}, 8:\penalty0 613--636, 2007.

\bibitem[Koller and Friedman(2009)]{kf}
Daphne Koller and Nir Friedman.
\newblock \emph{Probabilistic Graphical Models: Principles and Techniques}.
\newblock The MIT Press, 2009.

\bibitem[Lauritzen(1996)]{l}
S.~Lauritzen.
\newblock \emph{Graphical Models}.
\newblock Oxford Science Publications, 1996.

\bibitem[Lauritzen and Wermuth(1989)]{lw}
S.~Lauritzen and N.~Wermuth.
\newblock Graphical models for associations between variables, some of which
  are qualitative and some quantitative.
\newblock \emph{The Annals of Statistics}, 17\penalty0 (1):\penalty0 31--57,
  1989.

\bibitem[Lauritzen and Spiegelhalter(1988)]{asia}
S.~L. Lauritzen and D.~J. Spiegelhalter.
\newblock Local computations with probabilities on graphical structures and
  their application to expert systems.
\newblock \emph{Journal of the Royal Statistical Society. Series B
  (Methodological)}, 50\penalty0 (2):\penalty0 157--224, 1988.

\bibitem[Lauritzen and Jensen(2001)]{lauritzen2001stable}
Steffen~L Lauritzen and Frank Jensen.
\newblock Stable local computation with conditional {G}aussian distributions.
\newblock \emph{Statistics and Computing}, 11\penalty0 (2):\penalty0 191--203,
  2001.

\bibitem[Levitz et~al.(2001)Levitz, Perlman, and Madigan]{LPM2001}
Michael Levitz, Michael~D. Perlman, and David Madigan.
\newblock Separation and completeness properties for {AMP} chain graph {M}arkov
  models.
\newblock \emph{The Annals of Statistics}, 29\penalty0 (6):\penalty0
  1751--1784, 2001.

\bibitem[Ma et~al.(2008)Ma, Xie, and Geng]{mxg}
Z.~Ma, X.~Xie, and Z.~Geng.
\newblock Structural learning of chain graphs via decomposition.
\newblock \emph{Journal of Machine Learning Research}, 9:\penalty0 2847--2880,
  2008.

\bibitem[Motzek and M\"{o}ller(2017)]{Motzek2017}
Alexander Motzek and Ralf M\"{o}ller.
\newblock Indirect causes in dynamic {B}ayesian networks revisited.
\newblock \emph{J. Artif. Int. Res.}, 59\penalty0 (1):\penalty0 1–58, May
  2017.
\newblock ISSN 1076-9757.

\bibitem[Nagarajan et~al.(2013)Nagarajan, Scutari, and L\`ebre]{Nagarajan2013}
Radhakrishnan Nagarajan, Marco Scutari, and Sophie L\`ebre.
\newblock \emph{{B}ayesian Networks in R: With Applications in Systems
  Biology}.
\newblock Springer, 2013.

\bibitem[Neapolitan and Jiang(2018)]{Neapolitan18}
Richard~E. Neapolitan and Xia Jiang.
\newblock \emph{Artificial Intelligence: With an Introduction to Machine
  Learning}.
\newblock Chapman and Hall, 2nd edition, 2018.

\bibitem[{Netrapalli} et~al.(2010){Netrapalli}, {Banerjee}, {Sanghavi}, and
  {Shakkottai}]{Netrapalli2010}
P.~{Netrapalli}, S.~{Banerjee}, S.~{Sanghavi}, and S.~{Shakkottai}.
\newblock Greedy learning of {M}arkov network structure.
\newblock In \emph{2010 48th Annual Allerton Conference on Communication,
  Control, and Computing (Allerton)}, pages 1295--1302, Sep. 2010.

\bibitem[Pe\~{n}a(2011)]{Pena2011}
Jose~M. Pe\~{n}a.
\newblock Finding consensus {B}ayesian network structures.
\newblock \emph{J. Artif. Int. Res.}, 42\penalty0 (1):\penalty0 661–687,
  September 2011.
\newblock ISSN 1076-9757.

\bibitem[Pearl(1988)]{pearl88}
J.~Pearl.
\newblock \emph{Probabilistic Reasoning in Intelligent Systems: Networks of
  Plausible Inference}.
\newblock Morgan Kaufmann Publishers Inc. San Francisco, CA, USA, 1988.

\bibitem[Pearl(1984)]{Pearl1984}
Judea Pearl.
\newblock \emph{Heuristics: Intelligent Search Strategies for Computer Problem
  Solving}.
\newblock Addison-Wesley Longman Publishing Co., Inc., Boston, MA, USA, 1984.
\newblock ISBN 0-201-05594-5.

\bibitem[Pe{\~n}a(2014{\natexlab{a}})]{Addendum}
J.~M. Pe{\~n}a.
\newblock Learning multivariate regression chain graphs under faithfulness:
  Addendum.
\newblock Available at the author's website, 2014{\natexlab{a}}.

\bibitem[Pe{\~n}a(2015)]{p2}
J.~M. Pe{\~n}a.
\newblock Every {LWF} and {AMP} chain graph originates from a set of causal
  models.
\newblock \emph{Symbolic and quantitative approaches to reasoning with
  uncertainty, Lecture Notes in Comput. Sci., 9161, Lecture Notes in Artificial
  Intelligence, Springer, Cham}, pages 325--334, 2015.

\bibitem[Pe{\~n}a(2018{\natexlab{a}})]{p3}
J.~M. Pe{\~n}a.
\newblock Reasoning with alternative acyclic directed mixed graphs.
\newblock \emph{Behaviormetrika}, pages 1--34, 2018{\natexlab{a}}.

\bibitem[Pe{\~n}a(2018{\natexlab{b}})]{pena2018uai}
J.~M. Pe{\~n}a.
\newblock Identification of strong edges in {AMP} chain graphs.
\newblock In Amir Globerson and Ricardo Silva, editors, \emph{Proceedings of
  the 34th Conference on Uncertainty in artificial intelligence}, pages 33--42,
  2018{\natexlab{b}}.

\bibitem[Pe{\~n}a et~al.(2014)Pe{\~n}a, Sonntag, and Nielsen]{psn}
J.~M. Pe{\~n}a, D.~Sonntag, and J.~Nielsen.
\newblock An inclusion optimal algorithm for chain graph structure learning.
\newblock \emph{In Proceedings of the 17th International Conference on
  Artificial Intelligence and Statistics}, pages 778--786, 2014.

\bibitem[Pe{\~{n}}a(2008)]{Pena08Mb}
Jose~M. Pe{\~{n}}a.
\newblock Learning {G}aussian graphical models of gene networks with false
  discovery rate control.
\newblock In Elena Marchiori and Jason~H. Moore, editors, \emph{Evolutionary
  Computation, Machine Learning and Data Mining in Bioinformatics}, pages
  165--176, 2008.

\bibitem[Pe{\~n}a(2012)]{penea12amp}
Jose~M. Pe{\~n}a.
\newblock Learning {AMP} chain graphs under faithfulness.
\newblock In \emph{Proceedings of the Sixth European Workshop on Probabilistic
  Graphical Models}, pages 251--258, 2012.

\bibitem[Pe{\~n}a(2014{\natexlab{b}})]{PENA20141185}
Jose~M. Pe{\~n}a.
\newblock Marginal {AMP} chain graphs.
\newblock \emph{International Journal of Approximate Reasoning}, 55\penalty0
  (5):\penalty0 1185--1206, 2014{\natexlab{b}}.

\bibitem[Pe{\~{n}}a(2016)]{Pena2016}
Jose~M. Pe{\~{n}}a.
\newblock Alternative {M}arkov and causal properties for acyclic directed mixed
  graphs.
\newblock In \emph{Proceedings of the Thirty-Second Conference on Uncertainty
  in Artificial Intelligence}, UAI'16, pages 577--586, Arlington, Virginia,
  United States, 2016. AUAI Press.

\bibitem[Pe{\~n}a and G\'omez-Olmedo(2016)]{PENA2016MAMP}
Jose~M. Pe{\~n}a and Manuel G\'omez-Olmedo.
\newblock Learning marginal {AMP} chain graphs under faithfulness revisited.
\newblock \emph{International Journal of Approximate Reasoning}, 68:\penalty0
  108 -- 126, 2016.

\bibitem[Raghu et~al.(2018)Raghu, Ramsey, Morris, Manatakis, Sprites,
  Chrysanthis, Glymour, and Benos]{raghu2018comparison}
Vineet~K Raghu, Joseph~D Ramsey, Alison Morris, Dimitrios~V Manatakis, Peter
  Sprites, Panos~K Chrysanthis, Clark Glymour, and Panayiotis~V Benos.
\newblock Comparison of strategies for scalable causal discovery of latent
  variable models from mixed data.
\newblock \emph{International journal of data science and analytics},
  6\penalty0 (1):\penalty0 33--45, 2018.

\bibitem[Ramsey et~al.(2006)Ramsey, Spirtes, and Zhang]{Ramsey:2006}
Joseph Ramsey, Peter Spirtes, and Jiji Zhang.
\newblock Adjacency-faithfulness and conservative causal inference.
\newblock In \emph{Proceedings of the Twenty-Second Conference on Uncertainty
  in Artificial Intelligence}, UAI'06, pages 401--408, Arlington, Virginia,
  United States, 2006. AUAI Press.

\bibitem[Ravikumar et~al.(2010)Ravikumar, Wainwright, and
  Lafferty]{ravikumar2010}
Pradeep Ravikumar, Martin~J. Wainwright, and John~D. Lafferty.
\newblock High-dimensional {I}sing model selection using l1 -regularized
  logistic regression.
\newblock \emph{Ann. Statist.}, 38\penalty0 (3):\penalty0 1287--1319, 06 2010.
\newblock \doi{10.1214/09-AOS691}.
\newblock URL \url{https://doi.org/10.1214/09-AOS691}.

\bibitem[Richardson(1998)]{r1}
T.~S. Richardson.
\newblock Chain graphs and symmetric associations.
\newblock \emph{In: Jordan M.I. (eds) Learning in Graphical Models. NATO ASI
  Series (Series D: Behavioural and Social Sciences), vol 89}, pages 229--259,
  1998.

\bibitem[Richardson and Spirtes(2002)]{rs}
T.~S. Richardson and P.~Spirtes.
\newblock Ancestral graph {M}arkov models.
\newblock \emph{The Annals of Statistics}, 30\penalty0 (4):\penalty0 962--1030,
  2002.

\bibitem[Roverato and Rocca(2006)]{roverato06}
A.~Roverato and L.~La Rocca.
\newblock On block ordering of variables in graphical modelling.
\newblock \emph{Scandinavian Journal of Statistics}, 33\penalty0 (1):\penalty0
  65--81, 2006.

\bibitem[Roverato(2005)]{roverato05}
Alberto Roverato.
\newblock A unified approach to the characterization of equivalence classes of
  {DAG}s, chain graphs with no flags and chain graphs.
\newblock \emph{Scandinavian Journal of Statistics}, 32\penalty0 (2):\penalty0
  295--312, 2005.

\bibitem[Scutari(2017)]{Scutari17}
Marco Scutari.
\newblock {B}ayesian network constraint-based structure learning algorithms:
  Parallel and optimized implementations in the bnlearn {R} package.
\newblock \emph{Journal of Statistical Software, Articles}, 77\penalty0
  (2):\penalty0 1--20, 2017.
\newblock ISSN 1548-7660.
\newblock \doi{10.18637/jss.v077.i02}.
\newblock URL \url{https://www.jstatsoft.org/v077/i02}.

\bibitem[Scutari and Denis(2015)]{Scutari15}
Marco Scutari and Jean-Baptiste Denis.
\newblock \emph{{B}ayesian Networks with Examples in {R}}.
\newblock Chapman and Hall, 2015.

\bibitem[Shen and Liang(1997)]{shenliang}
Hong Shen and Weifa Liang.
\newblock Efficient enumeration of all minimal separators in a graph.
\newblock \emph{Theoretical Computer Science}, 180:\penalty0 169--180, 1997.

\bibitem[Sonntag(2016)]{s2}
D.~Sonntag.
\newblock \emph{Chain Graphs: Interpretations, Expressiveness and Learning
  Algorithms}.
\newblock PhD thesis, Link{\"o}ping University, 2016.

\bibitem[Sonntag and Pe{\~n}a(2012)]{sp}
D.~Sonntag and J.~M. Pe{\~n}a.
\newblock Learning multivariate regression chain graphs under faithfulness.
\newblock \emph{Proceedings of the 6th European Workshop on Probabilistic
  Graphical Models}, pages 299--306, 2012.

\bibitem[Sonntag and Pe{\~{n}}a(2015)]{Sonntag2015}
Dag Sonntag and Jose~M. Pe{\~{n}}a.
\newblock Chain graphs and gene networks.
\newblock In Arjen Hommersom and Peter~J.F. Lucas, editors, \emph{Foundations
  of Biomedical Knowledge Representation: Methods and Applications}, pages
  159--178. Springer, 2015.

\bibitem[Sonntag and Pe{\~n}a(2015)]{sonntagpena15}
Dag Sonntag and Jose~M. Pe{\~n}a.
\newblock Chain graph interpretations and their relations revisited.
\newblock \emph{International Journal of Approximate Reasoning}, 58:\penalty0
  39 -- 56, 2015.
\newblock Special Issue of the Twelfth European Conference on Symbolic and
  Quantitative Approaches to Reasoning with Uncertainty (ECSQARU 2013).

\bibitem[Spirtes et~al.(2000)Spirtes, Glymour, and Scheines]{sgs}
P.~Spirtes, C.~Glymour, and R.~Scheines.
\newblock \emph{Causation, Prediction and Search, second ed.}
\newblock MIT Press, Cambridge, MA., 2000.

\bibitem[{Studen\'y} et~al.(2009){Studen\'y}, {Roverato}, and {\v St\v
  ep\'anov\'a}]{studeny09}
M.~{Studen\'y}, A.~{Roverato}, and {\v S}.~{\v St\v ep\'anov\'a}.
\newblock {Two operations of merging and splitting components in a chain
  graph.}
\newblock \emph{{Kybernetika}}, 45\penalty0 (2):\penalty0 208--248, 2009.

\bibitem[Tian et~al.(1998)Tian, Paz, and Pearl]{tpp}
Jin Tian, Azaria Paz, and Judea Pearl.
\newblock Finding minimal d-separators.
\newblock Technical report, R-254, 1998.

\bibitem[Tsamardinos et~al.(2006)Tsamardinos, Brown, and
  Aliferis]{Tsamardinos2006}
Ioannis Tsamardinos, Laura~E. Brown, and Constantin~F. Aliferis.
\newblock The max-min hill-climbing {B}ayesian network structure learning
  algorithm.
\newblock \emph{Machine Learning}, 65\penalty0 (1):\penalty0 31--78, Oct 2006.

\bibitem[van~der Zander and Liskiewicz(2019)]{van2019finding}
Benito van~der Zander and Maciej Liskiewicz.
\newblock Finding minimal d-separators in linear time and applications.
\newblock In Ryan Adams and Vibhav Gogate, editors, \emph{Proceedings of the
  35th Conference on Uncertainty in artificial intelligence}. UAI, 2019.

\bibitem[van~der Zander et~al.(2019)van~der Zander, Li{\'s}kiewicz, and
  Textor]{van2019separators}
Benito van~der Zander, Maciej Li{\'s}kiewicz, and Johannes Textor.
\newblock Separators and adjustment sets in causal graphs: Complete criteria
  and an algorithmic framework.
\newblock \emph{Artificial Intelligence}, 270:\penalty0 1--40, 2019.

\bibitem[Xiang(2002)]{Xiang2002}
Yang Xiang.
\newblock \emph{Probabilistic Reasoning in Multi-Agent Systems: A Graphical
  Models Approach}.
\newblock Cambridge University Press, New York, NY, USA, 2002.

\bibitem[Xie et~al.(2006)Xie, Geng, and Zhao]{xie}
Xianchao Xie, Zhi Geng, and Qiang Zhao.
\newblock Decomposition of structural learning about directed acyclic graphs.
\newblock \emph{Artificial Intelligence}, 170\penalty0 (4-5):\penalty0
  422--439, 2006.

\end{thebibliography}

\end{document}